\documentclass{article}

\usepackage{microtype}
\usepackage{graphicx}
\usepackage{booktabs} % for professional tables

\usepackage{hyperref}

\usepackage[accepted]{icml2021}

\usepackage{enumerate}    
\usepackage{paralist}

\usepackage{booktabs}
\usepackage{enumitem}
\usepackage{amsmath}
\usepackage{amssymb}
\usepackage{amsthm}
\usepackage{bbm}
\usepackage{graphicx}
\usepackage{xcolor}
\newsavebox{\algleft}
\newsavebox{\algright}

\DeclareMathOperator*{\argmin}{arg\,min}
\allowdisplaybreaks

\newcommand{\E}[1]{\mathbb{E} \left[#1\right]}

\newtheorem{theorem}{Theorem}
\newtheorem{lemma}{Lemma}

\newtheorem{assumption}{Assumption}
\newtheorem{proposition}{Proposition}

\usepackage{caption}
\usepackage{subcaption}

\usepackage{titlesec}
\icmltitlerunning{Debiasing a First-order Heuristic for Approximate Bi-level Optimization}

\begin{document}

\twocolumn[
\icmltitle{Debiasing a First-order Heuristic for Approximate Bi-level Optimization}



\icmlsetsymbol{equal}{*}

\begin{icmlauthorlist}
\icmlauthor{Valerii Likhosherstov}{equal,cam}
\icmlauthor{Xingyou Song}{equal,goo}
\icmlauthor{Krzysztof Choromanski}{goo,col}
\icmlauthor{Jared Davis}{deep,stan}
\icmlauthor{Adrian Weller}{cam,turing}
\end{icmlauthorlist}

\icmlaffiliation{cam}{University of Cambridge}
\icmlaffiliation{goo}{Google Research, Brain Team}
\icmlaffiliation{col}{Columbia University}
\icmlaffiliation{stan}{Stanford University}
\icmlaffiliation{deep}{Deepmind}
\icmlaffiliation{turing}{The Alan Turing Institute}

\icmlcorrespondingauthor{Valerii Likhosherstov}{vl304@cam.ac.uk}

\icmlkeywords{Machine Learning, ICML}

\vskip 0.3in
]



\printAffiliationsAndNotice{\icmlEqualContribution} % otherwise use the standard text.

\begin{abstract}
Approximate bi-level optimization (ABLO) consists of (outer-level) optimization problems, involving numerical (inner-level) optimization loops. While ABLO has many applications across deep learning, it suffers from time and memory complexity proportional to the length $r$ of its inner optimization loop. To address this complexity, an earlier \textit{first-order method} (FOM) was proposed as a heuristic that omits second derivative terms, yielding significant speed gains and requiring only constant memory. Despite FOM's popularity, there is a lack of theoretical understanding of its convergence properties. We contribute by theoretically characterizing FOM's gradient bias under mild assumptions. We further demonstrate a rich family of examples where FOM-based SGD does not converge to a stationary point of the ABLO objective. We address this concern by proposing an unbiased FOM (UFOM) enjoying constant memory complexity as a function of $r$. We characterize the introduced time-variance tradeoff, demonstrate convergence bounds, and find an optimal UFOM for a given ABLO problem. Finally, we propose an efficient adaptive UFOM scheme.
\end{abstract}

\section{Introduction}

Bi-level optimization (BLO) is a popular technique, which defines an outer-level objective function using the result of another (inner-level) optimization. Approximate BLO (ABLO), where the inner-level optimization is approximated by an iterative solver, has multiple applications in various subtopics of deep learning: hyperparameter optimization \citep{hyperopt,hyperoptimpl,truncated,revopt}, neural architecture search \citep{darts}, adversarial robustness \citep{advrobbilevel,advrobgraphs} and model-agnostic meta-learning \citep{maml,truncated}.

In deep learning applications of ABLO, the outer-level optimization is conducted via stochastic gradient descent (SGD), where gradients are obtained by auto-differentiation \citep{autograd} through the $r$ steps of the inner gradient descent (GD). This leads to storing $r$ intermediate inner steps, which requires a prohibitive amount of memory.

To address this, several methods for ABLO have been proposed. Among them is truncated back-propagation \citep{truncated} which stores a fixed amount of intermediate steps. Another is implicit differentiation \citep{impl1,impl2,imaml,hyperoptimpl}, which takes advantage of the Implicit Function Theorem, but requires multiple restrictions on the problem \citep{imaml}.
Designing the optimizer's iteration as reversible dynamics \citep{revopt} allows a constant reduction in memory complexity, as only the bits lost in fixed-precision arithmetic operations are saved in memory. Forward differentiation \citep{fordiff} can be employed when outer-level gradients require a small number of parameters.
Checkpointing \citep{logckpt,ckpt} is a generic solution for memory reduction by a factor of $\sqrt{r}$.
Unbiased stochastic truncation \citep{randtelsums} can save time during stochastic optimization, but its maximum memory usage can still grow to $O(r)$.
Evolution Strategies, or ES \citep{esblo}, provides unbiased gradient estimates in $O(1)$ memory, but its variance depends linearly on the number of parameters, which can be prohibitively large in deep-learning applications.

One practical and simple heuristic is the \textit{first-order method} (FOM), which omits second-order computations, and does not store inner-loop rollout states in memory. FOM is widely used in the context of neural architecture search \citep{darts}, adversarial robustness \citep{advrobbilevel,advrobgraphs}, and meta-learning \citep{maml}. FOM is a limiting case of truncated back-propagation, where no states are cached in memory. We argue however, that there is a high %trade-off 
cost to FOM's computational simplicity: FOM suffers from biased ABLO gradients and can fail to converge to a stationary point of the ABLO objective. Consequently, it would be highly desirable to modify FOM %so that it would possess 
to achieve better convergence properties, without harming its $O(1)$ memory complexity.

We propose a solution: a FOM-related algorithm %which benefits from 
with unbiased gradient estimation. To achieve this, we first propose an algorithm to compute ABLO gradients in only $O(1)$ memory. Then we use this algorithm, computed ``rarely'' at random with a probability $q$, to debias FOM and yield convergent training. We call our approach \textit{Unbiased First-Order Method} (UFOM) and highlight the following benefits:

%Denote $\mathbb{D}^2$ and $\mathbb{V}^2$ to be the bias and variance of the exact estimator, respectively.
 
1. UFOM has $O(1)$ memory and $O(r + q r^2)$ expected time complexity per outer-loop iteration.

%2. Under mild assumptions, we theoretically provide an \textbf{a)} upper-bound on FOM's bias $\mathbb{D}^2$ and \textbf{b)} also provide an expression for UFOM's variance (Section \ref{sec:biasvar}). We conclude that the parameter $q$ can be used to control UFOM-based SGD's time-variance tradeoff, which is less pronounced when the bias $\mathbb{D}^2$ is small (i.e. when FOM is a very accurate gradient proxy).

2. Under mild assumptions, we show that UFOM-based SGD converges to a stationary point of the ABLO objective (Th. \ref{th:conv}). Confirming the need for UFOM, we prove that there is a rich family of ABLO problems where \textbf{FOM ends up arbitrarily far away from a stationary point} (Th. \ref{th:counter}).
    
3. We analyze the convergence rate of UFOM as a function of wall-clock time. As a result, we derive a condition on the FOM's bias $\mathbb{D}^2$ and the exact estimator's variance $\mathbb{V}^2$ that allows the optimal $q$ to be less than one, meaning that under the constant memory restriction, \textbf{the exact ABLO has a slower convergence rate}. We then give an expression for the optimal $q^* < 1$.

4. Since the expression for the optimal $q^*$ can be intractable in real applications, we propose a theory-inspired heuristic for choosing $q^*$ adaptively during optimization.

We demonstrate the utility of UFOM in a synthetic experiment, data hypercleaning on MNIST \citep{mnist}, and few-shot-learning on CIFAR100 \citep{cifar100} as well as Omniglot \citep{omniglot}. Full proofs are provided in Appendix \ref{sec:proofs} in the Supplement. We provide additional related work discussions in Appendix \ref{sec:related}.

\section{Preliminaries}

We begin by formulating the bi-level optimization (BLO), approximate BLO problems and the first-order method.% and few-shot learning as a prominent example of the BLO problem.

\subsection{Bi-level Optimization}

We present the bi-level optimization (BLO) problem in a form which is compatible with prominent large-scale deep-learning applications. Namely, let $\Omega_\mathcal{T}$ be a non-empty set of \textit{tasks} $\mathcal{T} \in \Omega_\mathcal{T}$ and let $p(\mathcal{T})$ be a probabilistic distribution defined on $\Omega_\mathcal{T}$. In practice, the procedure of sampling a task $\mathcal{T} \sim p(\mathcal{T})$ is usually resource-cheap and consists of retrieving a task from disk using either a random index (for a dataset of a fixed size), a stream, or a simulator. Define two functions as the inner and outer loss respectively: $\mathcal{L}^{in}, \mathcal{L}^{out}: \mathbb{R}^s \times \mathbb{R}^p \times \Omega_\mathcal{T} \to \mathbb{R}$. In addition, define $U: \mathbb{R}^s \times \Omega_\mathcal{T} \to \mathbb{R}^p$ so that $U (\theta, \mathcal{T})$ is a global solution of the inner (task-level) $\mathcal{L}^{in} (\theta, \cdot, \mathcal{T})$ optimization problem. Namely, for all $\theta \in \mathbb{R}^s, \mathcal{T} \in \Omega_\mathcal{T}$: $U (\theta, \mathcal{T}) = \argmin_{\phi \in \mathbb{R}^p} \mathcal{L}^{in} (\theta, \phi, \mathcal{T})$. Then, the \textit{bi-level optimization} is defined as an optimization over $\theta$ to find the optimal average value of $\mathcal{L}^{out} (\theta, U (\theta, \mathcal{T}), \mathcal{T})$:
\begin{equation}
    \!\! \min_\theta  \biggl( \mathcal{M} (\theta) = \mathbb{E}_{p (\mathcal{T})} \left[ \mathcal{L}^{out} (\theta, U (\theta, \mathcal{T}), \mathcal{T}) \right] \biggr) . \label{eq:eopt}
\end{equation}

In practice, however, the global optimum $U (\theta, \mathcal{T})$ is usually intractable and needs to be approximated with a few gradient descent (GD) iterations:
\begin{gather}
    \forall r \in \mathbb{N}: U^{(r)} (\theta, \mathcal{T}) = \phi_r, \quad  \phi_0 = V (\theta, \mathcal{T}), \label{eq:maml1} \\
    \forall 1 \leq j \leq r: \phi_j = \phi_{j - 1} - \alpha_j \frac{\partial}{\partial \phi} \mathcal{L}^{in} (\theta, \phi_{j - 1}, \mathcal{T}), \label{eq:maml}
\end{gather}
where further in the text, $\frac{\partial}{\partial \square}$ denotes a partial derivative of a function with respect to the corresponding argument, while $\nabla_\square$ denotes a full gradient. $V: \mathbb{R}^s \times \Omega_\mathcal{T} \to \mathbb{R}^p$ defines  \textit{a starting point} of the GD, $r$ is a number of GD iterations and $\{ \alpha_j > 0 \}_{j = 1}^\infty$ is a sequence of step sizes. Then the \textit{approximate BLO (ABLO) with $r$ GD steps} is defined as finding an optimal average $\mathcal{L}^{out} (\theta, U^{(r)} (\theta, \mathcal{T}), \mathcal{T})$:
\begin{gather}
    \min_\theta \biggl( \mathcal{M}^{(r)} (\theta) = \mathbb{E}_{p (\mathcal{T})} \left[ \mathcal{L}^{out} (\theta, U^{(r)} (\theta, \mathcal{T}), \mathcal{T}) \right] \biggr) . \label{eq:opt}
\end{gather}

The formulations (\ref{eq:maml1}-\ref{eq:opt}) unify various application scenarios. We can think of $\phi$ as model parameters and (\ref{eq:maml}) as a training loop, while $\theta$ are hyperparameters optimized in the outer loop -- a scenario matching the hyperparameter optimization paradigm \citep{hyperopt,hyperoptimpl,truncated,revopt}. Alternatively, $\theta$ may act as parameters encoding a neural-network's topology which corresponds to differentiable neural architecture search \citep{darts}. Here $\mathcal{T}$ encodes a minibatch drawn from a dataset. The adversarial robustness problem fits into (\ref{eq:maml}-\ref{eq:opt}) by assuming that $\theta$ are learned model parameters and $\phi$ is an input perturbation optimized through (\ref{eq:maml}) to decrease the model's performance. \textit{Model-Agnostic Meta-Learning} (MAML) \citep{maml} is an approach to few-shot learning (adaptation to a few training examples) by learning a generalizable initialization $\phi_0 = V(\theta, \mathcal{T}) = \theta$, where $\mathcal{L}^{in}$ and $\mathcal{L}^{out}$ correspond to an empirical risk on a small set of train and validation shots respectively.

We consider stochastic gradient descent (SGD) \citep{bottou} as a solver for (\ref{eq:opt}) -- see Algorithm \ref{alg:sgd}, where $\mathcal{G} (\theta, \mathcal{T})$ is the exact gradient $\nabla_\theta \mathcal{L}^{out} (U^{(r)} (\theta, \mathcal{T}), \mathcal{T})$ or its approximation and $\{ \gamma_k > 0 \}_{k = 1}^\infty$ is a sequence of outer-loop step sizes satisfying
\begin{equation} \label{eq:step}
    \sum_{k = 1}^\infty \gamma_k = \infty, \quad \lim_{k \to \infty} \gamma_k = 0.
\end{equation}

Assuming that the procedure for sampling tasks $\mathcal{T} \sim p(\mathcal{T})$ takes negligible resources, the time and memory requirement of Algorithm \ref{alg:sgd} is dominated by the part spent for computing $\mathcal{G} (\theta, \mathcal{T})$. Therefore, in our subsequent derivations, we analyse the computational complexity of finding the gradient estimate $\mathcal{G} (\theta, \mathcal{T})$.

\begin{algorithm}[tb]
   \caption{Outer SGD.}
   \label{alg:sgd}
\begin{algorithmic}
   \STATE {\bfseries Input:} $\theta_0 \in \mathbb{R}^s$, \# iterations $\tau$.
   \FOR{$k = 1$ {\bfseries to} $\tau$}
   \STATE Draw $\mathcal{T} \sim p (\mathcal{T})$;
   \STATE Set $\theta_k := \theta_{k - 1} - \gamma_k \mathcal{G} (\theta_{k - 1}, \mathcal{T})$;
   \ENDFOR
\end{algorithmic}
\end{algorithm}

\subsection{Exact ABLO gradients} \label{sec:maml}

Outer SGD requires the computation or approximation of a gradient $\nabla_\theta \mathcal{L}^{out} (\theta, U^{(r)} (\theta, \mathcal{T}), \mathcal{T})$. We apply the full derivative rule with respect to $\theta$ to the inner GD (\ref{eq:maml}) and deduce that
\begin{gather}
    \nabla_\theta \mathcal{L}^{out} (\theta, \phi_r, \mathcal{T}) = \frac{\partial}{\partial \theta} \mathcal{L}^{out} (\theta, \phi_r, \mathcal{T}) + \sum_{j = 1}^r D_j \label{eq:gradfirst} \\
    + \biggl( \frac{\partial}{\partial \theta} V(\theta, \mathcal{T}) \biggr)^\top \nabla_{\phi_0} \mathcal{L}^{out} (\theta, \phi_r, \mathcal{T}), \nonumber \\
    D_j = \biggl( \frac{\partial}{\partial \theta} ( \phi_{j - 1} - \alpha_j \frac{\partial}{\partial \phi} \mathcal{L}^{in} (\theta, \phi_{j - 1}, \mathcal{T}) ) \biggr)^\top \nonumber \\
    \times \nabla_{\phi_j} \mathcal{L}^{out} (\theta, \phi_r, \mathcal{T}) \nonumber \\
    = - \alpha_j \biggl( \frac{\partial^2}{\partial \theta \partial \phi} \mathcal{L}^{in} (\theta, \phi_{j - 1}, \! \mathcal{T}) \! \biggr)^\top \!\! \nabla_{\phi_j} \mathcal{L}^{out} (\theta, \phi_r, \! \mathcal{T}), \label{eq:gradpremid}
\end{gather}
where $\frac{\partial^2}{\partial \theta \partial \phi} \mathcal{L}^{in} (\theta, \phi_{j - 1}, \mathcal{T}) \in \mathbb{R}^{s \times p}$ is a second-order partial derivative matrix.
Next, we observe that
\begin{equation*}
    \nabla_{\phi_r} \mathcal{L}^{out} (\theta, \phi_r, \mathcal{T}) = \frac{\partial}{\partial \phi} \mathcal{L}^{out} (\theta, \phi_r, \mathcal{T})
\end{equation*}
and apply the chain rule to a sequence of states $\phi_1, \dots, \phi_r$ to deduce that for all $1 \leq j \leq r$,
\begin{gather}
    \nabla_{\phi_{j - 1}} \mathcal{L}^{out} (\theta, \phi_r, \mathcal{T}) = \biggl( \frac{\partial}{\partial \phi_{j - 1}}(\phi_{j - 1}  \label{eq:gradmid} \\
    - \alpha_j \frac{\partial}{\partial \phi} \mathcal{L}^{in} (\theta, \phi_{j - 1}, \! \mathcal{T})) \biggr)^\top \nabla_{\phi_j} \mathcal{L}^{out} (\theta, \phi_r, \! \mathcal{T}) \nonumber \\
    \!\!\! = \! \biggl( \! I \! - \! \alpha_j \frac{\partial^2}{\partial \phi^2} \mathcal{L}^{in} (\theta, \phi_{j - 1}, \! \mathcal{T}) \! \biggr)^\top \nabla_{\phi_j} \mathcal{L}^{out} (\theta, \phi_r, \! \mathcal{T}) , \label{eq:gradlast}
\end{gather}
where $\frac{\partial^2}{\partial \phi^2} \mathcal{L}^{in} (\theta, \phi_{r - 1}, \mathcal{T}) \in \mathbb{R}^{p \times p}$ is a second-order partial derivative matrix. Based on (\ref{eq:gradfirst}-\ref{eq:gradlast}), Algorithm \ref{alg:maml} illustrates inner loop for the exact ABLO gradient, which, together with Algorithm \ref{alg:sgd}, outlines the training procedure.

Hessian-vector products $( \frac{\partial^2}{\partial \theta \partial \phi} \mathcal{L}^{in} (\theta, \phi, \mathcal{T}) )^\top b_2$ and $( \frac{\partial^2}{\partial \phi^2} \mathcal{L}^{in} (\theta, \phi, \mathcal{T}) )^\top b_2$ can be computed efficiently at once using the reverse accumulation technique \citep{hessian} without explicitly constructing Hessian matrices $\frac{\partial^2}{\partial \theta \partial \phi} \mathcal{L}^{in} (\theta, \phi, \mathcal{T})$ and $\frac{\partial^2}{\partial \phi^2} \mathcal{L}^{in} (\theta, \phi, \mathcal{T})$. This technique consists of evaluation and automatic differentiation \citep{autograd} of a functional $H (\theta, \phi): \mathbb{R}^s \times \mathbb{R}^p \to \mathbb{R}$, $H(\theta, \phi) = \frac{\partial}{\partial \phi} \mathcal{L}^{in} (\theta, \phi, \mathcal{T})^\top b_2$, where $\frac{\partial}{\partial \phi} \mathcal{L}^{in} (\theta, \phi, \mathcal{T})$ is produced by another automatic differentiation. The result of differentiation is precisely $(\frac{\partial^2}{\partial \theta \partial \phi} \mathcal{L}^{in} (\theta, \phi, \mathcal{T}) )^\top b_2$ and $( \frac{\partial^2}{\partial \phi^2} \mathcal{L}^{in} (\theta, \phi, \mathcal{T}) )^\top b_2$.
Similarly, the Jacobian-vector product $(\frac{\partial}{\partial \theta} V (\theta, \mathcal{T}))^\top b_2$ can be computed by auto-differentiating through $V (\theta, \mathcal{T})^\top b_2$. According to \textit{the Cheap Gradient Principle} \citep{autograd}, the time complexities of the Hessian-vector product and the Jacobian-vector product are the same as the time complexities respectively, of computations $\frac{\partial}{\partial \phi} \mathcal{L}^{in} (\theta, \phi, \mathcal{T})$ and $V (\theta, \mathcal{T})$.

\begin{algorithm}[tb]
   \caption{Inner GD (exact).}
   \label{alg:maml}
\begin{algorithmic}
   \STATE {\bfseries Input:} $\theta \in \mathbb{R}^s$, $r \in \mathbb{N}$, $\mathcal{T}$.
   \STATE {\bfseries Output:} $\mathcal{G} (\theta, \mathcal{T}) = \nabla_\theta \mathcal{L}^{out} (\theta, U^{(r)} (\theta, \mathcal{T}), \mathcal{T})$.
   \STATE New array $\textrm{Arr}[1..r]$;
   \STATE Set $\phi := V(\theta, \mathcal{T})$;
   \FOR{$j = 1$ {\bfseries to} $r$}
   \STATE Set $\textrm{Arr} [j] := \phi$;
   \STATE $\phi := \phi - \alpha_j \frac{\partial}{\partial \phi} \mathcal{L}^{in} (\theta, \phi, \mathcal{T})$;
   \ENDFOR
   \STATE Set $b_1 := \frac{\partial}{\partial \theta} \mathcal{L}^{out} (\theta, \phi, \mathcal{T}), b_2 := \frac{\partial}{\partial \phi} \mathcal{L}^{out} (\theta, \phi, \mathcal{T})$;
   \FOR{$j = r$ {\bfseries to} $1$}
   \STATE Set $\phi := \textrm{Arr} [j]$;
   \STATE Set $b_1 := b_1 - \alpha_j ( \frac{\partial^2}{\partial \theta \partial \phi} \mathcal{L}^{in} (\theta, \phi, \mathcal{T}) )^\top b_2$;
   \STATE Set $b_2 := b_2 - \alpha_j ( \frac{\partial^2}{\partial \phi^2} \mathcal{L}^{in} (\theta, \phi, \mathcal{T}) )^\top b_2$;
   \ENDFOR
   \STATE {\bfseries return} $b_1 + (\frac{\partial}{\partial \theta} V (\theta, \mathcal{T}))^\top b_2$;
\end{algorithmic}
\end{algorithm}

\begin{algorithm}[tb]
   \caption{Inner GD (FOM).}
   \label{alg:fomaml}
\begin{algorithmic}
   \STATE {\bfseries Input:} $\theta \in \mathbb{R}^s$, $r \in \mathbb{N}$, $\mathcal{T}$.
   \STATE {\bfseries Output:} $\mathcal{G}_{FO} (\theta, \mathcal{T})$.
   \STATE Set $\phi := V(\theta, \mathcal{T})$;
   \FOR{$j = 1$ {\bfseries to} $r$}
   \STATE $\phi := \phi - \alpha_j \frac{\partial}{\partial \phi} \mathcal{L}^{in} (\theta, \phi, \mathcal{T})$;
   \ENDFOR
   \STATE Set $b_1 := \frac{\partial}{\partial \theta} \mathcal{L}^{out} (\theta, \phi, \mathcal{T}), b_2 := \frac{\partial}{\partial \phi} \mathcal{L}^{out} (\theta, \phi, \mathcal{T})$;

   \STATE {\bfseries return} $b_1 + (\frac{\partial}{\partial \theta} V (\theta, \mathcal{T}))^\top b_2$;
\end{algorithmic}
\end{algorithm}

\subsection{First-order heuristic for ABLO}

One limitation, which can significantly complicate application of ABLO in real-life scenarios when both the number of parameters $p$ and the number of gradient descent iterations $r$ are large, is the requirement to allocate an array $\mathrm{Arr}[1..r]$ for all intermediate states $\phi_1, \dots, \phi_r \in \mathbb{R}^p$ (Algorithm \ref{alg:maml}).

A simple improvement which sometimes performs well in practice is the \textit{first-order method} (FOM) \citep{maml,darts,advrobbilevel,advrobgraphs}; this proposes to approximate $\nabla_\theta \mathcal{L}^{out} (\theta, U^{(r)} (\theta, \mathcal{T}), \mathcal{T})$ with
\begin{gather}
    \mathcal{G}_{FO} (\theta, \mathcal{T}) = \frac{\partial}{\partial \theta} \mathcal{L}^{out} (\theta, \! \phi_r, \! \mathcal{T}) + \left(\frac{\partial}{\partial \theta} V (\theta, \mathcal{T})\right)^\top \label{eq:fo1} \\
    \times \frac{\partial}{\partial \phi} \mathcal{L}^{out} (\theta, \phi_r, \mathcal{T}), \quad \phi_r = U^{(r)} (\theta, \mathcal{T}), \label{eq:fo2}
\end{gather}
which corresponds to zeroing out Hessians $( \frac{\partial^2}{\partial \theta \partial \phi} \mathcal{L}^{in} (\theta, \phi_{j - 1}, \mathcal{T}) )^\top b_2$ and $( \frac{\partial^2}{\partial \phi^2} \mathcal{L}^{in} (\theta, \phi_{j - 1}, \mathcal{T}) )^\top b_2$ in Equations (\ref{eq:gradfirst}-\ref{eq:gradlast}). FOM (Algorithm \ref{alg:fomaml}) does not store any intermediate states, giving $O(1)$ memory complexity and $O(r)$ time complexity. Although FOM enjoys better memory complexity, it can fail to converge to a stationary point of the ABLO objective (\ref{eq:opt}), see %analysis in 
Section \ref{sec:theory}. In the next section we propose a solution -- modified FOM with the desired convergence properties due to unbiased gradients.

\section{Unbiased First-Order Method (UFOM)} \label{sec:prop}

An alternative, memory-efficient method for finding $\nabla_\theta \mathcal{L}^{out} (\theta, U^{(r)} (\theta, \mathcal{T}), \mathcal{T})$ is to recompute each $\phi_j$ from scratch instead of retrieving it from the array $\mathrm{Arr}$ during the backward loop of Algorithm \ref{alg:maml}. However, this would result in $O(r^2)$ complexity, as re-computation requires an $O(r)$-long nested loop inside each $j = r, \dots, 1$ iteration. Therefore, we propose an unbiased stochastic combination of this slow but memory-efficient exact algorithm, and a fast and cheap, but biased FOM-based estimator.

Let $\xi \in \{ 0, 1 \}$ be a Bernoulli random variable with $\mathbb{P} (\xi = 1) = q$, which we write as $\xi \sim \mathrm{Bernoulli} (q)$, where $q \in (0, 1]$. Recall that $\mathcal{G}_{FO}$ is the first-order gradient from Algorithm \ref{alg:fomaml}. We consider the following stochastic approximation to $\nabla_\theta \mathcal{L}^{out} (\theta, U^{(r)} (\theta, \mathcal{T}), \mathcal{T})$:
\begin{gather}
    \mathcal{G}_{UFO} (\theta, \mathcal{T}) = \mathcal{G}_{FO} (\theta, \mathcal{T}) \label{eq:unb1} \\
    + \frac{\xi}{q} (\nabla_\theta \mathcal{L}^{out} (\theta, U^{(r)} (\theta, \mathcal{T}), \mathcal{T}) - \mathcal{G}_{FO} (\theta, \mathcal{T})) . \label{eq:unb2}
\end{gather}
In fact, (\ref{eq:unb1}-\ref{eq:unb2}) is an unbiased estimate of $\nabla_\theta \mathcal{L}^{out} (\theta, U^{(r)} (\theta, \mathcal{T}), \mathcal{T})$. Indeed, since $\mathbb{E}_\xi \left[\xi\right] = q$:
\begin{gather*}
    \mathbb{E}_\xi \left[\mathcal{G}_{UFO} (\theta, \mathcal{T}) \right] = (1 - \frac{q}{q}) \mathcal{G}_{FO} (\theta, \mathcal{T}) \\
    + \frac{q}{q} \nabla_\theta \mathcal{L}^{out} (\theta, U^{(r)} (\theta, \! \mathcal{T}), \! \mathcal{T}) \! = \! \nabla_\theta \mathcal{L}^{out} (\theta, U^{(r)} (\theta, \! \mathcal{T}), \! \mathcal{T}) .
\end{gather*}
For this reason we call the estimate (\ref{eq:unb1}-\ref{eq:unb2}) an \textit{unbiased first-order method} (UFOM). Algorithm \ref{alg:unbiased} illustrates randomized computation of UFOM. It combines FOM and memory-efficient computation of the exact gradient. Hence, UFOM has constant memory complexity and $O(r + q r^2)$ expected time complexity:
\begin{proposition} \label{th:ufom}
Algorithm \ref{alg:unbiased} requires $r + 1 + q r (r - 1) / 2$ expected evaluations of the gradient $\frac{\partial}{\partial \phi} \mathcal{L}^\square (\theta, \phi, \mathcal{T})$, $\square \in \{ in, out \}$ and $q r$ expected evaluations of the Hessian-vector product $\left[ ( \frac{\partial^2}{\partial \theta \partial \phi} \mathcal{L}^{in} (\theta, \phi, \mathcal{T}) )^\top b_2, ( \frac{\partial^2}{\partial \phi^2} \mathcal{L}^{in} (\theta, \phi, \mathcal{T}) )^\top b_2 \right]$.
\end{proposition}

Observe that, when $q = 1$, Algorithm \ref{alg:unbiased} is reduced to the \textit{memory-efficient exact gradient estimator}, discussed above.

%By the law of large numbers \citep{prob}, the time complexity of UFOM approaches its expected value when $\tau$, the number of outer iterations, is large, which is typical for large-scale problems.

\begin{algorithm}[tb]
   \caption{Inner GD (UFOM).}
   \label{alg:unbiased}
\begin{algorithmic}
   \STATE {\bfseries Input:} $\theta \in \mathbb{R}^s$, $r \in \mathbb{N}$, $\mathcal{T}$, $q \in (0, 1]$.
   \STATE {\bfseries Output:} $\mathcal{G}_{UFO} (\theta, \mathcal{T})$.
   \STATE Set $\phi := V(\theta, \mathcal{T})$;
   \FOR{$j = 1$ {\bfseries to} $r$}
   \STATE $\phi := \phi - \alpha_j \frac{\partial}{\partial \phi} \mathcal{L}^{in} (\theta, \phi, \mathcal{T})$;
   \ENDFOR
   \STATE Set $b_1 := \frac{\partial}{\partial \theta} \mathcal{L}^{out} (\theta, \phi, \mathcal{T}), b_2 := \frac{\partial}{\partial \phi} \mathcal{L}^{out} (\theta, \phi, \mathcal{T})$;
   \STATE Set $b_{FO} := b_1 + (\frac{\partial}{\partial \theta} V (\theta, \mathcal{T}))^\top b_2$;
   \STATE Draw $\xi \sim \mathrm{Bernoulli} (q)$;
   \STATE \algorithmicif \, $\xi = 0$ \algorithmicthen \, {\bfseries return} $b_{FO}$ \algorithmicend\,\algorithmicif

   \FOR{$j = r$ {\bfseries to} $1$}
   \STATE Set $\phi = V(\theta, \mathcal{T})$;
   \FOR{$j_1 = 1$ {\bfseries to} $j - 1$}
   \STATE Set $\phi := \phi - \alpha_j \frac{\partial}{\partial \phi} \mathcal{L}^{in} (\theta, \phi, \mathcal{T})$;
   \ENDFOR
   \STATE Set $b_1 := b_1 - \alpha_j ( \frac{\partial^2}{\partial \theta \partial \phi} \mathcal{L}^{in} (\theta, \phi, \mathcal{T}) )^\top b_2$;
   \STATE  Set $b_2 := b_2 - \alpha_j ( \frac{\partial^2}{\partial \phi^2} \mathcal{L}^{in} (\theta, \phi, \mathcal{T}) )^\top b_2$;
   \ENDFOR
   \STATE Set $b_{Exact} = b_1 + (\frac{\partial}{\partial \theta} V (\theta, \mathcal{T}))^\top b_2$;
   \STATE {\bfseries return} $b_{FO} + \frac{1}{q} \left( b_{Exact} - b_{FO} \right)$;
\end{algorithmic}
\end{algorithm}

\section{Theoretical Results} \label{sec:theory}

We perform a theoretical analysis of the FOM and UFOM. First, we establish assumptions on the problem. Then, we characterize bias of the FOM approximation and variance of the UFOM. Next, we prove the convergence rate bounds for UFOM and show divergence of the FOM. Finally, we use these obtained bounds to select the optimal value of $q$ and propose a heuristical adaptive scheme.

\subsection{Assumptions}

%In this section we first provide convergence guarantees for UFOM under a set of broad, nonconvex assumptions (Theorem \ref{th:conv}).

We first formulate Assumptions \ref{as:liphes} and \ref{as:mreg} used for proofs. The first assumption requires uniformly bounded, uniformly Lipschitz-continuous gradients and Hessians of $\mathcal{L}^{in}$, $\mathcal{L}^{out}$.
\begin{assumption} \label{as:liphes}
    For any $\mathcal{T} \in \Omega_\mathcal{T}$, $\mathcal{L}^{in} (\theta, \phi, \mathcal{T})$, $\mathcal{L}^{out} (\theta, \phi, \mathcal{T})$ are twice differentiable as functions of $\theta, \phi$ and $V (\theta, \mathcal{T})$ is differentiable as a function of $\theta$. There exist $M_1, M_2, L_1, L_2, L_3 > 0$ such that for any $\mathcal{T} \in \Omega_\mathcal{T}, \theta', \theta'' \in \mathbb{R}^s, \phi', \phi'' \in \mathbb{R}^p, \square \in \{ in, out \}$
    \begin{gather*}
        \| \frac{\partial}{\partial \theta} V (\theta', \mathcal{T}) \|_2 \leq M_1, \\
        %%%%
        \| \frac{\partial}{\partial \theta} V (\theta', \mathcal{T}) - \frac{\partial}{\partial \theta} V (\theta'', \mathcal{T}) \|_2 \leq M_2 \| \theta' - \theta'' \|_2, \\
        %%%%
        \max ( \| \frac{\partial}{\partial \theta} \mathcal{L}^\square (\theta', \phi', \mathcal{T}) \|_2, \| \frac{\partial}{\partial \phi} \mathcal{L}^\square (\theta', \phi', \mathcal{T}) \|_2 ) \leq L_1, \\
        %%%%
        \max ( \| \frac{\partial}{\partial \theta} \mathcal{L}^\square (\theta', \phi', \mathcal{T}) - \frac{\partial}{\partial \theta} \mathcal{L}^\square (\theta'', \phi'', \mathcal{T}) \|_2, \\
        \| \frac{\partial}{\partial \phi} \mathcal{L}^\square (\theta', \phi', \mathcal{T}) - \frac{\partial}{\partial \phi} \mathcal{L}^\square (\theta'', \phi'', \mathcal{T}) \|_2 ) \\
        \leq L_2 \| \theta' - \theta'' \|_2 + L_2 \| \phi' - \phi'' \|_2, \\
        %%%%
        \max ( \| \frac{\partial^2}{\partial \theta \partial \phi} \mathcal{L}^{in} (\theta', \phi', \mathcal{T}) - \frac{\partial^2}{\partial \theta \partial \phi} \mathcal{L}^{in} (\theta'', \phi'', \mathcal{T}) \|_2, \\
        \| \frac{\partial^2}{\partial \phi^2} \mathcal{L}^{in} (\theta', \phi', \mathcal{T}) - \frac{\partial^2}{\partial \phi^2} \mathcal{L}^{in} (\theta'', \phi'', \mathcal{T}) \|_2 ) \\
        \leq L_3 \| \theta' - \theta'' \|_2 + L_3 \| \phi' - \phi'' \|_2 .
    \end{gather*}
\end{assumption}

The next assumption is satisfied in particular when $p (\mathcal{T})$ is defined on a finite set of tasks $\mathcal{T}$ and $\mathcal{L}^{out}$ is lower-bounded.
\begin{assumption}[Regularity of $\mathcal{M}^{(r)}$] \label{as:mreg}
    For each $\theta \in \mathbb{R}^s$, the terms $\mathcal{M}^{(r)} (\theta)$, $\frac{\partial}{\partial \theta} \mathcal{M}^{(r)} (\theta)$, $\mathbb{E}_{p (\mathcal{T})} \left[\nabla_\theta \mathcal{L}^{out} (\theta, U^{(r)} (\theta, \mathcal{T}), \mathcal{T}) \right]$ are well-defined and $\frac{\partial}{\partial \theta} \mathcal{M}^{(r)} (\theta) = \mathbb{E}_{p (\mathcal{T})} \left[\nabla_\theta \mathcal{L}^{out} (\theta, U^{(r)} (\theta, \mathcal{T}), \mathcal{T})\right]$. Let $\mathcal{M}^{(r)}_* = \inf_{\theta \in \mathbb{R}^s} \mathcal{M}^{(r)} (\theta)$, then $\mathcal{M}^{(r)}_* > - \infty$.
\end{assumption}

\subsection{Bias of FOM and Variance of UFOM} \label{sec:biasvar}

FOM provides a cheap but biased estimate of the exact gradient $\nabla_\theta \mathcal{L}^{out}$. %Therefore, it is important t
To characterize this bias, and its effect on the convergence %rate 
of UFOM, we define %. We use the following quantity, equivalent to 
$\mathbb{D}^2$ as: 
\begin{gather}
    \!\! \sup_{\theta \in \mathbb{R}^s} \mathbb{E}_{\mathcal{T}} [ \| \mathcal{G}_{FO} (\theta, \mathcal{T}) \! -  \! \nabla_\theta \mathcal{L}^{out} (\theta, \! U^{(r)} (\theta, \! \mathcal{T}), \mathcal{T}) \|_2^2 ] .\label{eq:dtruedef}
\end{gather}
Intuitively, $\mathbb{D}^2$ is the maximal average bias over all $\theta$ values. Under Assumptions \ref{as:liphes} and \ref{as:mreg}, Lemma \ref{lemma:unbbnd} in Appendix \ref{sec:proofs} establishes the bound $\mathbb{D}^2 \leq \mathbb{D}_{bound}^2$, where
\begin{equation}
    \mathbb{D}_{bound} = (1 + M_1) L_1 L_2 \sum_{j = 1}^r \alpha_j \prod_{j' = j}^r (1 + \alpha_{j'} L_2) \label{eq:ddef} .
\end{equation}
The ability to derive the bound (\ref{eq:ddef}) and memory-efficiency \textbf{justify the use of FOM gradient proxies in the debiasing scheme} (Algorithm \ref{alg:unbiased}), as opposed to any other gradient approximations. Intuitively, when $\mathbb{D}^2$ is small -- or sufficiently, its tractable proxy $\mathbb{D}_{bound}^2$ is small -- the variance of the UFOM-based gradient estimate is reduced. This agrees with our theoretical findings: in Lemma \ref{lemma:unbbnd} we show that
\begin{equation*}
    \sup_{\theta \in \mathbb{R}^s} \mathbb{E}_{p(\mathcal{T})} \left[\| \mathcal{G}_{UFO} (\theta, \mathcal{T}) \|_2^2 \right] \leq (\frac{1}{q} - 1) \mathbb{D}^2 + \mathbb{V}^2 ,
\end{equation*}
where the left-hand side characterizes the variance of UFOM and $\mathbb{V}^2$ is the maximal variance of the exact ABLO gradient:
\begin{equation}
    \mathbb{V}^2 \! = \! \sup_{\theta \in \mathbb{R}^s} \mathbb{E}_{p (\mathcal{T})} \! \left[ \| \nabla_\theta \mathcal{L}^{out} (\theta, U^{(r)} (\theta, \mathcal{T}), \mathcal{T}) \|_2^2 \right]. \label{eq:vtruedef}
\end{equation}
The upper bound on the variance of UFOM approaches $\mathbb{V}^2$ if either $\mathbb{D}^2$ approaches $0$, or $q$ approaches $1$. %In Lemma \ref{lemma:unbbnd} we derive a closed-form proxy for $\mathbb{V}_{true}^2$:
%\begin{gather}
%    \mathbb{V}_{true}^2 \leq \mathbb{V}^2, \quad \mathbb{V} = L_1 + M_1 L_1 \prod_{j = 1}^r (1 + \alpha_j L_2) \nonumber \\
%    + L_1 L_2 \sum_{j = 1}^r \alpha_j \prod_{j' = j}^r (1 + \alpha_{j'} L_2) . \label{eq:vdef}
%\end{gather}

\subsection{Convergence of UFOM and Divergence of FOM}

We analyze FOM and UFOM as algorithms which seek a stationary point of the ABLO objective (\ref{eq:opt}), i.e. a point $\theta^* \in \mathbb{R}^s$ such that $\frac{\partial}{\partial \theta} \mathcal{M}^{(r)} (\theta^*) = \mathbf{0}_s$ (a vector of $s$ zeros). Motivated by searching for $\theta^*$, we prove a standard result for stochastic optimization of nonconvex functions \citep[Section 4.3]{bottou}:
\begin{equation} \label{eq:liminf}
    \liminf_{k \to \infty} \E{ \| \frac{\partial}{\partial \theta} \mathcal{M}^{(r)} (\theta_k) \|_2^2} = 0 ,
\end{equation}
where $\theta_k$ are iterates of SGD with UFOM gradient estimation. Intuitively, equation (\ref{eq:liminf}) implies that there exist iterates of UFOM which approach some stationary point $\theta^*$ up to any level of proximity. Proofs are in Appendix \ref{sec:proofs}.
\begin{theorem}[Convergence of UFOM] \label{th:conv}
Let $p, r, s \in \mathbb{N}$, $\{ \alpha_j > 0 \}_{j = 1}^\infty$ and $\{ \gamma_k > 0 \}_{k = 1}^\infty$ be any sequences, $q \in (0, 1], \theta_0 \in \mathbb{R}^s$, $p(\mathcal{T})$ be a distribution on a nonempty set $\Omega_\mathcal{T}$, $V: \mathbb{R}^s \times \Omega_\mathcal{T} \to \mathbb{R}^p, \mathcal{L}^{in}, \mathcal{L}^{out} : \mathbb{R}^s \times \mathbb{R}^p \times \Omega_\mathcal{T} \to \mathbb{R}$ be functions satisfying Assumption \ref{as:liphes}, and let $U^{(r)}: \mathbb{R}^s \times \Omega_\mathcal{T} \to \mathbb{R}^p$ be defined according to (\ref{eq:maml1}-\ref{eq:maml}), $\mathcal{M}^{(r)} : \mathbb{R}^p \to \mathbb{R}$ be defined according to (\ref{eq:opt}) and satisfy Assumption \ref{as:mreg}. Define $\mathcal{G}_{FO}: \mathbb{R}^s \times \Omega_\mathcal{T} \to \mathbb{R}^s$ as in (\ref{eq:fo1}-\ref{eq:fo2}) and $\mathcal{G}: \mathbb{R}^s \times \Omega_\mathcal{T} \times \{ 0, 1 \} \to \mathbb{R}^s$ as
\begin{gather*}
    \mathcal{G} (\theta, \mathcal{T}, x) = \mathcal{G}_{FO} (\theta, \mathcal{T}) + \frac{x}{q} (\nabla_\theta \mathcal{L}^{out} (\theta,  U^{(r)} (\theta, \mathcal{T}), \mathcal{T}) \\
    - \mathcal{G}_{FO} (\theta, \mathcal{T})) .
\end{gather*}
Let $\{ \mathcal{T}_k \}_{k = 1}^\infty, \{ \xi_k \}_{k = 1}^\infty$ be sequences of i.i.d. samples from $p(\mathcal{T})$ and $\mathrm{Bernoulli} (q)$ respectively, such that $\sigma$-algebras populated by both sequences are independent.
Let $\{ \theta_k \in \mathbb{R}^s \}_{k = 0}^\infty$ be a sequence where for all $k \in \mathbb{N}$ $\theta_k = \theta_{k - 1} - \gamma_k \mathcal{G} (\theta_{k - 1}, \mathcal{T}_k, \xi_k)$. Then it holds that

1) if $\{ \gamma_k \}_{k = 1}^\infty$ satisfies (\ref{eq:step}) and $\sum_{k = 1}^\infty \gamma_k^2 < \infty$, then $\liminf_{k \to \infty} \E{\| \frac{\partial}{\partial \theta} \mathcal{M}^{(r)} (\theta_k) \|_2^2} = 0$;

2) if $\forall k \in \mathbb{N}: \gamma_k = k^{-0.5}$, then
\begin{gather}
    \min_{0 \leq u < k} \E{ \| \frac{\partial}{\partial \theta} \mathcal{M}^{(r)} (\theta_u) \|_2^2 }   \nonumber \\
    = \Big( (1 / q - 1) \mathbb{D}^2 + \mathbb{V}^2 \Big) \cdot O (k^{-0.5 + \epsilon}), \label{eq:conv}
\end{gather}
where $\epsilon$ is any positive number and the constant, hidden in $O(k^{-0.5 + \epsilon})$, does not depend on $q$. $\mathbb{D}^2, \mathbb{V}^2$ are defined by (\ref{eq:dtruedef}), (\ref{eq:vtruedef}) respectively.
\end{theorem}

Our second contribution is a proof that Equation (\ref{eq:liminf}) does not hold for FOM under the same assumptions. % (Theorem \ref{th:counter}). 
More specifically, we show that for any $D > 0$, there exists a problem of type (\ref{eq:opt}) such that $\liminf_{k \to \infty} \E{\| \nabla_{\theta_k} \mathcal{M}^{(r)} (\theta_k) \|_2^2} > D$ where $\{ \theta_k \}$ are iterates of FOM. The intuition behind this result is that \textbf{FOM cannot find a solution with gradient norm lower than $D$}.
\begin{theorem} [Divergence of FOM] \label{th:counter}
Let $p = s, r \in \mathbb{N}, \alpha > 0, \theta_0 \in \mathbb{R}^s$, $\{ \gamma_k > 0 \}_{k = 1}^\infty$ be any sequence satisfying (\ref{eq:step}), and $D$ be any positive number. Then there exists a set $\Omega_\mathcal{T}$ with a distribution $p(\mathcal{T})$ on it and functions $V : \mathbb{R}^s \times \Omega_\mathcal{T} \to \mathbb{R}^p, \mathcal{L}^{in}, \mathcal{L}^{out} : \mathbb{R}^s \times \mathbb{R}^p \times \Omega_\mathcal{D} \to \mathbb{R}$ satisfying Assumption \ref{as:liphes}, such that for $U^{(r)}: \mathbb{R}^p \times \mathcal{T} \to \mathbb{R}^p$ defined according to (\ref{eq:maml1}-\ref{eq:maml}), where $\forall j : \alpha_j = \alpha$, $\mathcal{M}^{(r)} : \mathbb{R}^p \to \mathbb{R}$ defined according to (\ref{eq:opt}) and satisfying Assumption \ref{as:mreg}, the following holds: define $\mathcal{G}_{FO}: \mathbb{R}^s \times \Omega_\mathcal{T} \to \mathbb{R}^s$ as in (\ref{eq:fo1}-\ref{eq:fo2}). Let $\{ \mathcal{T}_k \}_{k = 1}^\infty$ be a sequence of i.i.d. samples from $p(\mathcal{T})$. Let $\{ \theta_k \in \mathbb{R}^p \}_{k = 0}^\infty$ be a sequence where for all $k \in \mathbb{N}$, $\theta_k = \theta_{k - 1} - \gamma_k \mathcal{G}_{FO} (\theta_{k - 1}, \mathcal{T}_k)$. Then $\liminf_{k \to \infty} \E {\| \frac{\partial}{\partial \theta} \mathcal{M}^{(r)} (\theta_k) \|_2^2} > D$.
\end{theorem}

\subsection{Optimal Choice of $q$} \label{sec:optqth}

Note that as a special case of Theorem \ref{th:conv}, we obtain a convergence proof for ABLO with exact gradients ($q = 1$). The case $q \to 0$, on the other hand, corresponds to FOM without corrections. We analyze the bound (\ref{eq:conv}) in order to develop an intuition about the choice of $q$, leading to the fastest convergence as a function of wall-clock time. Let $C_1, C_2$ denote the time required to compute $\frac{\partial}{\partial \phi} \mathcal{L}^\square (\theta, \phi, \mathcal{T})$, $\square \in \{ in, out \}$ and $\left[ ( \frac{\partial^2}{\partial \theta \partial \phi} \mathcal{L}^{in} (\theta, \phi, \mathcal{T}) )^\top b_2, ( \frac{\partial^2}{\partial \phi^2} \mathcal{L}^{in} (\theta, \phi, \mathcal{T}) )^\top b_2 \right]$ respectively.
%In practice, $C_1$ is usually smaller than $C_2$ \citep{autograd}.
Fix $0 < \epsilon$. The smaller is $\epsilon$, the tighter is the bound (\ref{eq:conv}), so in addition we assume that $\epsilon < 0.5$. Let $\mathcal{H} (\epsilon)$ be a hidden constant in $O$-notation of (\ref{eq:conv}). Then for $\delta$ small enough and , it holds that one needs
\begin{equation*}
     \left( \mathcal{H} (\epsilon)^{-1} \cdot \delta \cdot ((1/q - 1) \mathbb{D}^2 + \mathbb{V}^2)^{-1} \right)^{\frac{1}{- 0.5 + \epsilon}}
\end{equation*}
iterations to achieve $\min_{0 \leq u < k} \E{ \| \frac{\partial}{\partial \theta} \mathcal{M}^{(r)} (\theta_u) \|_2^2 } \leq \delta$. According to Proposition \ref{th:ufom}, the expected amount of time to get such precision is proportional to
\begin{equation}
    \left((1/q - 1) \mathbb{D}^2 + \mathbb{V}^2 \right)^{\frac{2}{1 - 2 \epsilon}} (C_{det} + C_{rnd} \cdot q) , \label{eq:time}
\end{equation}
where $C_{det} = C_1 (r + 1), C_{rnd} = ( C_1 (r - 1) / 2 + C_2) r$, and we ignore the negligible time needed for the vector sums and scalar products in  %and we do not encounter negligible time required for a couple of vector sums and products by a scalar in 
Algorithm \ref{alg:unbiased}. The convergence estimate as a function of time is tighter for UFOM than for the exact memory-efficient gradient, in the case when the $q^*$, minimizing (\ref{eq:time}), is smaller than $1$. In Appendix \ref{sec:qchoice}, we show that %, assuming $\epsilon < \frac12$,%
\textbf{this is equivalent to the condition}
\begin{equation}
    \mathbb{D}^2 < \frac{C_{rnd}}{ \frac{2}{1 - 2 \epsilon} (C_{det} + C_{rnd})} \mathbb{V}^2 . \label{eq:condd}
\end{equation}
%where the second transition 
%where the first condition is sufficient without dependence on $\epsilon$.
We also prove that in case of (\ref{eq:condd}), the optimum $q^*$ lies in $(0, 1)$ as the root of the quadratic equation
\begin{gather}
    C_{rnd} (\mathbb{V}^2 - \mathbb{D}^2) q^2 + \frac{2 \epsilon + 1}{2 \epsilon - 1} \mathbb{D}^2 C_{rnd} q \nonumber \\
    + \frac{2}{2 \epsilon - 1} \mathbb{D}^2 C_{det} = 0 . \label{eq:closeform}
\end{gather}

%Since $\mathbb{D}^2_{true} \leq \mathbb{D}^2, \mathbb{V}^2_{true} \leq \mathbb{V}^2$, one can substitute $\mathbb{D}^2_{true} \to \mathbb{D}^2, \mathbb{V}^2_{true} \to \mathbb{V}^2$ in the convergence bound (\ref{eq:conv}) to get a possibly more loose, but tractable estimate. Then the same substition applies for the condition (\ref{eq:condd}).

\subsection{Adaptive UFOM}

For simplicity, in Assumption \ref{as:liphes} we require global bounds $M_1, M_2, L_1, L_2, L_3$ on norms and Lipschitz constants of $V, \mathcal{L}^{in}, \mathcal{L}^{out}$ gradients. However, it can be deduced from the proof of Theorem \ref{th:conv}, that this assumption can be relaxed by only assuming bounded and Lipschitz-continuous gradients in the open set containing the trajectory $[\theta_0, \theta_1] \cup [\theta_1, \theta_2] \cup \dots$ of SGD. Similarly, (\ref{eq:dtruedef},\ref{eq:vtruedef}) can be relaxed by taking supremum over this open set rather than the whole $\mathbb{R}^{s}$. With this observation in mind, we propose a heuristic algorithm for choosing $q$ adaptively during training (\textit{Adaptive UFOM}). According to this algorithm, we maintain and update approximate estimates $\overline{\mathbb{D}^2_k}, \overline{\mathbb{V}^2_k}$ of $\mathbb{D}^2, \mathbb{V}^2$ respectively, and, during each iteration of SGD, compute $q^*$ as described in Section \ref{sec:optqth} by substituting these approximations. $C_1, C_2, \epsilon$ are assumed to be known and set by the user. We define
\begin{equation}
    \overline{\mathbb{D}^2_k} \! = \mathbb{D}^2_{sm,k} / (1 - \beta^{k_{upd}}), \overline{\mathbb{V}^2_k} \! = \mathbb{V}^2_{sm,k} / (1 - \beta^{k_{upd}}), \label{eq:ovdvdef}
\end{equation}
where $0 < \beta < 1$ is a user-defined exponential smoothing constant and $\mathbb{D}^2_{sm,k}, \mathbb{V}^2_{sm,k}$ are exponentially smoothed $\| b_{FO} - b_{Exact} \|^2_2, \| b_{Exact} \|_2^2$ from Algorithm \ref{alg:unbiased}. That is, initially $\mathbb{D}^2_{sm,0} = \mathbb{V}^2_{sm,0} = 0$. During the $k$th step of SGD, if $\xi = 0$ in Algorithm \ref{alg:unbiased}, then $\mathbb{D}^2_{sm,k} = \mathbb{D}^2_{sm,k - 1}, \mathbb{V}^2_{sm,k} = \mathbb{V}^2_{sm,k - 1}$. If $\xi = 1$, then we use the update rule
\begin{gather}
    \mathbb{D}^2_{sm,k} = \beta \mathbb{D}^2_{sm,k - 1} + (1 - \beta) \| b_{FO} - b_{Exact} \|^2_2 , \label{eq:adapt1} \\
    \mathbb{V}^2_{sm,k} = \beta \mathbb{V}^2_{sm,k - 1} + (1 - \beta) \| b_{Exact} \|^2_2 . \label{eq:adapt2}
\end{gather}
$k_{upd}$ is the number of updates (\ref{eq:adapt1}-\ref{eq:adapt2}) performed before the $k$th iteration. This way, $\overline{\mathbb{D}^2_k}, \overline{\mathbb{V}^2_k}$ resemble ``local'' estimates of (\ref{eq:dtruedef},\ref{eq:vtruedef}) at the current region of SGD trajectory.

To perform updates (\ref{eq:adapt1}-\ref{eq:adapt2}) frequently enough, we set $q = \max (q^*, q_{min})$, where $q_{min}$ is a small constant to guarantee that $\xi$ is nonzero at least sometimes when running Algorithm \ref{alg:unbiased}. The proposed Adaptive UFOM doesn't result in any additional evaluations of $\mathcal{L}^{in}, \mathcal{L}^{out}$ or its first or second derivatives. As for the additional memory, Adaptive UFOM only requires to store and update two values $\mathbb{D}^2_{sm,k}, \mathbb{V}^2_{sm,k}$.

\section{Experiments} \label{sec:exp}

We illustrate our theoretical findings on a synthetic experiment and then evaluate Adaptive UFOM on data hypercleaning and few-shot learning. As baselines, we compare with the memory-efficient exact estimator (equivalent to UFOM with $q = 1$) and FOM (equivalent to $q \to 0$). We demonstrate that, under the same memory constraints, UFOM can significantly speed up convergence of the exact estimator. As a measure of algorithm speed, we use the total number of gradient and Hessian-vector product calls, which we refer to as \textit{function calls}. This is equivalent to the ``wall-clock time'' from Section \ref{sec:qchoice}, if we set $C_1 = C_2 = 1$. Since $\epsilon$ is any small positive number, we further set it to $0$ for simplicity.

\subsection{Synthetic Experiment}

Theorem \ref{th:counter} is proven by explicitly constructing the following counterexample: $\Omega_\mathcal{T} = \{ \mathcal{T}^{(1)}, \mathcal{T}^{(2)} \}$ where both tasks $\mathcal{T}^{(1)}, \mathcal{T}^{(2)}$ are equally likely under $p(\mathcal{T})$, $V(\theta, \mathcal{T}) = \theta$ and functions $\mathcal{L}^{in} (\theta, \phi, \mathcal{T}^{(i)}) = \mathcal{L}^{out} (\theta, \phi, \mathcal{T}^{(i)})$ are independent of $\theta$ and are convex piecewise-polynomials of $\phi$ for (see Appendix \ref{sec:proofs} for details). Inner-GD step sizes are the same ($\forall j : \alpha_j = \alpha$).  As a demonstration of our theoretical findings in a simple experimental setup, we simulate via exact gradients, both FOM and (non-adaptive) UFOM for this synthetic ABLO formulation. We opt for a single parameter ($s = p = 1$) and inner-GD length of $r = 10$.

Figures \ref{fig:sinth1}, \ref{fig:sinth2}, \ref{fig:sinth3} demonstrate the result of Theorems \ref{th:conv} and \ref{th:counter}, illustrating an example problem where UFOM ($q = 0.1$) converges to a stationary point, while FOM does not. In Figure \ref{fig:sinth4}, we vary the $\alpha$ parameter and illustrate how approximate values of $\mathbb{V}^2, \mathbb{D}^2$ change. On Figure \ref{fig:sinth5}, we compare a theoretical estimate of $q^*$ obtained by solving (\ref{eq:closeform}), and an empirical estimate of $q^*$ obtained by grid search. We observe that our theoretical findings roughly agree with the experiment. Finally, on Figure \ref{fig:sinth6}, we illustrate the case where the theoretically optimal $q^*$ is less than $1$. We show that, indeed, optimization is faster when $q = q^* \approx 0.08$ than when $q = 1$ (exact memory-efficient gradient computation). More details can be found in Appendix \ref{sec:synth}.

\begin{figure}[h!]
  \centering
    \begin{subfigure}[b]{0.49\linewidth}
        \centering
        \includegraphics[width=\linewidth]{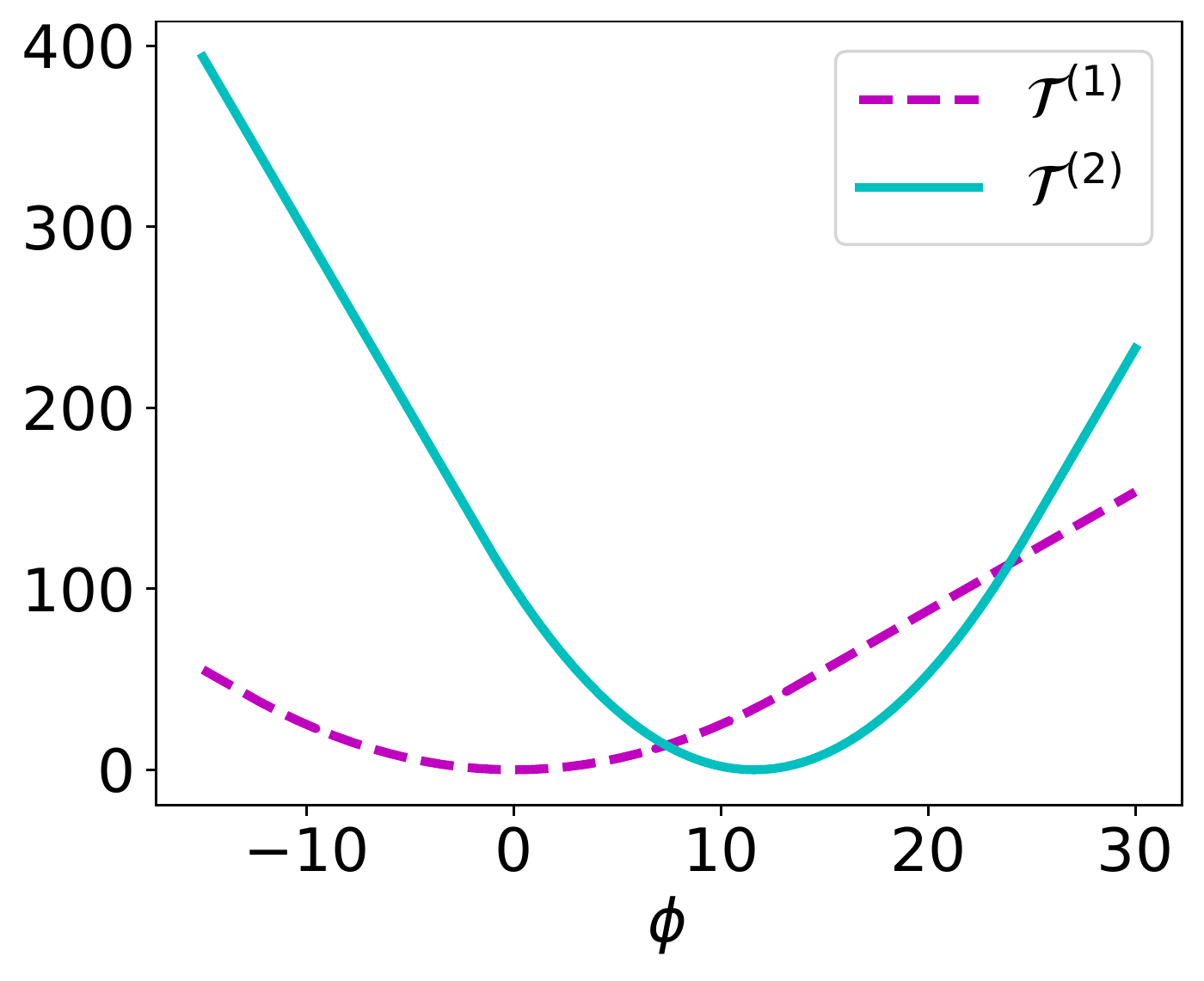}
        \caption[]%
        {{\small \,}}
        \label{fig:sinth1}
    \end{subfigure}
    %\hfill
    \begin{subfigure}[b]{0.49\linewidth}  
        \centering 
        \includegraphics[width=\linewidth]{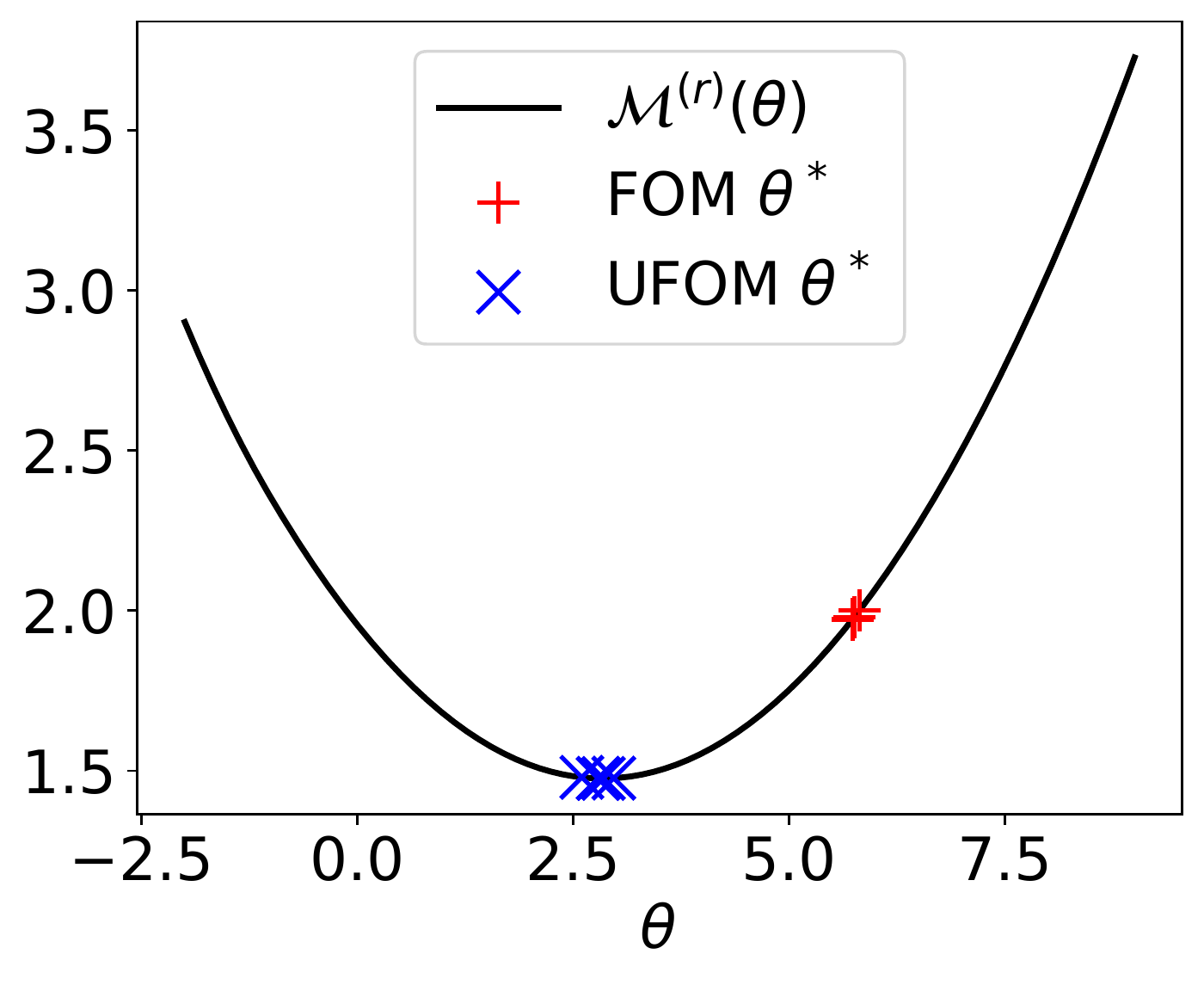}
        \caption[]%
        {{\small \,}}    
        \label{fig:sinth2}
    \end{subfigure}
    \vskip\baselineskip\vspace{-5pt}
    \begin{subfigure}[b]{0.49\linewidth}  
        \centering 
        \includegraphics[width=\linewidth]{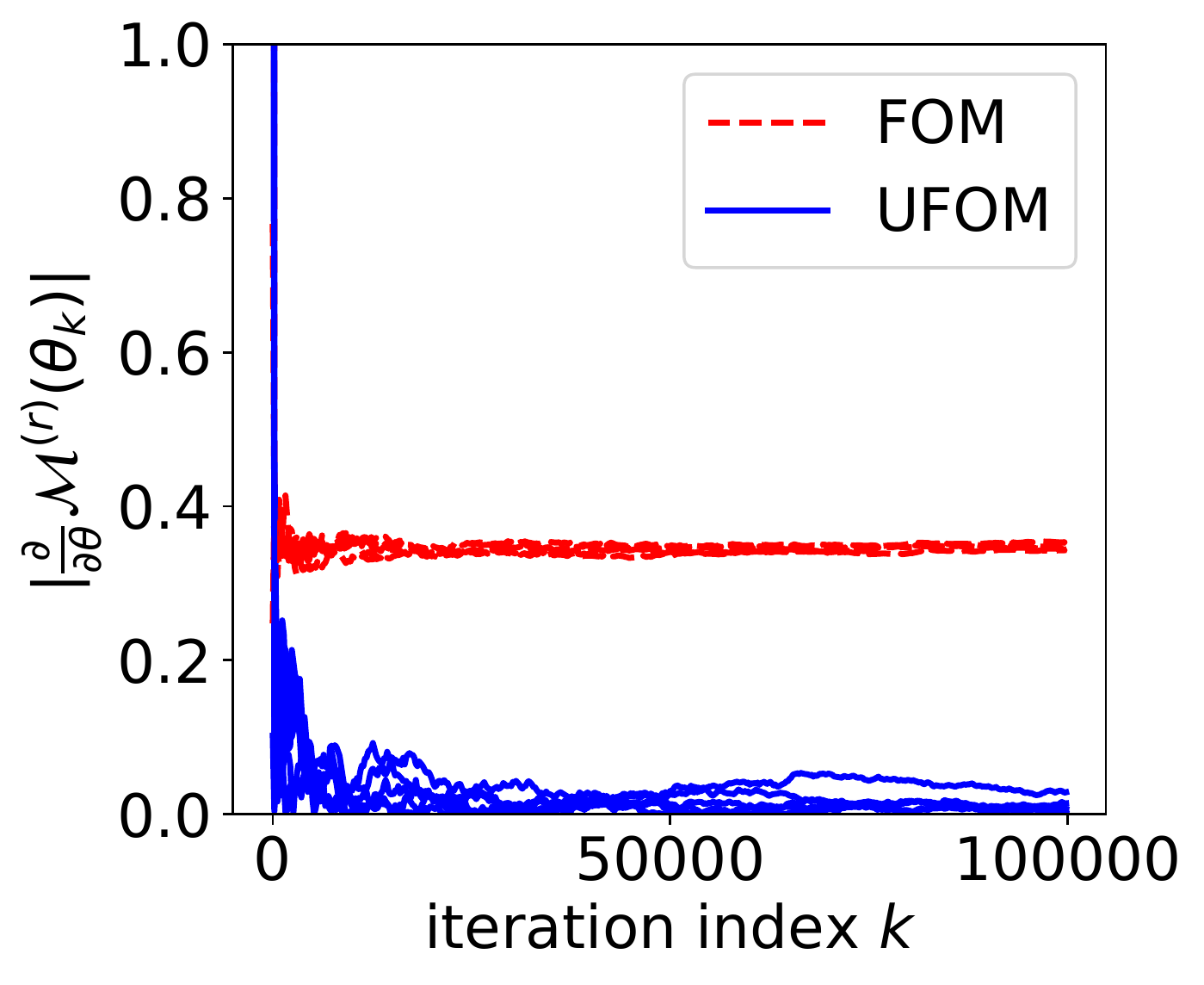}
        \caption[]%
        {{\small \,}}    
        \label{fig:sinth3}
    \end{subfigure}
    \hfill
    \begin{subfigure}[b]{0.49\linewidth}   
        \centering 
        \includegraphics[width=\linewidth]{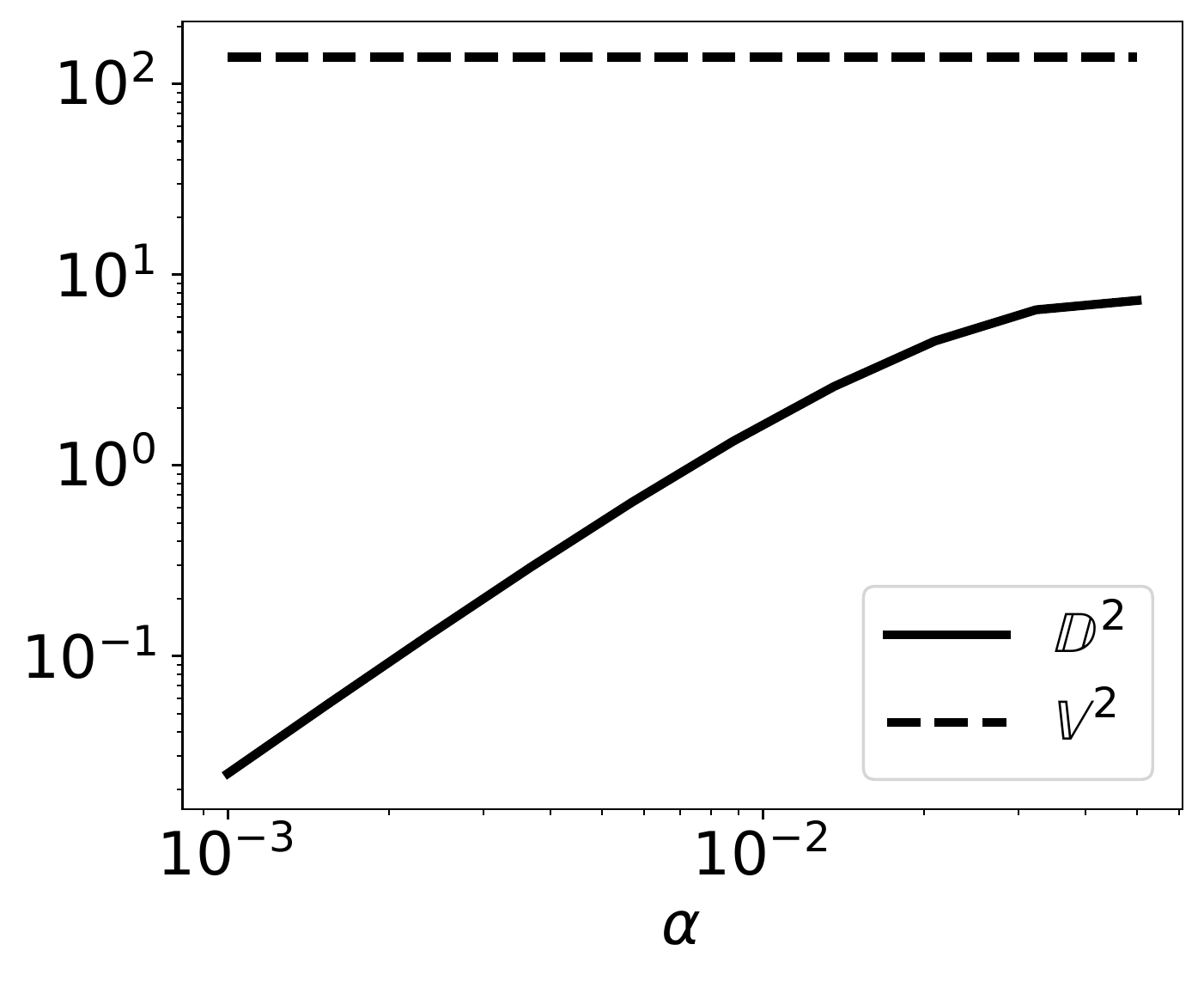}
        \caption[]%
        {{\small \,}}    
        \label{fig:sinth4}
    \end{subfigure}
    \vskip\baselineskip\vspace{-5pt}
    \begin{subfigure}[b]{0.49\linewidth}   
        \centering 
        \includegraphics[width=\linewidth]{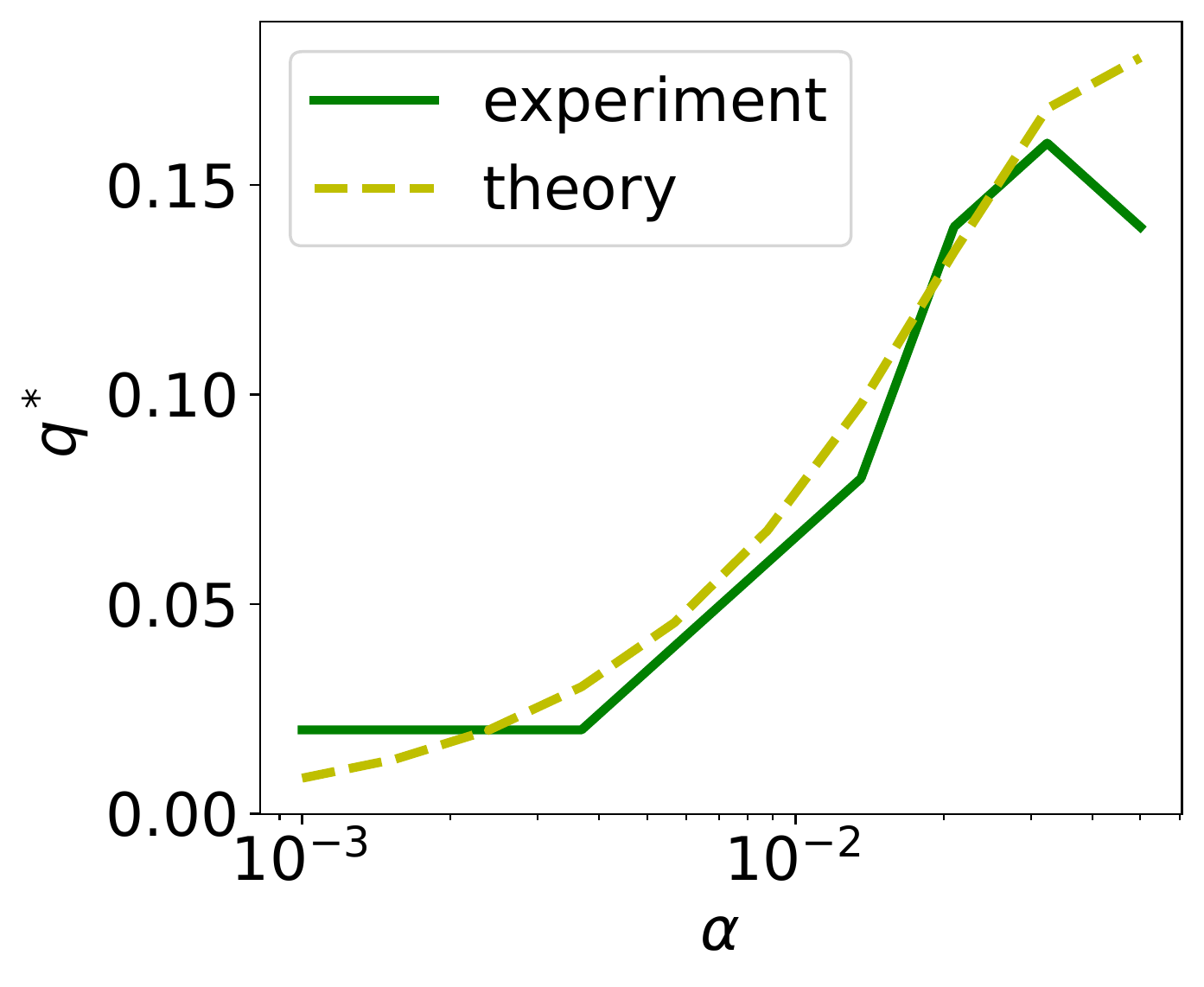}
        \caption[]%
        {{\small \,}}    
        \label{fig:sinth5}
    \end{subfigure}
    \hfill
    \begin{subfigure}[b]{0.49\linewidth}  
        \centering 
        \includegraphics[width=\linewidth]{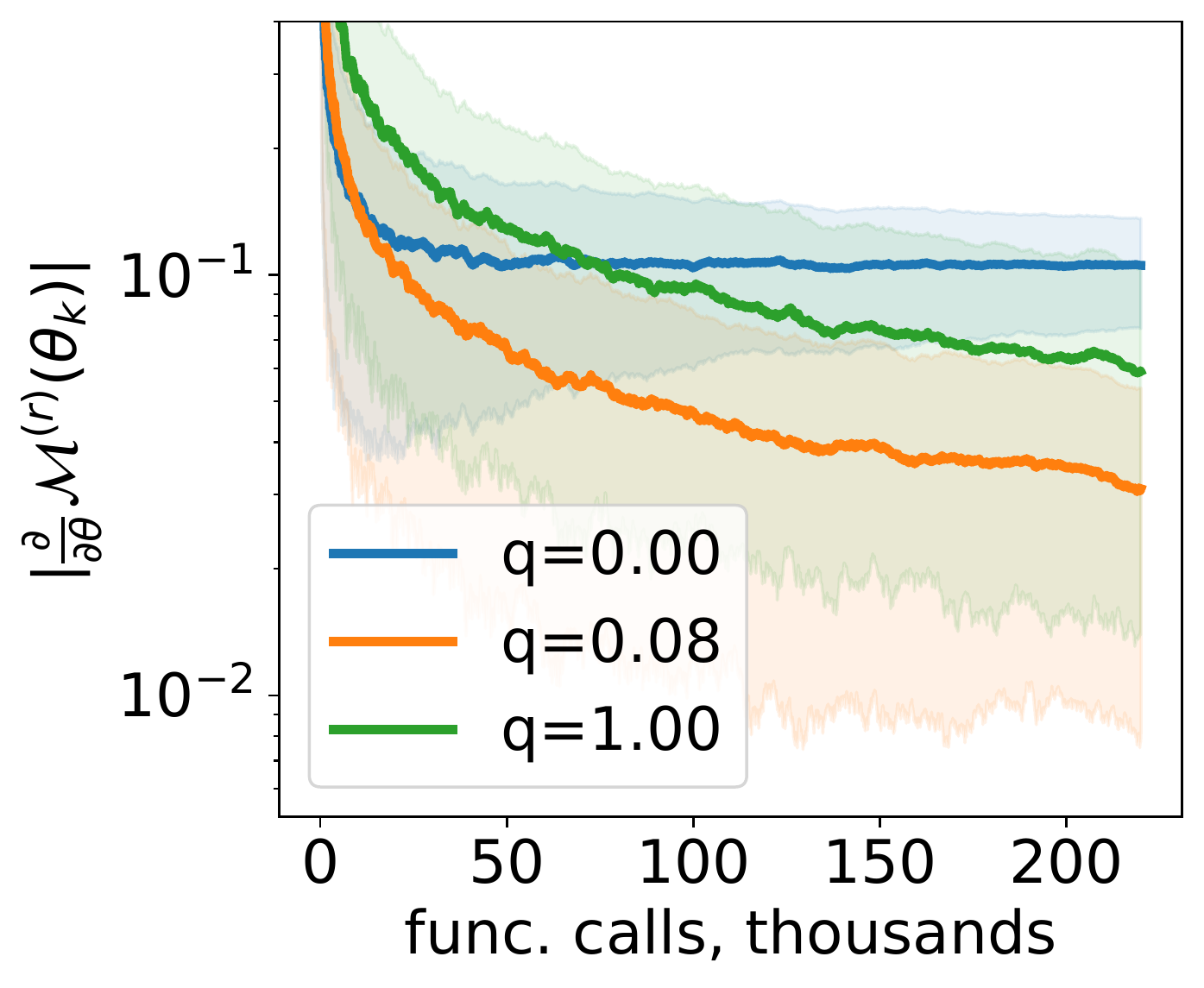}
        \caption[]%
        {{\small \,}}    
        \label{fig:sinth6}
    \end{subfigure}
  \caption{\textbf{(a)} Two tasks are sampled with equal probability from $p(\mathcal{T})$. The goal of task $i$'s inner GD is to optimize convex piecewise-polynomial $f_i (\phi)$, $\phi \in \mathbb{R}$. \textbf{(b)} Resulting ABLO objective $\mathcal{M}^{(r)} (\theta)$ where $\theta$ is a starting $\phi$ value for $(r = 10)$-step inner GD. Markers indicate results of outer SGD from a random starting parameter $\theta_0$ using FOM and UFOM ($q = 0.1$). \textbf{(c)} Convergence of $| \frac{\partial}{\partial \theta} \mathcal{M}^{(r)} (\theta_k) |$, $k$ is an outer-SGD iteration. UFOM is approaching zero gradient norm, which is not true for FOM. \textbf{(d)} We vary inner-GD learning rate $\alpha$ in a range of $[10^{-3}, 5 \cdot 10^{-2}]$ and output numerically approximated $\mathbb{V}^2$ and $\mathbb{D}^2$. \textbf{(e)} Using obtained $\mathbb{V}^2$ and $\mathbb{D}^2$ estimates, we compute theoretically optimal $q^*$ as proposed in Section \ref{sec:optqth} (``theory'' curve). Also, we find optimal $q$ empirically by grid search (``experiment'' curve). \textbf{(f)} Using the setup from Figures (d), (e), we fix $\alpha = 10^{-2}$ and plot optimization curves for $q = 0$ (FOM), $q = 1$ (exact gradient) and $q = q^* \approx 0.08$, where $q^*$ is taken from the ``theory'' plot on Figure (e). We observe the fastest convergence when $q = q^*$.}
  \label{fig:synth}
\end{figure}

\begin{figure*}[t]
    \centering
    \includegraphics[width=\textwidth]{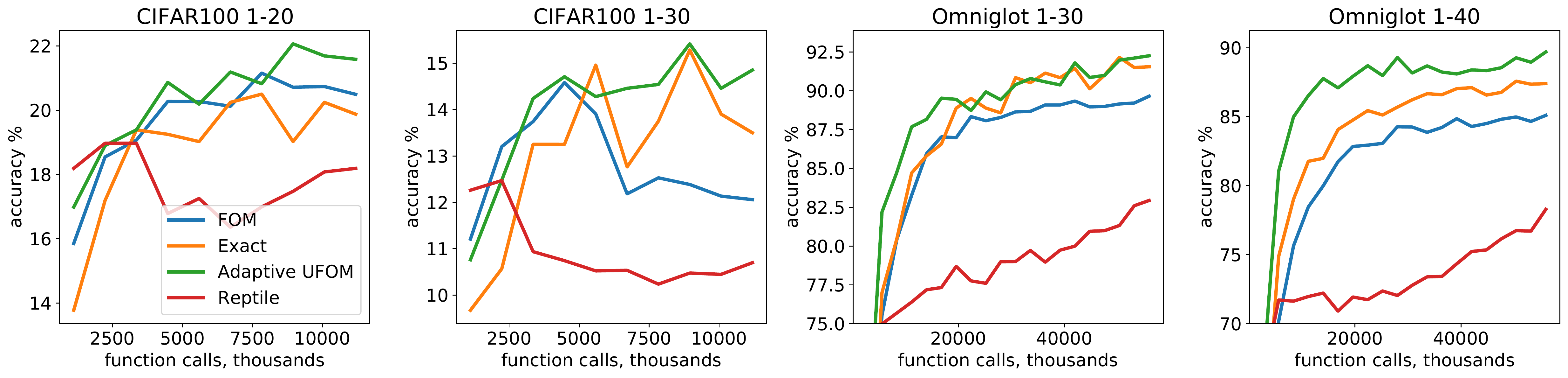}
    \caption{Test-set accuracy curves. $K$-$m$ in the title corresponds to $K$-shot $m$-way setup.}
    \label{fig:maml}
\end{figure*}

\begin{figure}[t]
    \centering
    \includegraphics[width=\linewidth]{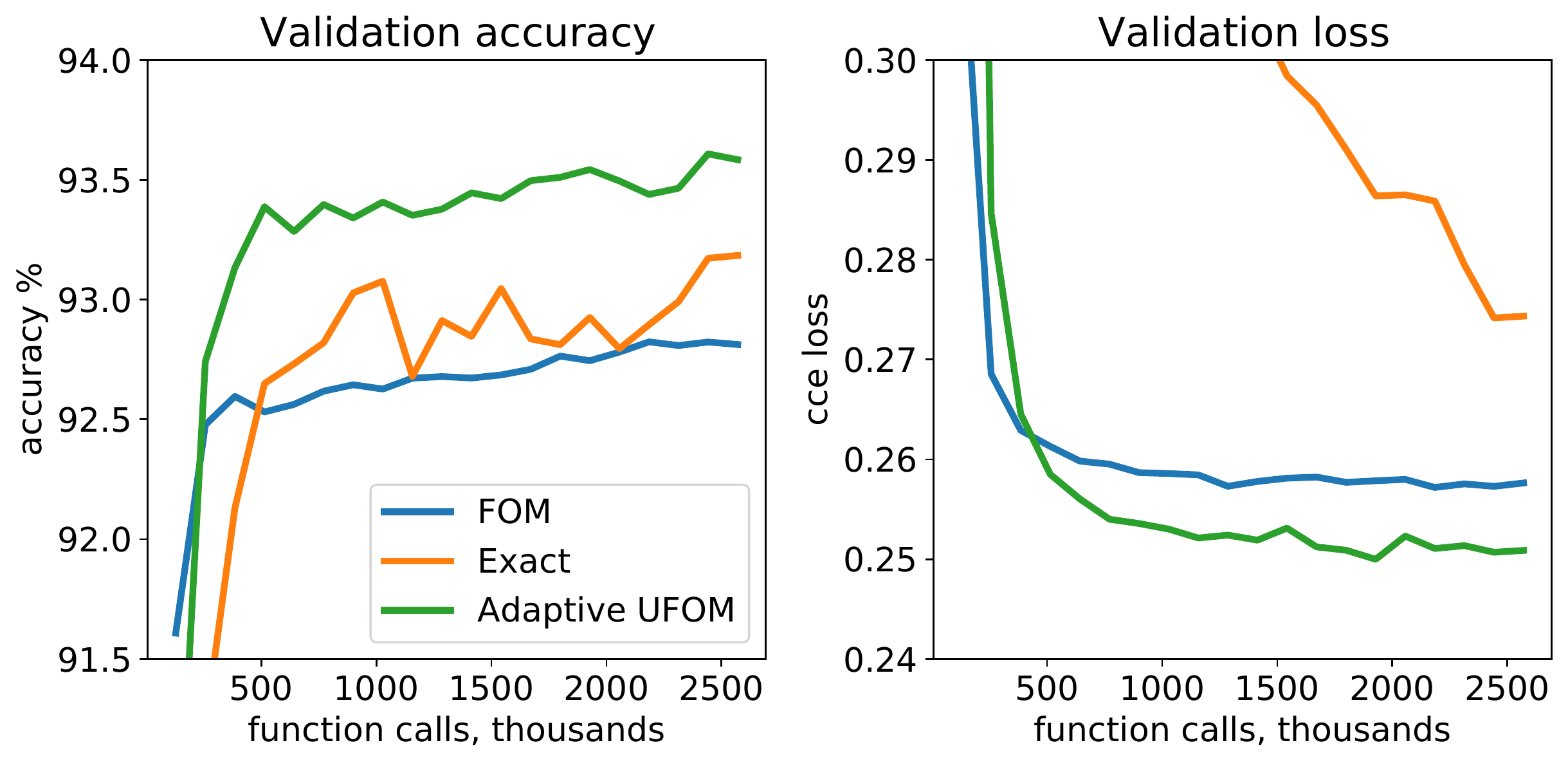}
    \caption{Data hypercleaning, training curves.}
    \label{fig:hc}
\end{figure}

\subsection{Data Hypercleaning}

We evaluate Adaptive UFOM in a hyperparameter optimization problem on MNIST, based on the setup from \citep{truncated,fordiff}. The task is to train a classifier $g (\phi, X_i) \in \mathbb{R}^{10}, X_i \in \mathbb{R}^{784}$, parametrized by $\phi$, on a subset of $5000$ labeled images, where $2500$ labels $Y_i$ have been corrupted. For that, we define $\theta \in \mathbb{R}^{5000}$, $| \Omega_\mathcal{T} | = 1$ and the inner loss $\mathcal{L}^{in}$ has the form $\mathcal{L}^{in} (\theta, \phi, \mathcal{T}) = \sum_{i = 1}^{5000} \sigma( \theta^{(i)} ) l_{CCE} (g (\phi, X_i), Y_i)$, where $\sigma ( \cdot )$ is a sigmoid function, $\theta^{(i)}$ is the $i$th element of $\theta$ and $l_{CCE} (\cdot, Y)$ is a categorical cross entropy (CCE) with respect to a label $Y_i \in \{ 0, \dots, 9 \}$. $\mathcal{L}^{out}$ is defined as a cross entropy on the validation set. As in \citep{truncated}, we set $r = 100$ and $\alpha = 1$. This way, the inner loop trains classifier $g(\phi, X)$, while the outer loop is optimizing weights of each training object. Presumably, $\sigma (\theta^{(i)}) $ should assign the bigger weight to uncorrupted examples.  We modify a setup of \citep{truncated} by using a two- instead of one-layer feedforward network, with ReLU nonlinearity, to obtain a nonconvex optimization problem. For Adaptive UFOM, on a validation score comparison we find that $q_{min} = 0.05, \beta = 0.99$ performs reasonably well. Further, we empirically find that Adaptive UFOM works best when (\ref{eq:ovdvdef}) is modified so that $\overline{\mathbb{D}_k^2} = 0.1 \cdot \mathbb{D}^2_{sm,k} / (1 - \beta^{k_{upd}})$ (dividing initial $\overline{\mathbb{D}_k^2}$ by $10$). We optimize all compared methods for the same amount of function calls -- see more details, including the evolution of adaptive probabilities, in Appendix \ref{sec:hcdet}. Table \ref{tab:testacc_cleaning} and Figure~\ref{fig:hc} demonstate optimization results: as compared to exact gradient and FOM, Adaptive UFOM results in a faster optimization and better generalization. We keep the same value of $q_{min}, \beta$ and modification of (\ref{eq:ovdvdef}) in our subsequent few-shot learning experiments.

\subsection{Few-shot Image Classification} \label{sec:fewshot}

Few-shot image classification \citep{maml} addresses adaptation to a new task when supplied with a small amount of training data. Each task $\mathcal{T}$ is then a pair $\mathcal{T} = ( \mathcal{D}^{tr}_\mathcal{T}, \mathcal{D}^{test}_\mathcal{T} )$, where
\begin{gather*}
    \mathcal{D}^{tr}_\mathcal{T} = ( ( X^{tr}_i, Y^{tr}_i ) )_{i = 1}^t, \quad \mathcal{D}^{test}_\mathcal{T} = ( ( X^{test}_i, Y^{test}_i ) )_{i = 1}^{t'},
\end{gather*}
where $X$ are classified images, $Y \in \{ 1, \dots, m \}$ are labels, $\mathcal{D}^{tr}_\mathcal{T}$ is a training set of a small size $t$, and $\mathcal{D}^{test}_\mathcal{T}$ is a test set of size $t'$. Let $g (\phi, X) \in \mathbb{R}^m$, be a classifier with parameters $\phi$ and input $X$.
Define $\mathcal{L}_\mathrm{CCE} (\phi, \mathcal{D}) = \frac{1}{| \mathcal{D} |} \sum_{(X, Y) \in \mathcal{D}} l_\mathrm{CCE} (g(\phi, X), Y)$. \textit{Model-Agnostic Meta-Learning} (MAML) \citep{maml} states the problem of few-shot classification as ABLO (\ref{eq:maml1}-\ref{eq:opt}) where $V(\theta, \mathcal{T}) \equiv \theta \in \mathbb{R}^p$, $\mathcal{L}^{in} (\theta, \phi, \mathcal{T}) \equiv \mathcal{L}_\mathrm{CCE} (\phi, \mathcal{D}_\mathcal{T}^{tr})$ and $\mathcal{L}^{out} (\theta, \phi, \mathcal{T}) \equiv \mathcal{L}_\mathrm{CCE} (\phi, \mathcal{D}_\mathcal{T}^{test})$. This way, inner GD corresponds to fitting $g (\phi, \cdot)$ to a training set $\mathcal{D}_\mathcal{T}^{tr}$ of a small size, while the outer SGD is searching for an initialization $\theta = \phi_0$ maximizing generalization on the unseen data $\mathcal{D}_\mathcal{T}^{test}$.

We evaluate Adaptive UFOM on CIFAR100 and Omniglot. Both datasets consist of many classes with a few images for each class.
To sample from $p(\mathcal{T})$ in the $K$-shot $m$-way setting, $m$ classes are chosen randomly and $K + 1$ examples are drawn from each class: $K$ examples for training and $1$ for testing, i.e. $s = mK, t = m$. We reuse convolutional architectures for $g(\phi, X)$ from \citep{maml} and set inner-loop length to $r = 10$, as in \citep{reptile}. We also compare with Reptile \citep{reptile} -- another modification of MAML which does not store intermediate states in memory. We run all methods for the same amount of function calls (Figure \ref{fig:maml} and Table \ref{tab:testacc_maml}). Adaptive UFOM shows the best performance in all setups. More details, including adaptive $q$ plots, can be found in Appendix \ref{sec:expdet}.

\begin{table}[t]
\caption{Data hypercleaning, accuracy (\%) and CCE. First row shows accuracy on the uncorrupted part of the train set.}
\label{tab:testacc_cleaning}
\begin{center}
\begin{small}
\begin{sc}
\begin{tabular}{lccc}
\toprule
 & Train & Test & Test CCE \\
\midrule
Exact & 94.07 & 90.82 & 0.4795 \\
FOM & 92.04 & 90.83 & 0.3187 \\
Adaptive UFOM & \textbf{95.27} & \textbf{92.24} & \textbf{0.2811} \\
\bottomrule
\end{tabular}
\end{sc}
\end{small}
\end{center}

\end{table}

\begin{table}[t]
\caption{Test accuracy (\%) in a range of $(K = 1)$-shot setups, with varying $m$-ways. Abbreviations: ``C-100" = CIFAR100 and ``OMNI." = Omniglot.}
\label{tab:testacc_maml}
\begin{center}
\begin{small}
\begin{sc}
\begin{tabular}{lccccc}
\toprule
& $m$ & Exact & Reptile & FOM & Ad.UFOM \\
\midrule
C-100 & 20 & 19.9 & 18.5 & 21.3 & \textbf{22.1} \\
C-100 & 30 & 13.8 & 11.0 & 12.6 & \textbf{14.7} \\
Omni. & 30 & 91.6 & 83.1 & 89.3 & \textbf{92.0} \\
Omni. & 40 & 87.5 & 78.2 & 84.88 & \textbf{89.1} \\
\bottomrule
\end{tabular}
\end{sc}
\end{small}
\end{center}
\end{table}

\section{Conclusion}
\label{sec:conclusion}

We proposed an unbiased first-order method (UFOM), providing a memory-efficient unbiased gradient estimator for approximate bi-level optimization (ABLO). We show that UFOM-based SGD converges to a stationary point of the ABLO problem, and derive an expression for the optimal form of UFOM, which results in faster convergence than the exact gradient and first-order methods. Finally, we propose a method for choosing the parameter of UFOM adaptively.% (Adaptive UFOM). %We evaluate Adaptive UFOM in the task of data hypercleaning and few-shot learning.

\section{Acknowledgements}

We thank John Bronskill for helpful feedback on an early version of the manuscript. We further thank anonymous reviewers for their valuable feedback.

Valerii Likhosherstov acknowledges support from the
Cambridge Trust and DeepMind. Adrian Weller acknowledges support from a Turing AI Fellowship under grant EP/V025379/1, The Alan Turing Institute under EPSRC grant EP/N510129/1 and TU/B/000074, and the Leverhulme Trust via CFI.

\bibliography{references}
\bibliographystyle{icml2021}

\newpage

\onecolumn
\LARGE
\textsc{Debiasing a First-order Heuristic for Approximate Bi-level Optimization:
Supplementary Materials}

\normalsize
%\aistatstitle{Debiasing a First-order Heuristic for Approximate Bi-level Optimization: \\ 
%Supplementary Materials}

\appendix

\section{Related work}
\label{sec:related}
%\textbf{Memory-efficient computation graphs.} Limited and expensive memory is often a bottleneck in modern massive-scale deep learning applications requiring hundreds of GPUs or TPUs employed in the training process simultaneously \citep{xlnet,megatron}. A variety of cross-domain techniques have been adopted to circumvent this issue. For instance, checkpointing \citep{gckpt,ckpt} is a generic solution to memory reduction at the cost of longer running time. A number of deep learning applications benefit from reversible architecture design allowing memory-efficient back-propagation. Among them are hyperparameter optimization \citep{revopt}, image classification \citep{revnet} with residual neural networks, autoregressive \citep{reformer} and flow-based generative modelling \citep{neuralode,nf,graphnf}. Another popular heuristic to save memory during back-propagation, though not always theoretically justified, is truncated back-propagation which is employed in bi-level optimization \citep{truncated}, recurrent neural network (RNN) \citep{tbptt}, Transformer \citep{transformerxl,xlnet} training and generalized meta-learning \citep{gilm}.

\textbf{Unbiased gradient estimation.} Stochastic gradient descent (SGD) \citep{bottou} is an essential component of large-scale machine learning. Unbiased gradient estimation, as a part of SGD, guarantees convergence to a stationary point of the optimization objective. For this reason, many algorithms were proposed to perform unbiased gradient estimation in various applications, e.g. REINFORCE \citep{reinforce} and its low-variance modifications \citep{rebar,muprop} with applications in reinforcement learning and evolution strategies \citep{nes}. The variational autoencoder \citep{autoencoder} and variational dropout \citep{vardrop} are based on a reparametrization trick for unbiased back-propagation through continuous or, involving a relaxation \citep{gumbel,concrete}, discrete random variables. %Similar to this work, \citep{unbtruncated} propose an unbiased version of truncated back-propagation through the RNN. The crucial difference is that \citep{unbtruncated} propose a ``local" correction for each temporal position of the RNN with a stochastic memory reduction, while we propose to correct for the whole outer-loop iteration and manage to obtain a deterministic memory bound which is a better match for the scenario of a \textit{fixed, limited memory budget}.

\textbf{Theory of meta-learning.} Our proof technique fits into the realm of theoretical understanding for meta-learning, which has been explored in \citep{convergence_theory, gcsmaml} for nonconvex functions (see also \cite{seff} for convergence analysis in certain bi-level optimization setups), as well as \citep{online_metalearning, provable_guarantees} for convex functions and their extensions, such as online convex optimization \citep{oco_book}. While \citep{convergence_theory} provides a brief counterexample for which $(r = 1)$-step FOM does not converge, we establish a rigorous non-convergence counterexample proof for FOM with any number of steps $r$ when using \textit{stochastic} gradient descent. Our proof is based on arguments using expectations and probabilities, providing new insights into stochastic optimization during meta-learning. Furthermore, while \citep{gcsmaml} touches on the \textit{zero-order} case found in \citep{esmaml}, which is mainly focused on reinforcement learning, our work studies the case where exact gradients are available, which is suited for supervised learning.

\section{Synthetic Experiment Details} \label{sec:synth}

For the synthetic experiments shown in Figures \ref{fig:sinth1}, \ref{fig:sinth2}, and \ref{fig:sinth3}, we set the following parameters from Theorem \ref{th:conv} and proof of Theorem \ref{th:counter}:
\begin{gather*}
    r = 10, \quad \alpha = 0.1, \quad q = 0.1 \; \text{(UFOM)}, \quad \forall k \in \mathbb{N}: \gamma_k = \frac{10}{k}, \\
    a_1 = 0.5, \quad a_2 = 1.5, \quad b_1 = 0, \quad b_2 = 17.39, \quad A = 12.59
\end{gather*}
($b_2$ and $A$ values are obtained by setting $D = 0.06$ in the Theorem \ref{th:counter} proof).
We do 5 simulations for FOM and UFOM, where we sample $\theta_0$ from a uniform distribution on a segment $[-10, 30]$.

To demonstrate a wider range of $q^*$ values, for Figures \ref{fig:sinth4}, \ref{fig:sinth5}, and \ref{fig:sinth6}, we opt for a slightly different set of parameters:
\begin{equation*}
    r = 10, \quad \alpha = 0.1, \quad \forall k \in \mathbb{N}: \gamma_k = \frac{10}{k}, \quad a_1 = 0.5, \quad a_2 = 1.5, \quad b_2 = 10, \quad A = 10.
\end{equation*}

To approximate $\mathbb{V}^2, \mathbb{D}^2$ on Figure \ref{fig:sinth4}, we find a maximal value of the corresponding expectation (computed precisely for two tasks with equal probability) on a grid of 10000 $\theta$ values on $[-50, 50]$.

To output the ``experiment'' curve on Figure \ref{fig:sinth5}, for each value of $\alpha$, we search for $q$ on a grid of $20$ elements between $0.02$ and $0.4$. For each $q$ on a grid we simulate $10000$ SGD loops ($100$ iterations each) from a starting point drawn uniformly on $[-50, 50]$. Then we compute the average of the curves corresponding to $| \frac{\partial}{\partial \theta} \mathcal{M}^{(p)} |$ for these $10000$ simulations. Given the best $q$ for each $\alpha$, we choose the one which achieves the minimal average value of $| \frac{\partial}{\partial \theta} \mathcal{M}^{(p)} |$ in the fastest time, computed for $q = 0.02$.

For Figure \ref{fig:sinth6}, we report mean and standard error over $1000$ curves starting from a point drawn uniformly on $[-50, 50]$.

\section{Data Hypercleaning Details} \label{sec:hcdet}

Validation loss and $\phi_0$ do not depend on $\theta$, and thus to use UFOM, we include the last inner optimization step into the definition of $\mathcal{L}^{out}$. This implies that $\mathcal{L}^{out}$ depends on $\theta$. We partition the original MNIST train set into sets of size 5000 for the hyperclearning task's train and validation sets. We use the MNIST test set for testing. We corrupt half of training examples by drawing labels uniformly from ${0, \dots, 9}$. For the classifier, we use a 2-layer feedforward network with dimensions $784 \to 256 \to 10$, with ReLU nonlinearities. We use Adam \citep{adam} as an outer-loop optimizer with a learning rate of $0.1$. We run all methods for a number of function calls equivalent to $500$ outer-loop iterations for the memory-efficient exact gradient.

Figure \ref{fig:hc_probs} demonstrates adaptively chosen $q$ during optimization using Adaptive UFOM. We observe that $q$ stabilizes soon in the beginning of optimization and doesn't change much during training. This could mean that statistics $\overline{\mathbb{D}^2}, \overline{\mathbb{V}^2}$ are roughly the same along the whole optimization trajectory.

\begin{figure}
    \centering
    \includegraphics[width=0.25\textwidth]{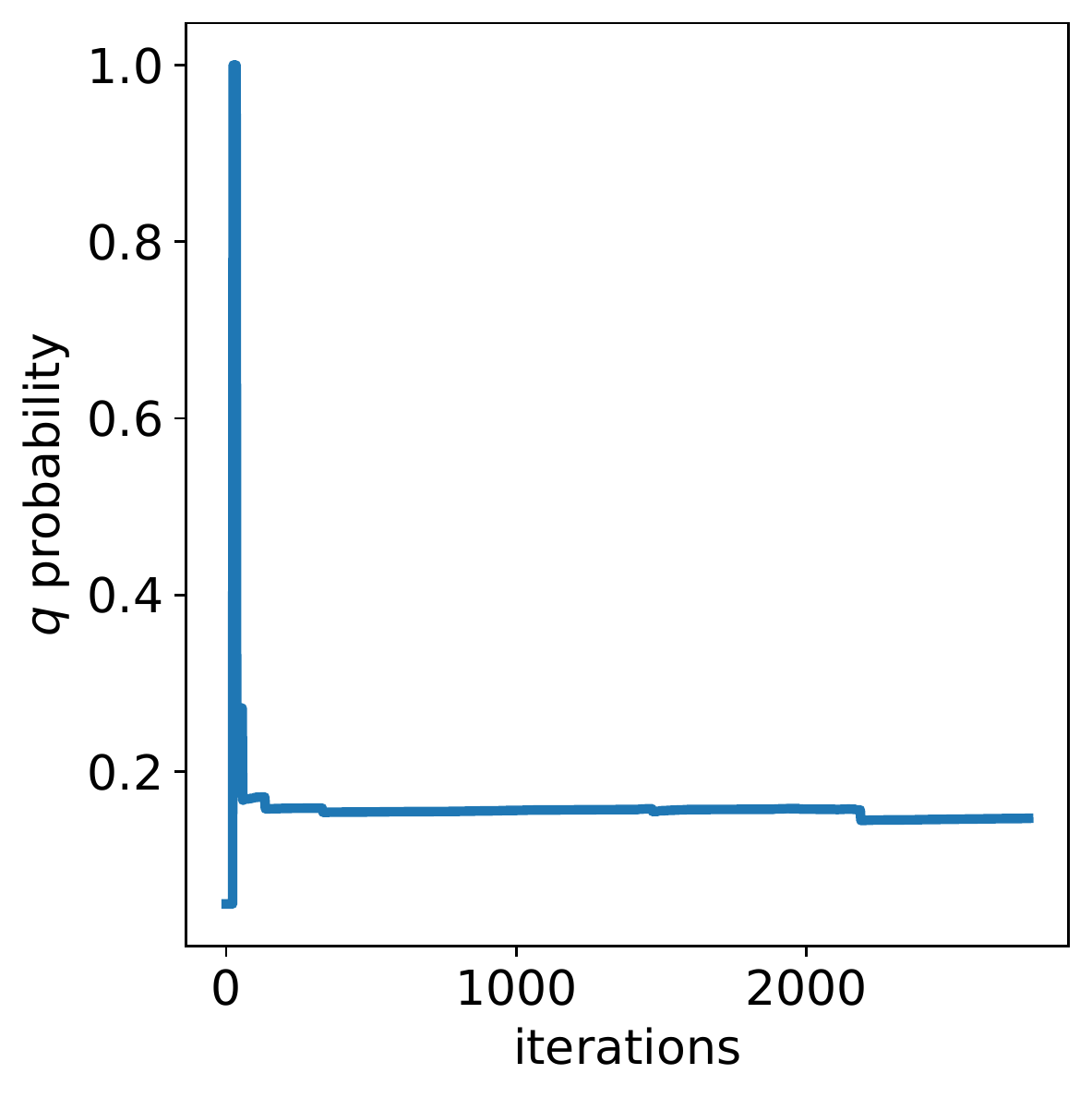}
    \caption{Adaptive $q$ probabilities during training for the hypercleaning setup.}
    \label{fig:hc_probs}
\end{figure}

\section{Few-Shot Learning Details} \label{sec:expdet}

All results are reported in a transductive setting \citep{reptile}. In all setups for Reptile, we reuse the code from \citep{reptile}. For the exact ABLO and Adaptive UFOM, we clip each entry of the gradient to be in $[-0.1, 0.1]$. We use the following hyperparameters for the two datasets:
\begin{itemize}
    \item \textbf{Omniglot}. We run all methods for the number of function calls equivalent to $\tau = 200000$ outer iterations of the memory-efficient exact ABLO. For exact ABLO/FOM/Adaptive UFOM, we set $\forall k: \gamma_k = 0.1$, meta-batch size of 5, $\alpha = 0.005$. In all setups for Reptile, we set hyperparameter values to be equal to the ones found in the 1-shot 20-way case \citep{reptile}. This is because Reptile underperforms if its hyperparameters are set to the values used for exact ABLO/FOM/Adaptive UFOM. We take train and test splits as in \citep{maml,reptile}.

    \item \textbf{CIFAR100}. We use the same hyperparameters as in Omniglot setup, but run all methods for the number of function calls equivalent to $\tau = 40000$ outer iterations of the memory-efficient exact ABLO. For a train-test split, we combine CIFAR100's train and test sets and randomly split classes into 80 train and 20 test classes.
\end{itemize}

Figure \ref{fig:probs} demonstrates adaptively chosen $q$ during optimization using Adaptive UFOM. Again, we observe that $q$ stabilizes and doesn't change much during training, most probably meaning that statistics $\overline{\mathbb{D}^2}, \overline{\mathbb{V}^2}$ are roughly the same along the optimization trajectory.

\begin{figure}
    \centering
    \includegraphics[width=\textwidth]{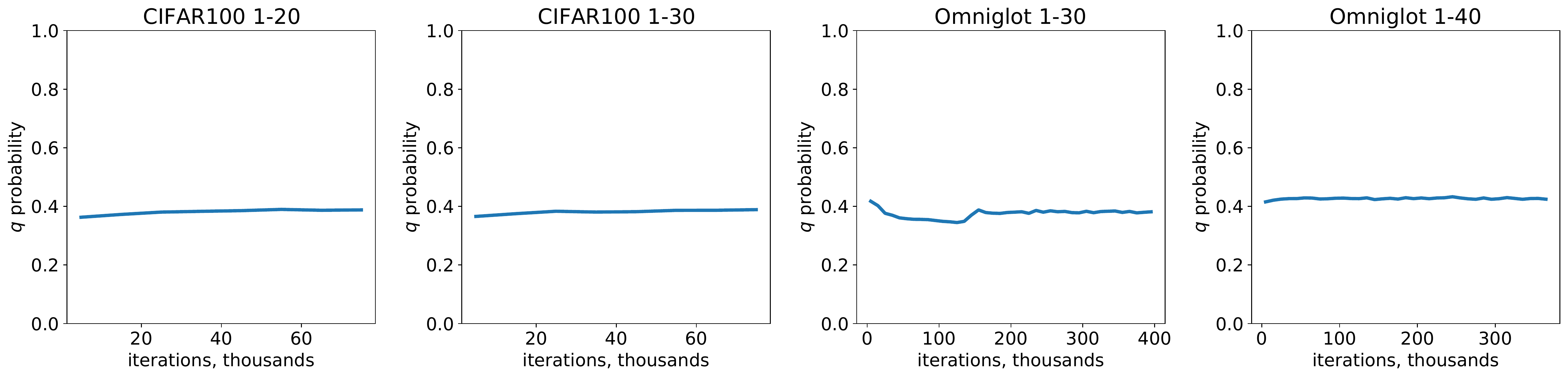}
    \caption{Adaptive $q$ probabilities during training for the few-shot learning setup.}
    \label{fig:probs}
\end{figure}

\section{Proofs}
\label{sec:proofs}

In this section, we provide proofs for Theorems \ref{th:conv} and \ref{th:counter} from the main body of the paper.

\subsection{Theorem \ref{th:conv}}

We start by formulating and proving three helpful lemmas. In proofs we use the fact that, as a direct consequence of Assumption \ref{as:liphes}, for all $\theta \in \mathbb{R}^s, \phi \in \mathbb{R}^p, \mathcal{T} \in \Omega_\mathcal{T}$
\begin{equation*}
    \max ( \| \frac{\partial^2}{\partial \theta \partial \phi} \mathcal{L}^{in} (\theta, \phi, \mathcal{T}) \|_2, \| \frac{\partial^2}{\partial \phi^2} \mathcal{L}^{in} (\theta, \phi, \mathcal{T}) \|_2 ) \leq L_2 .
\end{equation*}
\begin{lemma} \label{lemma:unbbnd}
Let $p, r, s \in \mathbb{N}$, $\{ \alpha_j > 0 \}_{j = 1}^\infty$ be any sequence, $q \in (0, 1]$, $p(\mathcal{T})$ be a distribution on a nonempty set $\Omega_\mathcal{T}$, $\xi \sim \mathrm{Bernoulli} (q)$ be independent of $p(\mathcal{T})$, $V : \mathbb{R}^s \times \Omega_\mathcal{T} \to \mathbb{R}^p, \mathcal{L}^{in}, \mathcal{L}^{out} : \mathbb{R}^s \times \mathbb{R}^p \times \Omega_\mathcal{T} \to \mathbb{R}$ be functions satisfying Assumption \ref{as:liphes}, and let $U^{(r)}: \mathbb{R}^s \times \Omega_\mathcal{T} \to \mathbb{R}^p$ be defined according to (\ref{eq:maml1}-\ref{eq:maml}), $\mathcal{M}^{(r)} : \mathbb{R}^s \to \mathbb{R}$ be defined according to (\ref{eq:opt}) and satisfy Assumption \ref{as:mreg}. Define $\mathcal{G}_{FO}: \mathbb{R}^s \times \Omega_\mathcal{T} \to \mathbb{R}^s$ and $\mathcal{G}: \mathbb{R}^s \times \Omega_\mathcal{T} \times \{ 0, 1 \} \to \mathbb{R}^s$ as
\begin{gather*}
    \mathcal{G}_{FO} (\theta, \mathcal{T}) = \frac{\partial}{\partial \theta} \mathcal{L}^{out} (\theta, \phi_r, \mathcal{T}) + (\frac{\partial}{\partial \theta} V (\theta, \mathcal{T}))^\top \frac{\partial}{\partial \phi} \mathcal{L}^{out} (\theta, \phi_r, \mathcal{T}), \quad \phi_r = U^{(r)} (\theta, \mathcal{T}), \\
    \mathcal{G} (\theta, \mathcal{T}, x) = \mathcal{G}_{FO} (\theta, \mathcal{T}) + \frac{x}{q} (\nabla_\theta \mathcal{L}^{out} (\theta, \phi_r, \mathcal{T}) - \mathcal{G}_{FO} (\theta, \mathcal{T})) .
\end{gather*}
Then
\begin{gather}
    \mathbb{D}^2 \leq \mathbb{D}^2_{bound}, \label{eq:fodiff} \\
    \mathbb{V}^2 \leq \mathbb{V}_{bound}^2, \quad \mathbb{V}_{bound} = L_1 + M_1 L_1 \prod_{j = 1}^r (1 + \alpha_j L_2) + L_1 L_2 \sum_{j = 1}^r \alpha_j \prod_{j' = j}^r (1 + \alpha_{j'} L_2), \label{eq:varub}
\end{gather}
where $\mathbb{D}, \mathbb{V}, \mathbb{D}_{bound}$ are defined in (\ref{eq:dtruedef}), (\ref{eq:vtruedef}), (\ref{eq:ddef}) respectively. Further, for all $\theta \in \mathbb{R}^s$
\begin{gather}
    \mathbb{E}_{\xi, p(\mathcal{T})} \left[\mathcal{G} (\theta, \mathcal{T}, \xi) \right]= \frac{\partial}{\partial \theta} \mathcal{M}^{(r)} (\theta), \label{eq:unbgrad} \\
    \mathbb{E}_{\xi, p(\mathcal{T})} \left[\| \mathcal{G} (\theta, \mathcal{T}, \xi) \|_2^2 \right] \leq (\frac{1}{q} - 1) \mathbb{D}^2 + \mathbb{V}^2 . \label{eq:bndvar}
\end{gather}
\end{lemma}
\begin{proof}
First, we show (\ref{eq:fodiff}). Let $\phi_0, \dots, \phi_r$ be inner-GD rollouts (\ref{eq:maml1}-\ref{eq:maml}) corresponding to $\theta$ and $\mathcal{T}$. Observe that by Assumption \ref{as:liphes}
\begin{equation*}
    \| \frac{\partial}{\partial \phi} \mathcal{L}^{out} (\theta, \phi_r, \mathcal{T}) \|_2 = \| \nabla_{\phi_r} \mathcal{L}^{out} (\theta, \phi_r, \mathcal{T}) \|_2 \leq L_1
\end{equation*}
and according to (\ref{eq:gradmid}-\ref{eq:gradlast}) for each $1 \leq j \leq r$
\begin{align}
    \| \nabla_{\phi_{j - 1}} \mathcal{L}^{out} (\theta, \phi_r, \mathcal{T}) \|_2 &= \| \nabla_{\phi_j} \mathcal{L}^{out} (\theta, \phi_r, \mathcal{T}) - \alpha_j \biggl( \frac{\partial^2}{\partial \phi^2} \mathcal{L}^{in} (\theta, \phi_{j - 1}, \mathcal{T}) \biggr)^\top \nabla_{\phi_j} \mathcal{L}^{out} (\theta, \phi_r, \mathcal{T}) \|_2 \nonumber \\
    &= \| \nabla_{\phi_j} \mathcal{L}^{out} (\theta, \phi_r, \mathcal{T}) \|_2 + \alpha_j \| \frac{\partial^2}{\partial \phi^2} \mathcal{L}^{in} (\theta, \phi_{j - 1}, \mathcal{T}) \|_2 \| \nabla_{\phi_j} \mathcal{L}^{out} (\theta, \phi_r, \mathcal{T}) \|_2 \nonumber \\
    &\leq (1 + \alpha_j L_2) \| \nabla_{\phi_j} \mathcal{L}^{out} (\theta, \phi_r, \mathcal{T}) \|_2 \leq \dots \leq \| \nabla_{\phi_r} \mathcal{L}^{out} (\theta, \phi_r, \mathcal{T}) \|_2 \prod_{j' = j}^r (1 + \alpha_{j'} L_2) \nonumber \\
    &\leq L_1 \prod_{j' = j}^r (1 + \alpha_{j'} L_2). \label{eq:egradub}
\end{align}

In addition, we deduce that
\begin{align*}
    &\| \nabla_{\phi_0} \mathcal{L}^{out} (\theta, \phi_r, \mathcal{T}) - \frac{\partial}{\partial \phi} \mathcal{L}^{out} (\theta, \phi_r, \mathcal{T}) \|_2 = \| \nabla_{\phi_1} \mathcal{L}^{out} (\theta, \phi_r, \mathcal{T}) - \alpha_1 \biggl( \frac{\partial^2}{\partial \phi^2} \mathcal{L}^{in} (\theta, \phi_0, \mathcal{T}) \biggr)^\top \nabla_{\phi_1} \mathcal{L}^{out} (\theta, \phi_r, \mathcal{T}) \\
    &- \frac{\partial}{\partial \phi} \mathcal{L}^{out} (\theta, \phi_r, \mathcal{T}) \|_2 \\
    &\leq \| \nabla_{\phi_1} \mathcal{L}^{out} (\theta, \phi_r, \mathcal{T}) - \frac{\partial}{\partial \phi} \mathcal{L}^{out} (\theta, \phi_r, \mathcal{T}) \|_2 + \alpha_1 \| \frac{\partial^2}{\partial \phi^2} \mathcal{L}^{in} (\theta, \phi_0, \mathcal{T}) \|_2 \| \nabla_{\phi_1} \mathcal{L}^{out} (\theta, \phi_r, \mathcal{T}) \|_2 \\
    &\leq \| \nabla_{\phi_1} \mathcal{L}^{out} (\theta, \phi_r, \mathcal{T}) - \frac{\partial}{\partial \phi} \mathcal{L}^{out} (\theta, \phi_r, \mathcal{T}) \|_2 + \alpha_1 L_1 L_2 \prod_{j' = 1}^r (1 + \alpha_{j'} L_2) \leq \dots \\
    &\leq \| \nabla_{\phi_r} \mathcal{L}^{out} (\theta, \phi_r, \mathcal{T}) - \frac{\partial}{\partial \phi} \mathcal{L}^{out} (\theta, \phi_r, \mathcal{T}) \|_2 + L_1 L_2 \sum_{j = 1}^r \alpha_j \prod_{j' = j}^r (1 + \alpha_{j'} L_2) = L_1 L_2 \sum_{j = 1}^r \alpha_j \prod_{j' = j}^r (1 + \alpha_{j'} L_2) .
\end{align*}

Then:
\begin{align*}
    &\| \mathcal{G}_{FO} (\theta, \mathcal{T}) - \nabla_\theta \mathcal{L}^{out} (\theta, U^{(r)} (\theta, \mathcal{T}), \mathcal{T}) \|_2 = \| \biggl( \frac{\partial}{\partial \theta} V(\theta, \mathcal{T}) \biggr)^\top (\nabla_{\phi_0} \mathcal{L}^{out} (\theta, \phi_r, \mathcal{T}) - \frac{\partial}{\partial \phi} \mathcal{L}^{out} (\theta, \phi_r, \mathcal{T})) \\
    &- \sum_{j = 1}^r \alpha_j \biggl( \frac{\partial^2}{\partial \theta \partial \phi} \mathcal{L}^{in} (\theta, \phi_{j - 1}, \mathcal{T}) \biggr)^\top \nabla_{\phi_j} \mathcal{L}^{out} (\theta, \phi_r, \mathcal{T}) \|_2 \\
    &\leq \| \frac{\partial}{\partial \theta} V(\theta, \mathcal{T}) \|_2 \| \nabla_{\phi_0} \mathcal{L}^{out} (\theta, \phi_r, \mathcal{T}) - \frac{\partial}{\partial \phi} \mathcal{L}^{out} (\theta, \phi_r, \mathcal{T}) \|_2 + \sum_{j = 1}^r \alpha_j \| \frac{\partial^2}{\partial \theta \partial \phi} \mathcal{L}^{in} (\theta, \phi_{j - 1}, \mathcal{T}) \|_2 \\
    &\cdot \| \nabla_{\phi_j} \mathcal{L}^{out} (\theta, \phi_r, \mathcal{T}) \|_2 \\
    &\leq M_1 L_1 L_2 \sum_{j = 1}^r \alpha_j \prod_{j' = j}^r (1 + \alpha_{j'} L_2) + \sum_{j = 1}^r \alpha_j L_2 L_1 \prod_{j' = j}^r (1 + \alpha_{j'} L_2) = \mathbb{D}_{bound} .
\end{align*}
Hence, for each $\theta \in \mathbb{R}^s$ $\mathbb{E}_{p(\mathcal{T})} \left[ \| \mathcal{G}_{FO} (\theta, \mathcal{T}) - \nabla_\theta \mathcal{L}^{out} (\theta, U^{(r)} (\theta, \mathcal{T}), \mathcal{T}) \|_2^2 \right]$ is well-defined and bounded by $\mathbb{D}_{bound}^2$. Therefore, $\mathbb{D}^2$ is well-defined and bounded by $\mathbb{D}^2_{bound}$.

Next, we show (\ref{eq:varub}). From (\ref{eq:egradub}) it follows that
\begin{align}
    &\| \nabla_\theta \mathcal{L}^{out} (\theta, U^{(r)} (\theta, \mathcal{T}), \mathcal{T}) \|_2 = \| \frac{\partial}{\partial \theta} \mathcal{L}^{out} (\theta, \phi_r, \mathcal{T}) + \biggl( \frac{\partial}{\partial \theta} V(\theta, \mathcal{T}) \biggr)^\top \nabla_{\phi_0} \mathcal{L}^{out} (\theta, \phi_r, \mathcal{T}) \nonumber \\
    &- \sum_{j = 1}^r \alpha_j \biggl( \frac{\partial^2}{\partial \theta \partial \phi} \mathcal{L}^{in} (\theta, \phi_{j - 1}, \mathcal{T}) \biggr)^\top \nabla_{\phi_j} \mathcal{L}^{out} (\theta, \phi_r, \mathcal{T}) \|_2 \nonumber \\
    &\leq \| \frac{\partial}{\partial \theta} \mathcal{L}^{out} (\theta, \phi_r, \mathcal{T}) \|_2 + \| \frac{\partial}{\partial \theta} V(\theta, \mathcal{T}) \|_2 \| \nabla_{\phi_0} \mathcal{L}^{out} (\theta, \phi_r, \mathcal{T}) \|_2 + \sum_{j = 1}^r \alpha_j \| \frac{\partial^2}{\partial \theta \partial \phi} \mathcal{L}^{in} (\theta, \phi_{j - 1}, \mathcal{T}) \|_2 \nonumber \\
    &\cdot \| \nabla_{\phi_j} \mathcal{L}^{out} (\theta, \phi_r, \mathcal{T}) \|_2 \nonumber \\
    &\leq L_1 + M_1 L_1 \prod_{j = 1}^r (1 + \alpha_j L_2) + \sum_{j = 1}^r \alpha_j L_2 L_1 \prod_{j' = j}^r (1 + \alpha_{j'} L_2) = \mathbb{V}_{bound}. \nonumber %\label{eq:tegradub}
\end{align}
Hence, for each $\theta \in \mathbb{R}^s$ $\mathbb{E}_{p (\mathcal{T})} \left[ \| \nabla_\theta \mathcal{L}^{out} (\theta, U^{(r)} (\theta, \mathcal{T}), \mathcal{T}) \|_2^2 \right]$ is well-defined and bounded by $\mathbb{V}^2_{bound}$ for all $\theta \in \mathbb{R}^s$. Consequently, $\mathbb{V}^2$ is well-defined and is also bounded by $\mathbb{V}^2_{bound}$.

(\ref{eq:unbgrad}) is satisfied by observing that
\begin{align*}
    \mathbb{E}_{\xi, p(\mathcal{T})} \left[\mathcal{G} (\theta, \mathcal{T}, \xi) \right] &= \mathbb{E}_{p(\mathcal{T})} \biggl[ \mathbb{E}_\xi \left[\mathcal{G}(\theta, \mathcal{T}, \xi) \right] \biggr] \\
    &= \mathbb{E}_{p(\mathcal{T})} \biggl[ \mathcal{G}_{FO} (\theta, \mathcal{T}) + \frac{q}{q} (\nabla_\theta \mathcal{L}^{out} (\theta, \phi_r, \mathcal{T}) - \mathcal{G}_{FO} (\theta, \mathcal{T}) ) \biggr] \\
    &= \mathbb{E}_{p(\mathcal{T})} \biggl[ \nabla_\theta \mathcal{L}^{out} (\theta, \phi_r, \mathcal{T}) \biggr] = \mathbb{E}_{p(\mathcal{T})} \biggl[ \nabla_\theta \mathcal{L}^{out} (\theta, U^{(r)} (\theta, \mathcal{T}), \mathcal{T}) \biggr] \\
    &= \nabla_\theta \mathbb{E}_{p(\mathcal{T})} \biggl[ \mathcal{L}^{out} (\theta, U^{(r)} (\theta, \mathcal{T}), \mathcal{T}) \biggr] = \frac{\partial}{\partial \theta} \mathcal{M}^{(r)} (\theta).
\end{align*}

To show (\ref{eq:bndvar}), we fix $\mathcal{T} \in \Omega_\mathcal{T}$. Let $\phi_0, \dots, \phi_r$ be inner-GD rollouts (\ref{eq:maml1}-\ref{eq:maml}) corresponding to $\theta$ and $\mathcal{T}$.

Next:
\begin{align*}
    \| \mathcal{G} (\theta, \mathcal{T}, \xi) \|_2 &\leq \| \mathcal{G} (\theta, \mathcal{T}, \xi) - \nabla_\theta \mathcal{L}^{out} (\theta, U^{(r)} (\theta, \mathcal{T}), \mathcal{T}) \|_2 + \| \nabla_\theta \mathcal{L}^{out} (\theta, U^{(r)} (\theta, \mathcal{T}), \mathcal{T}) \|_2 \\
    &= (1 - \frac{x}{q}) \| \mathcal{G}_{FO} (\theta, \mathcal{T}) - \nabla_\theta \mathcal{L}^{out} (\theta, U^{(r)} (\theta, \mathcal{T}), \mathcal{T}) \|_2 + \| \nabla_\theta \mathcal{L}^{out} (\theta, U^{(r)} (\theta, \mathcal{T}), \mathcal{T}) \|_2 .
\end{align*}

Take square and then expectation:
\begin{align*}
    &\mathbb{E}_{\xi, p(\mathcal{T})} \left[ \| \mathcal{G} (\theta, \mathcal{T}, \xi) \|_2^2 \right] = \mathbb{E}_{p(\mathcal{T})} \mathbb{E}_\xi \left[ \| \mathcal{G} (\theta, \mathcal{T}, \xi) \|_2^2 \right] \\
    &\leq \mathbb{E}_{p(\mathcal{T})} \mathbb{E}_\xi \biggl[ (1 - 2 \frac{\xi}{q} + \frac{\xi^2}{q^2}) \| \mathcal{G}_{FO} (\theta, \mathcal{T}) - \nabla_\theta \mathcal{L}^{out} (\theta, U^{(r)} (\theta, \mathcal{T}), \mathcal{T}) \|_2^2 \\
    &+ 2 (1 - \frac{\xi}{q}) \| \mathcal{G}_{FO} (\theta, \mathcal{T}) - \nabla_\theta \mathcal{L}^{out} (\theta, U^{(r)} (\theta, \mathcal{T}), \mathcal{T}) \|_2 \| \nabla_\theta \mathcal{L}^{out} (\theta, U^{(r)} (\theta, \mathcal{T}), \mathcal{T}) \|_2 \\
    &+ \| \nabla_\theta \mathcal{L}^{out} (\theta, U^{(r)} (\theta, \mathcal{T}), \mathcal{T}) \|_2^2 \biggr] \\
    &= \mathbb{E}_{p(\mathcal{T})} \biggl[ (\frac{1}{q} - 1) \| \mathcal{G}_{FO} (\theta, \mathcal{T}) - \nabla_\theta \mathcal{L}^{out} (\theta, U^{(r)} (\theta, \mathcal{T}), \mathcal{T}) \|_2^2 \\
    &+ 2 \cdot 0 \cdot \| \mathcal{G}_{FO} (\theta, \mathcal{T}) - \nabla_\theta \mathcal{L}^{out} (\theta, U^{(r)} (\theta, \mathcal{T}), \mathcal{T}) \|_2 \| \nabla_\theta \mathcal{L}^{out} (\theta, U^{(r)} (\theta, \mathcal{T}), \mathcal{T}) \|_2 \\
    &+ \| \nabla_\theta \mathcal{L}^{out} (\theta, U^{(r)} (\theta, \mathcal{T}), \mathcal{T}) \|_2^2 \biggr] \\
    &\leq (\frac{1}{q} - 1) \mathbb{D}^2 + \mathbb{V}^2 .
\end{align*}
\end{proof}

\begin{lemma}
\label{lemma:mlipsch}
Let $p, r, s \in \mathbb{N}$, $\{ \alpha_j > 0 \}_{j = 1}^\infty$ be any sequence, $p(\mathcal{T})$ be a distribution on a nonempty set $\Omega_\mathcal{T}$, $V: \mathbb{R}^s \times \Omega_\mathcal{T} \to \mathbb{R}^p, \mathcal{L}^{in}, \mathcal{L}^{out} : \mathbb{R}^s \times \mathbb{R}^p \times \Omega_\mathcal{T} \to \mathbb{R}$ be functions satisfying Assumption \ref{as:liphes}, and let $U^{(r)}: \mathbb{R}^s \times \Omega_\mathcal{T} \to \mathbb{R}^p$ be defined according to (\ref{eq:maml1}-\ref{eq:maml}), $\mathcal{M}^{(r)} : \mathbb{R}^s \to \mathbb{R}$ be defined according to (\ref{eq:opt}) and satisfy Assumption \ref{as:mreg}. Then for all $\theta', \theta'' \in \mathbb{R}^s$ it holds that
\begin{equation*}
    \| \frac{\partial}{\partial \theta} \mathcal{M} (\theta') - \frac{\partial}{\partial \theta} \mathcal{M} (\theta'') \|_2 \leq \mathcal{C} \| \theta' - \theta'' \|_2 ,
\end{equation*}
where
\begin{gather}
    \mathcal{C} = L_2 + L_2 \mathcal{A}_r + \sum_{j = 1}^r \alpha_j \biggl( L_2 \mathcal{B}_j + L_3 ( 1 + \mathcal{A}_{j - 1} ) L_1 \prod_{j' = j + 1}^r (1 + \alpha_{j'} L_2) \biggr) + M_1 \mathcal{B}_0 + M_2 L_1 \prod_{j = 1}^r (1 + \alpha_j L_2), \label{eq:cdef1} \\
    \mathcal{B}_j = \biggl( L_2 (1 + \mathcal{A}_r) (1 + \alpha_j L_2) + L_1 L_3 \sum_{j' = j + 1}^r \alpha_{j'} (1 + \mathcal{A}_{j' - 1}) \biggr) \prod_{j' = j + 1}^r (1 + \alpha_{j'} L_2), \label{eq:cdef2} \\
    \mathcal{A}_j = \biggl( M_1 \prod_{j' = 1}^j (1 + \alpha_{j'} L_2) + L_2 \sum_{j' = 1}^j \alpha_{j'} \prod_{j'' = j' + 1}^j (1 + \alpha_{j''} L_2) \biggr) \label{eq:cdef3} .
\end{gather}
\end{lemma}
\begin{proof}
Fix $\mathcal{T} \in \Omega_\mathcal{T}$. Let $\phi_0' = V(\theta', \mathcal{T}), \dots, \phi_r'$ and $\phi_0'' = V(\theta'', \mathcal{T}), \dots, \phi_r''$ be inner-GD rollouts (\ref{eq:maml1}-\ref{eq:maml}) for $\theta'$ and $\theta''$ respectively. For each $1 \leq j \leq r$  inequalities applies:
\begin{align*}
    \| \phi_j' - \phi_j'' \|_2 &= \| \phi_{j - 1}' - \phi_{j - 1}'' - \alpha_j (\frac{\partial}{\partial \phi} \mathcal{L}^{in} (\theta', \phi_{j - 1}', \mathcal{T}) - \frac{\partial}{\partial \phi} \mathcal{L}^{in} (\theta'', \phi_{j - 1}'', \mathcal{T})) \|_2 \\
    &\leq \| \phi_{j - 1}' - \phi_{j - 1}'' \|_2 + \alpha_j \| \frac{\partial}{\partial \phi} \mathcal{L}^{in} (\theta', \phi_{j - 1}', \mathcal{T}) - \frac{\partial}{\partial \phi} \mathcal{L}^{in} (\theta'', \phi_{j - 1}'', \mathcal{T}) \|_2 \\
    &\leq \| \phi_{j - 1}' - \phi_{j - 1}'' \|_2 + \alpha_j L_2 \| \phi_{j - 1}' - \phi_{j - 1}'' \|_2 + \alpha_j L_2 \| \theta' - \theta'' \|_2 \\
    &= (1 + \alpha_j L_2) \| \phi_{j - 1}' - \phi_{j - 1}'' \|_2 + \alpha_j L_2 \| \theta' - \theta'' \|_2 .
\end{align*}
Therefore, for each $0 \leq j \leq r$
\begin{align}
    \| \phi_j' - \phi_j'' \|_2 &\leq \| \phi_0' - \phi_0'' \|_2 \prod_{j' = 1}^j (1 + \alpha_{j'} L_2) + L_2 \| \theta' - \theta'' \|_2 \sum_{j' = 1}^j \alpha_{j'} \prod_{j'' = j' + 1}^j (1 + \alpha_{j''} L_2) \nonumber \\
    &= \mathcal{A}_j \cdot \| \theta' - \theta'' \|_2. \label{eq:phibound}
\end{align}
Therefore,
\begin{align*}
    &\| \nabla_{\phi_r'} \mathcal{L}^{out} (\theta', \phi_r', \mathcal{T})) - \nabla_{\phi_r''} \mathcal{L}^{out} (\theta'', \phi_r'', \mathcal{T}) \|_2 = \| \frac{\partial}{\partial \phi} \mathcal{L}^{out} (\theta', \phi_r', \mathcal{T})) - \frac{\partial}{\partial \phi} \mathcal{L}^{out} (\theta'', \phi_r'', \mathcal{T}) \|_2 \\
    &\leq L_2 \| \phi_r' - \phi_r'' \|_2 + L_2 \| \theta' - \theta'' \|_2 \leq L_2 (1 + \mathcal{A}_r) \| \theta' - \theta'' \|_2.
\end{align*}

For each $1 \leq j \leq r$ the following chain of inequalities applies as a result of (\ref{eq:gradmid}-\ref{eq:gradlast}):
\begin{align*}
    &\| \nabla_{\phi_{j - 1}'} \mathcal{L}^{out} (\theta', \phi_r', \mathcal{T}) - \nabla_{\phi_{j - 1}''} \mathcal{L}^{out} (\theta'', \phi_r'', \mathcal{T}) \|_2 = \| \nabla_{\phi_j'} \mathcal{L}^{out} (\theta', \phi_r', \mathcal{T}) - \nabla_{\phi_j''} \mathcal{L}^{out} (\theta'', \phi_r'', \mathcal{T}) \\
    &- \alpha_j ( \biggl( \frac{\partial^2}{\partial \phi^2} \mathcal{L}^{in} (\theta', \phi_{j - 1}', \mathcal{T}) \biggr)^\top \nabla_{\phi_j'} \mathcal{L}^{out} (\theta', \phi_r', \mathcal{T}) - \biggl( \frac{\partial^2}{\partial \phi^2} \mathcal{L}^{in} (\phi_{j - 1}'', \mathcal{T}) \biggr)^\top \nabla_{\phi_j''} \mathcal{L}^{out} (\phi_r'', \mathcal{T}) ) \|_2 \\
    %%%%
    &= \| \nabla_{\phi_j'} \mathcal{L}^{out} (\theta', \phi_r', \mathcal{T}) - \nabla_{\phi_j''} \mathcal{L}^{out} (\theta'', \phi_r'', \mathcal{T}) - \alpha_j \biggl( \frac{\partial^2}{\partial \phi^2} \mathcal{L}^{in} (\theta', \phi_{j - 1}', \mathcal{T}) \biggr)^\top ( \nabla_{\phi_j'} \mathcal{L}^{out} (\theta', \phi_r', \mathcal{T}) \\
    &- \nabla_{\phi_j''} \mathcal{L}^{out} (\theta'', \phi_r'', \mathcal{T}) ) - \alpha_j \biggl( \frac{\partial^2}{\partial \phi^2} \mathcal{L}^{in} (\theta', \phi_{j - 1}', \mathcal{T}) - \frac{\partial^2}{\partial \phi^2} \mathcal{L}^{in} (\theta'', \phi_{j - 1}'', \mathcal{T}) \biggr)^\top \nabla_{\phi_j''} \mathcal{L}^{out} (\theta'', \phi_r'', \mathcal{T}) \|_2 \\
    %%%
    &\leq \| \nabla_{\phi_j'} \mathcal{L}^{out} (\theta', \phi_r', \mathcal{T}) - \nabla_{\phi_j''} \mathcal{L}^{out} (\theta'', \phi_r'', \mathcal{T}) \|_2 + \alpha_j \| \frac{\partial^2}{\partial \phi^2} \mathcal{L}^{in} (\theta', \phi_{j - 1}', \mathcal{T}) \|_2 \| \nabla_{\phi_j'} \mathcal{L}^{out} (\theta', \phi_r', \mathcal{T}) \\
    &- \nabla_{\phi_j''} \mathcal{L}^{out} (\theta'', \phi_r'', \mathcal{T}) \|_2 + \alpha_j \| \frac{\partial^2}{\partial \phi^2} \mathcal{L}^{in} (\theta', \phi_{j - 1}', \mathcal{T}) - \frac{\partial^2}{\partial \phi^2} \mathcal{L}^{in} (\theta'', \phi_{j - 1}'', \mathcal{T}) \|_2  \cdot \| \nabla_{\phi_j''} \mathcal{L}^{out} (\theta'', \phi_r'', \mathcal{T}) \|_2 \\
    %%%%
    &\leq (1 + \alpha_j L_2) \| \nabla_{\phi_j'} \mathcal{L}^{out} (\theta', \phi_r', \mathcal{T}) - \nabla_{\phi_j''} \mathcal{L}^{out} (\theta'', \phi_r'', \mathcal{T}) \|_2 + \alpha_j L_3 L_1 ( \| \phi_{j - 1}' - \phi_{j - 1}'' \|_2 \\
    &+ \| \theta' - \theta'' \|_2 ) \prod_{j' = j + 1}^r (1 + \alpha_{j'} L_2) \\
    %%%%
    &\leq (1 + \alpha_j L_2) \| \nabla_{\phi_j'} \mathcal{L}^{out} (\theta', \phi_r', \mathcal{T}) - \nabla_{\phi_j''} \mathcal{L}^{out} (\theta'', \phi_r'', \mathcal{T}) \|_2 + \alpha_j L_1 L_3 (1 + \mathcal{A}_{j - 1}) \| \theta' - \theta'' \|_2 \prod_{j' = j + 1}^r (1 + \alpha_{j'} L_2) \\
    %%%%
    &\leq \dots \\
    %%%%
    &\leq \| \nabla_{\phi_r'} \mathcal{L}^{out} (\theta', \phi_r', \mathcal{T}) - \nabla_{\phi_r''} \mathcal{L}^{out} (\theta'', \phi_r'', \mathcal{T}) \|_2 \prod_{j' = j}^r (1 + \alpha_{j'} L_2) \\
    &+ L_1 L_3 \| \theta' - \theta'' \|_2 \sum_{j' = j}^r \alpha_{j'} (1 + \mathcal{A}_{j' - 1}) \prod_{j'' = j'}^r (1 + \alpha_{j''} L_2) \prod_{j'' = j + 1}^{j' - 1} (1 + \alpha_{j''} L_2) \leq \mathcal{B}_{j - 1} \| \theta' - \theta'' \|_2,
\end{align*}
where we use (\ref{eq:egradub}).
Using (\ref{eq:gradfirst}-\ref{eq:gradpremid}), we deduce that
\begin{align*}
    &\| \nabla_{\theta'} \mathcal{L}^{out} (\theta', U^{(r)} (\theta', \mathcal{T}), \mathcal{T}) - \nabla_{\theta''} \mathcal{L}^{out} (\theta'', U^{(r)} (\theta'', \mathcal{T}), \mathcal{T}) \|_2 = \| \frac{\partial}{\partial \theta} \mathcal{L}^{out} (\theta', \phi_r', \mathcal{T}) - \frac{\partial}{\partial \theta} \mathcal{L}^{out} (\theta'', \phi_r'', \mathcal{T}) \\
    &- \sum_{j = 1}^r \alpha_j \biggl( \biggl( \frac{\partial^2}{\partial \theta \partial \phi} \mathcal{L}^{in} (\theta', \phi_{j - 1}', \mathcal{T}) \biggr)^\top \nabla_{\phi_j'} \mathcal{L}^{out} (\theta', \phi_r', \mathcal{T}) - \biggl( \frac{\partial^2}{\partial \theta \partial \phi} \mathcal{L}^{in} (\theta'', \phi_{j - 1}'', \mathcal{T}) \biggr)^\top \nabla_{\phi_j''} \mathcal{L}^{out} (\theta'', \phi_r'', \mathcal{T}) \biggr) \\
    &+ \biggl( \frac{\partial}{\partial \theta} V(\theta', \mathcal{T}) \biggr)^\top \nabla_{\phi_0'} \mathcal{L}^{out} (\theta', \phi_r', \mathcal{T}) - \biggl( \frac{\partial}{\partial \theta} V(\theta'', \mathcal{T}) \biggr)^\top \nabla_{\phi_0''} \mathcal{L}^{out} (\theta'', \phi_r'', \mathcal{T}) \|_2 \\
    %%%%
    &\leq \| \frac{\partial}{\partial \theta} \mathcal{L}^{out} (\theta', \phi_r', \mathcal{T}) - \frac{\partial}{\partial \theta} \mathcal{L}^{out} (\theta'', \phi_r'', \mathcal{T}) \|_2 \\
    &+ \sum_{j = 1}^r \alpha_j \| \biggl( \frac{\partial^2}{\partial \theta \partial \phi} \mathcal{L}^{in} (\theta', \phi_{j - 1}', \mathcal{T}) \biggr)^\top \nabla_{\phi_j'} \mathcal{L}^{out} (\theta', \phi_r', \mathcal{T}) - \biggl( \frac{\partial^2}{\partial \theta \partial \phi} \mathcal{L}^{in} (\theta'', \phi_{j - 1}'', \mathcal{T}) \biggr)^\top \nabla_{\phi_j''} \mathcal{L}^{out} (\theta'', \phi_r'', \mathcal{T}) \|_2 \\
    &+ \| \biggl( \frac{\partial}{\partial \theta} V(\theta', \mathcal{T}) \biggr)^\top \nabla_{\phi_0'} \mathcal{L}^{out} (\theta', \phi_r', \mathcal{T}) - \biggl( \frac{\partial}{\partial \theta} V(\theta'', \mathcal{T}) \biggr)^\top \nabla_{\phi_0''} \mathcal{L}^{out} (\theta'', \phi_r'', \mathcal{T}) \|_2 \\
    %%%%
    &\leq \| \frac{\partial}{\partial \theta} \mathcal{L}^{out} (\theta', \phi_r', \mathcal{T}) - \frac{\partial}{\partial \theta} \mathcal{L}^{out} (\theta'', \phi_r'', \mathcal{T}) \|_2 \\
    &+ \sum_{j = 1}^r \alpha_j \| \biggl( \frac{\partial^2}{\partial \theta \partial \phi} \mathcal{L}^{in} (\theta', \phi_{j - 1}', \mathcal{T}) \biggr)^\top \nabla_{\phi_j'} \mathcal{L}^{out} (\theta', \phi_r', \mathcal{T}) - \biggl( \frac{\partial^2}{\partial \theta \partial \phi} \mathcal{L}^{in} (\theta', \phi_{j - 1}', \mathcal{T}) \biggr)^\top \nabla_{\phi_j''} \mathcal{L}^{out} (\theta'', \phi_r'', \mathcal{T}) \\
    &+ \biggl( \frac{\partial^2}{\partial \theta \partial \phi} \mathcal{L}^{in} (\theta', \phi_{j - 1}', \mathcal{T}) \biggr)^\top \nabla_{\phi_j''} \mathcal{L}^{out} (\theta'', \phi_r'', \mathcal{T}) - \biggl( \frac{\partial^2}{\partial \theta \partial \phi} \mathcal{L}^{in} (\theta'', \phi_{j - 1}'', \mathcal{T}) \biggr)^\top \nabla_{\phi_j''} \mathcal{L}^{out} (\theta'', \phi_r'', \mathcal{T}) \|_2 \\
    &+ \| \biggl( \frac{\partial}{\partial \theta} V(\theta', \mathcal{T}) \biggr)^\top \nabla_{\phi_0'} \mathcal{L}^{out} (\theta', \phi_r', \mathcal{T}) - \biggl( \frac{\partial}{\partial \theta} V(\theta', \mathcal{T}) \biggr)^\top \nabla_{\phi_0''} \mathcal{L}^{out} (\theta'', \phi_r'', \mathcal{T}) \\
    &+ \biggl( \frac{\partial}{\partial \theta} V(\theta', \mathcal{T}) \biggr)^\top \nabla_{\phi_0''} \mathcal{L}^{out} (\theta'', \phi_r'', \mathcal{T}) - \biggl( \frac{\partial}{\partial \theta} V(\theta'', \mathcal{T}) \biggr)^\top \nabla_{\phi_0''} \mathcal{L}^{out} (\theta'', \phi_r'', \mathcal{T}) \|_2 \\
    %%%%
    &\leq \| \frac{\partial}{\partial \theta} \mathcal{L}^{out} (\theta', \phi_r', \mathcal{T}) - \frac{\partial}{\partial \theta} \mathcal{L}^{out} (\theta'', \phi_r'', \mathcal{T}) \|_2 \\
    &+ \sum_{j = 1}^r \alpha_j \| \frac{\partial^2}{\partial \theta \partial \phi} \mathcal{L}^{in} (\theta', \phi_{j - 1}', \mathcal{T}) \|_2 \| \nabla_{\phi_j'} \mathcal{L}^{out} (\theta', \phi_r', \mathcal{T}) - \nabla_{\phi_j''} \mathcal{L}^{out} (\theta'', \phi_r'', \mathcal{T}) \|_2 \\
    &+ \sum_{j = 1}^r \alpha_j \| \frac{\partial^2}{\partial \theta \partial \phi} \mathcal{L}^{in} (\theta', \phi_{j - 1}', \mathcal{T})  - \frac{\partial^2}{\partial \theta \partial \phi} \mathcal{L}^{in} (\theta'', \phi_{j - 1}'', \mathcal{T}) \|_2 \| \nabla_{\phi_j'} \mathcal{L}^{out} (\theta'', \phi_r'', \mathcal{T}) \|_2 \\
    &+ \| \frac{\partial}{\partial \theta} V(\theta', \mathcal{T}) \|_2 \| \nabla_{\phi_0'} \mathcal{L}^{out} (\theta', \phi_r', \mathcal{T}) - \nabla_{\phi_0''} \mathcal{L}^{out} (\theta'', \phi_r'', \mathcal{T}) \|_2 \\
    &+ \| \frac{\partial}{\partial \theta} V(\theta', \mathcal{T}) - \frac{\partial}{\partial \theta} V(\theta'', \mathcal{T}) \|_2 \| \nabla_{\phi_0'} \mathcal{L}^{out} (\theta'', \phi_r'', \mathcal{T}) \|_2 \\
    %%%%
    &\leq L_2 \| \theta' - \theta'' \|_2 + L_2 \| \phi_r' - \phi_r'' \|_2 + \sum_{j = 1}^r \alpha_j \biggl( L_2 \mathcal{B}_j \| \theta' - \theta'' \|_2 + L_3 ( \| \theta' - \theta'' \|_2 \\
    &+ \| \phi_{j - 1}' - \phi_{j - 1}'' \|_2 ) L_1 \prod_{j' = j + 1}^r (1 + \alpha_{j'} L_2) \|_2 \biggr) + M_1 \mathcal{B}_0 \| \theta' - \theta'' \|_2 + M_2 \| \theta' - \theta'' \|_2 L_1 \prod_{j = 1}^r (1 + \alpha_j L_2) \\
    %%%%
    &\leq L_2 \| \theta' - \theta'' \|_2 + L_2 \mathcal{A}_r \| \theta' - \theta'' \|_2 + \sum_{j = 1}^r \alpha_j \biggl( L_2 \mathcal{B}_j \| \theta' - \theta'' \|_2 + L_3 ( \| \theta' - \theta'' \|_2 \\
    &+ \mathcal{A}_{j - 1} \| \theta' - \theta'' \|_2 ) L_1 \prod_{j' = j + 1}^r (1 + \alpha_{j'} L_2) \|_2 \biggr) + M_1 \mathcal{B}_0 \| \theta' - \theta'' \|_2 + M_2 \| \theta' - \theta'' \|_2 L_1 \prod_{j = 1}^r (1 + \alpha_j L_2) \leq \mathcal{C} \| \theta' - \theta'' \|_2,
\end{align*}
where we use (\ref{eq:egradub}) and (\ref{eq:phibound}).
Finally, by taking expectation with respect to $\mathcal{T} \sim p(\mathcal{T})$ and applying Jensen inequality we get
\begin{align*}
    \| \frac{\partial}{\partial \theta} \mathcal{M}^{(r)} (\theta') - \frac{\partial}{\partial \theta} \mathcal{M}^{(r)} (\theta'') \|_2^2 &\leq \mathbb{E}_{p(\mathcal{T})} \left[\| \nabla_{\theta'} \mathcal{L}^{out} (\theta', U^{(r)}(\theta', \mathcal{T}), \mathcal{T}) - \nabla_{\theta''} \mathcal{L}^{out} (\theta'', U^{(r)}(\theta'', \mathcal{T}), \mathcal{T}) \|_2^2  \right] \\
    &\leq \mathcal{C}^2 \| \theta' - \theta'' \|_2^2,
\end{align*}
which is equivalent to the statement of Lemma.
\end{proof}

\begin{lemma} \label{th:convexp}
Let $p, r, s \in \mathbb{N}$, $\{ \alpha_j > 0 \}_{j = 1}^\infty$ and $\{ \gamma_k > 0 \}_{k = 1}^\infty$ be any sequences, $q \in (0, 1], \theta_0 \in \mathbb{R}^p$, $p(\mathcal{T})$ be a distribution on a nonempty set $\Omega_\mathcal{T}$, $V: \mathbb{R}^s \times \Omega_\mathcal{T} \to \mathbb{R}^p, \mathcal{L}^{in}, \mathcal{L}^{out} : \mathbb{R}^s \times \mathbb{R}^p \times \Omega_\mathcal{T} \to \mathbb{R}$ be functions satisfying Assumption \ref{as:liphes}, and let $U^{(r)}: \mathbb{R}^s \times \Omega_\mathcal{T} \to \mathbb{R}^p$ be defined according to (\ref{eq:maml1}-\ref{eq:maml}), $\mathcal{M}^{(r)} : \mathbb{R}^p \to \mathbb{R}$ be defined according to (\ref{eq:opt}) and satisfy Assumption \ref{as:mreg}. Define $\mathcal{G}_{FO}: \mathbb{R}^s \times \Omega_\mathcal{T} \to \mathbb{R}^s$ and $\mathcal{G}: \mathbb{R}^s \times \Omega_\mathcal{T} \times \{ 0, 1 \} \to \mathbb{R}^s$ as
\begin{gather*}
    \mathcal{G}_{FO} (\theta, \mathcal{T}) = \frac{\partial}{\partial \theta} \mathcal{L}^{out} (\theta, \phi_r, \mathcal{T}) + (\frac{\partial}{\partial \theta} V (\theta, \mathcal{T}))^\top \frac{\partial}{\partial \phi} \mathcal{L}^{out} (\theta, \phi_r, \mathcal{T}), \quad \phi_r = U^{(r)} (\theta, \mathcal{T}), \\
    \mathcal{G} (\theta, \mathcal{T}, x) = \mathcal{G}_{FO} (\theta, \mathcal{T}) + \frac{x}{q} (\nabla_\theta \mathcal{L}^{out} (\theta, \phi_r, \mathcal{T}) - \mathcal{G}_{FO} (\theta, \mathcal{T})) .
\end{gather*}
Let $\{ \mathcal{T}_k \}_{k = 1}^\infty, \{ \xi_k \}_{k = 1}^\infty$ be sequences of i.i.d. samples from $p(\mathcal{T})$ and $\mathrm{Bernoulli} (q)$ respectively, such that $\sigma$-algebras populated by both sequences are independent.
Let $\{ \theta_k \in \mathbb{R}^s \}_{k = 0}^\infty$ be a sequence where for all $k \in \mathbb{N}$ $\theta_k = \theta_{k - 1} - \gamma_k \mathcal{G} (\theta_{k - 1}, \mathcal{T}_k, \xi_k)$.
Then for each $k \in \mathbb{N}$
\begin{align}
    \sum_{u = 1}^k \gamma_u \mathbb{E} \left[ \| \frac{\partial}{\partial \theta} \mathcal{M}^{(r)} (\theta_{u - 1}) \|_2^2 \right] &\leq \mathcal{M}^{(r)} (\theta_0) - \mathcal{M}^{(r)}_* + \mathcal{C} \biggl( (\frac{1}{q} - 1) \mathbb{D}^2 + \mathbb{V}^2 \biggr) \sum_{u = 1}^k\gamma_u^2 \label{eq:convres}
\end{align}
where $\mathcal{C}$ is defined in (\ref{eq:cdef1}-\ref{eq:cdef3}), $\mathbb{D}, \mathbb{V}$ are defined in (\ref{eq:dtruedef}), (\ref{eq:vtruedef}) respectively.
\end{lemma}

\begin{proof}
Let $\mathcal{F}_u$ denote a $\sigma$-algebra populated by $\{ \mathcal{T}_\kappa,  \xi_\kappa \}_{\kappa < u}$. Using Lemma \ref{lemma:mlipsch}, we apply Inequality 4.3 from \citep{bottou} to obtain that for all $\theta', \theta'' \in \mathbb{R}^s$
\begin{equation*}
    \mathcal{M}^{(r)} (\theta') \leq \mathcal{M}^{(r)} (\theta'') + \biggl( \frac{\partial}{\partial \theta} \mathcal{M}^{(r)} (\theta'') \biggr)^\top (\theta' - \theta'') + \frac{1}{2} \mathcal{C} \| \theta' - \theta'' \|_2^2.
\end{equation*}
For any $u \in \mathbb{N}$, by setting $\theta' = \theta_u$, $\theta'' = \theta_{u - 1}$ we deduce that
\begin{equation*}
    \mathcal{M}^{(r)} (\theta_u) \leq \mathcal{M}^{(r)} (\theta_{u - 1}) - \gamma_u \biggl( \frac{\partial}{\partial \theta} \mathcal{M}^{(r)} (\theta_{u - 1}) \biggr)^\top \mathcal{G} (\theta_{u - 1}, \mathcal{T}_u, \xi_u) + \frac{1}{2} \gamma_u^2 \mathcal{C} \| \mathcal{G} (\theta_{u - 1}, \mathcal{T}_u, \xi_u) \|_2^2.
\end{equation*}
Take expectation with respect to $\mathcal{F}_u$:
\begin{align}
    \E { \mathcal{M}^{(r)} (\theta_u) | \mathcal{F}_u } &\leq \mathcal{M}^{(r)} (\theta_{u - 1}) - \gamma_u \biggl( \frac{\partial}{\partial \theta} \mathcal{M}^{(r)} (\theta_{u - 1}) \biggr)^\top \E { \mathcal{G} (\theta_{u - 1}, \mathcal{T}_u, \xi_u) | \mathcal{F}_u } \nonumber \\
    &+ \frac{1}{2} \gamma_u^2 \mathcal{C} \E { \| \mathcal{G} (\theta_{u - 1}, \mathcal{T}_u, \xi_u) \|_2^2 | \mathcal{F}_u } \nonumber \\
    &\leq \mathcal{M}^{(r)} (\theta_{u - 1}) - \gamma_u \| \frac{\partial}{\partial \theta} \mathcal{M}^{(r)} (\theta_{u - 1}) \|_2^2 + \frac{1}{2} \gamma_u^2 \mathcal{C} \biggl( (\frac{1}{q} - 1) \mathbb{D}^2 + \mathbb{V}^2 \biggr), \nonumber
\end{align}
where we use Lemma \ref{lemma:unbbnd}'s result.
Take the full expectation:
\begin{equation*}
    \E{\mathcal{M}^{(r)} (\theta_u)} \leq \E{\mathcal{M}^{(r)} (\theta_{u - 1})} - \gamma_u \E{\| \frac{\partial}{\partial \theta} \mathcal{M}^{(r)} (\theta_{u - 1}) \|_2^2} + \frac{1}{2} \gamma_u^2 \mathcal{C} \biggl( (\frac{1}{q} - 1) \mathbb{D}^2 + \mathbb{V}^2 \biggr)
\end{equation*}
which is equivalent to
\begin{equation} \label{eq:sgdnonlinexp}
   \gamma_u \E{\| \frac{\partial}{\partial \theta} \mathcal{M}^{(r)} (\theta_{u - 1}) \|_2^2} \leq \E{\mathcal{M}^{(r)} (\theta_{u - 1})} - \E{\mathcal{M}^{(r)} (\theta_u)} + \frac{1}{2} \gamma_u^2 \mathcal{C} \biggl( (\frac{1}{q} - 1) \mathbb{D}^2 + \mathbb{V}^2 \biggr).
\end{equation}
Sum inequalities (\ref{eq:sgdnonlinexp}) for all $1 \leq u \leq k$:
\begin{align*}
    \sum_{u = 1}^k \gamma_u \E{\| \frac{\partial}{\partial \theta} \mathcal{M}^{(r)} (\theta_{u - 1}) \|_2^2} &\leq \E{\mathcal{M}^{(r)} (\theta_0)} - \E{\mathcal{M}^{(r)} (\theta_k)} + \mathcal{C} \biggl( (\frac{1}{q} - 1) \mathbb{D}^2 + \mathbb{V}^2 \biggr) \sum_{u = 1}^k\gamma_u^2 \\
    &\leq \mathcal{M}^{(r)} (\theta_0) - \mathcal{M}^{(r)}_* + \mathcal{C} \biggl( (\frac{1}{q} - 1) \mathbb{D}^2 + \mathbb{V}^2 \biggr) \sum_{u = 1}^k\gamma_u^2 .
\end{align*}
\end{proof}

\begin{proof}[Theorem \ref{th:conv} proof]

Under conditions of the theorem results of Lemma \ref{th:convexp} are true.

First, we prove 1. If $\sum_{k = 1}^\infty \gamma_k^2 < \infty$, then the right-hand side of (\ref{eq:convres}) converges to a finite value when $k \to \infty$. Therefore, the left-hand side also converges to a finite value. Suppose the statement of 1 is false. Then there exists $k_0 \in \mathbb{N}, A > 0$ such that $\forall u \geq k_0 : \E{ \| \frac{\partial}{\partial \theta} \mathcal{M}^{(r)} (\theta_{u - 1}) \|_2^2 } > A$. But then for all $k \geq k_0$
\begin{equation*}
    \sum_{u = 1}^k \gamma_u \E{ \| \frac{\partial}{\partial \theta} \mathcal{M}^{(r)} (\theta_{u - 1}) \|_2^2 } \geq A \sum_{u = k_0}^k \gamma_u \to \infty
\end{equation*}
when $k \to \infty$, which is a contradiction. Therefore, 1 is true.

Next, we prove 2. Observe that
\begin{align*}
    \min_{0 \leq u < k} \E{\| \frac{\partial}{\partial \theta} \mathcal{M}^{(r)} (\theta_u) \|_2^2} \sum_{u = 1}^k \gamma_u &\leq \sum_{u = 1}^k \gamma_u \E{\| \frac{\partial}{\partial \theta} \mathcal{M}^{(r)} (\theta_{u - 1}) \|_2^2} \\
    &\leq \mathcal{M}^{(r)} (\theta_0) - \mathcal{M}^{(r)}_* + \mathcal{C} \biggl( (\frac{1}{q} - 1) \mathbb{D}^2 + \mathbb{V}^2 \biggr) \sum_{u = 1}^k \gamma_u^2.
\end{align*}
Divide by $\sum_{u = 1}^k \gamma_u$:
\begin{equation*}
    \min_{0 \leq u < k} \E{ \| \frac{\partial}{\partial \theta} \mathcal{M}^{(r)} (\theta_u) \|_2^2} \leq \frac{1}{\sum_{u = 1}^k \gamma_u} (\mathcal{M}^{(r)} (\theta_0) - \mathcal{M}^{(r)}_*) + \mathcal{C} \biggl( (\frac{1}{q} - 1) \mathbb{D}^2 + \mathbb{V}^2 \biggr) \cdot \frac{1}{\sum_{u = 1}^k \gamma_u} \cdot \sum_{u = 1}^k \gamma_u^2 .
\end{equation*}
2 is satisfied by observing that
\begin{equation*}
    \sum_{u = 1}^k \gamma_u = \sum_{u = 1}^k u^{-0.5} = \Omega (k^{0.5}), \quad \sum_{u = 1}^k \gamma_u^2 = \sum_{u = 1}^k u^{-1} = O (\log k) = O (k^\epsilon)
\end{equation*}
for any $\epsilon > 0$.
\end{proof}

\subsection{Theorem \ref{th:counter}}

\begin{proof}
Consider a set $\Omega_\mathcal{T}$ consisting of two elements: $\Omega_\mathcal{T} = \{ \mathcal{T}^{(1)}, \mathcal{T}^{(2)} \}$. Define $p(\mathcal{T})$ so that
\begin{equation*}
    \mathbb{P}_{p(\mathcal{T})}(\mathcal{T} = \mathcal{T}^{(1)}) = \mathbb{P}_{p(\mathcal{T})}(\mathcal{T} = \mathcal{T}^{(2)}) = \frac{1}{2}.
\end{equation*}
Choose arbitrary numbers $0 < a_1, a_2 < \frac{1}{\alpha}$, $a_1 \neq a_2$ and set $b_1 = 0$. Since $a_1 \neq a_2$, $(1 - \alpha a_1)/(1 - \alpha a_2) \neq 1$ and, consequently,
\begin{equation*}
    \biggl( \frac{1 - \alpha a_1}{1 - \alpha a_2} \biggr)^r \neq \biggl( \frac{1 - \alpha a_1}{1 - \alpha a_2} \biggr)^{2 r}.
\end{equation*}
Multiply by $\frac{a_1}{a_2} \neq 0$:
\begin{equation} \label{eq:ineq1}
    \frac{a_1}{a_2} \cdot \biggl( \frac{1 - \alpha a_1}{1 - \alpha a_2} \biggr)^r \neq \frac{a_1}{a_2} \cdot \biggl( \frac{1 - \alpha a_1}{1 - \alpha a_2} \biggr)^{2 r}.
\end{equation}
From (\ref{eq:ineq1}) and since $\frac{a_1}{a_2} (\frac{1 - \alpha a_1}{1 - \alpha a_2})^r, \frac{a_1}{a_2} (\frac{1 - \alpha a_1}{1 - \alpha a_2})^{2 r} > 0$ it follows that
\begin{equation*}
    \frac{\frac{a_1}{a_2} (\frac{1 - \alpha a_1}{1 - \alpha a_2})^{2 r} + 1}{\frac{a_1}{a_2} (\frac{1 - \alpha a_1}{1 - \alpha a_2})^r + 1} - 1 \neq 0.
\end{equation*}
Multiply inequality by $(1 - \alpha a_2)^{2 r} \neq 0$ and numerator/denominator by $a_2 (1 - \alpha a_2)^r \neq 0$:
\begin{equation*}
    \frac{a_1 (1 - \alpha a_1)^{2 r} + a_2 (1 - \alpha a_2)^{2 r}}{a_1 (1 - \alpha a_1)^r + a_2 (1 - \alpha a_2)^r} (1 - \alpha a_2)^r - (1 - \alpha a_2)^{2 r} \neq 0.
\end{equation*}
Because of the inequality above, we can define a number $b_2$ as
\begin{equation} \label{eq:b2def}
    b_2 = 2 \sqrt{2 D} \biggl| \frac{a_1 (1 - \alpha a_1)^{2 r} + a_2 (1 - \alpha a_2)^{2 r}}{a_1 (1 - \alpha a_1)^r + a_2 (1 - \alpha a_2)^r} (1 - \alpha a_1)^r - (1 - \alpha a_2)^{2 r} \biggr|^{-1} > 0
\end{equation}
and select arbitrary number $A$ so that
\begin{equation} \label{eq:adef}
    A > \left| \frac{b_1}{a_1} - \frac{b_2}{a_2} \right| .
\end{equation}
Consider two functions $f_i (x)$, $f_i: \mathbb{R} \to \mathbb{R}$, $i \in \{ 1, 2 \}$ defined as follows (denote $z_i = z_i (x) = | x - \frac{b_i}{a_i} |$)
\begin{equation} \label{eq:fdef}
    f_i (x) = 
    \begin{cases}
        \frac{1}{2} a_i z_i^2 & \text{if } z_i \leq A \\
        - \frac{1}{6} a_i (z_i - A)^3 + \frac{1}{2} a_i (z_i - A)^2 + a_i A z_i - \frac{1}{2} a_i A^2 & \text{if } A < z_i \leq A + 1 \\
        (\frac{1}{2} a_i + a_i A) z_i - \frac{1}{6} a_i - \frac{1}{2} a_i A^2 - \frac{1}{2} a_i A & \text{if } A + 1 < z_i
    \end{cases}.
\end{equation}
It is easy to check that for $i \in \{ 1, 2 \}$ $f_i (x)$ is twice differentiable with a global minimum at $\frac{b_i}{a_i}$. The following expressions apply for the first and second derivative:
\begin{align}
    f_i' (x) &= 
    \begin{cases}
        a_i x - b_i & \text{if } z_i \leq A \\
        \biggl( - \frac{1}{2} a_i (z_i - A)^2 + a_i z_i \biggr) \textrm{sign} (x - \frac{b_i}{a_i}) & \text{if } A < z_i \leq A + 1 \\
        (\frac{1}{2} a_i + a_i A) \textrm{sign} (x - \frac{b_i}{a_i}) & \text{if } A + 1 < z_i
    \end{cases}, \label{eq:fdif1} \\
    f_i'' (x) &= 
    \begin{cases}
        a_i & \text{if } z_i \leq A \\
        - a_i z_i + a_i + a_i A & \text{if } A < z_i \leq A + 1 \\
        0 & \text{if } A + 1 < z_i
    \end{cases}. \label{eq:fdif2}
\end{align}
From (\ref{eq:fdif1}-\ref{eq:fdif2}) it follows that each $f_i$ has bounded, Lipschitz-continuous gradients and Hessians. Define $V (\theta, \mathcal{T}) \equiv \theta$,  $\mathcal{L}^{in} (\theta, \phi, \mathcal{T}_i) \equiv \mathcal{L}^{out} (\theta, \phi, \mathcal{T}_i) \equiv f_i (\phi^{(1)})$ for $i \in \{ 1, 2 \}$, where $\phi^{(1)}$ denotes a first element of $\phi$, then Assumption \ref{as:liphes} is satisfied. Since $\Omega_\mathcal{T}$ is finite, Assumption \ref{as:mreg} is also satisfied.

Let $I = [\frac{b_2}{a_2} - A, \frac{b_1}{a_1} + A]$. Observe that from (\ref{eq:adef}) it follows that $\frac{b_1}{a_1}, \frac{b_2}{a_2} \in I$ and $I \subseteq [\frac{b_i}{a_i} - A, \frac{b_i}{a_i} + A]$ for $i \in \{ 1, 2 \}$, i.e. $I$ corresponds to a quadratic part of both $f_1 (x)$ and $f_2 (x)$. If $x \in I$, then for $i \in \{ 1, 2 \}$
\begin{align}
    x - \alpha f'_i (x) &= x - \alpha (a_i x - b_i) = (1 - \alpha a_i) x + \alpha b_i \nonumber \\
    &= (1 - \alpha a_i) \cdot x + \alpha a_i \cdot \frac{b_i}{a_i} \in [\min(x, \frac{b_i}{a_i}), \max(x, \frac{b_i}{a_i})] \subseteq I \label{eq:iprop}
\end{align}
since $x - \alpha f'_i (x)$ is a convex combination of $x$ and $\frac{b_i}{a_i}$ ($0 < \alpha a_i, 1 - \alpha a_i < 1$). From (\ref{eq:iprop}) and the definition of $\mathcal{L}^{in} (\theta, \phi, \mathcal{T}), \mathcal{L}^{out} (\theta, \phi, \mathcal{T})$ it follows that if $\phi_0, \dots, \phi_r$ is a rollout of inner GD (\ref{eq:maml}) for task $\mathcal{T}^{(i)}$ and $\theta^{(1)} = \phi_0^{(1)} \in I$, then $\phi_1^{(1)}, \dots, \phi_r^{(1)} \in I$ and, hence,
\begin{gather}
    \nabla_{\phi_r} \mathcal{L}^{out} (\theta, \phi_r, \mathcal{T}^{(i)})^{(1)} = f_i' (\phi_r^{(1)}) = a_i \phi_r^{(1)} - b_i, \nonumber \\
    \forall j \in \{ 1, \dots, r \}: \phi_j^{(1)} = (1 - \alpha a_i) \phi_{j - 1}^{(1)} + \alpha b_i. \label{eq:mamlcnt}
\end{gather}
From (\ref{eq:mamlcnt}) we derive that
\begin{gather}
    \phi_j^{(1)} - \frac{b_i}{a_i} = (1 - \alpha a_i) (\phi_{j - 1}^{(1)} - \frac{b_i}{a_i}), \quad \phi_r^{(1)} - \frac{b_i}{a_i} = (1 - \alpha a_i)^r (\phi_0^{(1)} - \frac{b_i}{a_i}), \nonumber \\
    \phi_r^{(1)} = (1 - \alpha a_i)^r (\phi_0^{(1)} - \frac{b_i}{a_i}) + \frac{b_i}{a_i}, \nonumber \\
    \nabla_{\phi_r} \mathcal{L}^{out} (\theta, \phi_r, \mathcal{T}^{(i)})^{(1)} = a_i \biggl( (1 - \alpha a_i)^r (\phi_0^{(1)} - \frac{b_i}{a_i}) + \frac{b_i}{a_i} \biggr) - b_i = a_i (1 - \alpha a_i)^r ( \phi_0^{(1)} - \frac{b_i}{a_i} ) . \label{eq:fomamlcnt}
\end{gather}

From (\ref{eq:step}) it follows that there exists a deterministic number $k_0 \in \mathbb{N}$ such that for all $k \geq k_0$
\begin{equation} \label{eq:k0def1}
    \gamma_k < \frac{1}{2} \min_{i \in \{ 1, 2 \} } \frac{1}{a_i (1 + \alpha a_i)^r} .
\end{equation}
If (\ref{eq:k0def1}) holds, then it also holds that
\begin{equation} \label{eq:k0def2}
    \gamma_k < \min_{i \in \{ 1, 2 \} } \frac{1}{a_i (1 + \alpha a_i)^r}, \quad \gamma_k < \min_{i \in \{ 1, 2 \} } \frac{1}{a_i (1 - \alpha a_i)^r}.
\end{equation}
For any $k \geq k_0$ the following cases are possible:
\begin{enumerate}
    \item Case 1: $\theta_{k - 1}^{(1)} \in I$. An identity (\ref{eq:fomamlcnt}) allows to write that for $i \in \{ 1, 2 \}$
\begin{equation} \label{eq:fomamloutcnt}
    \mathcal{G}_{FO} (\theta_{k - 1}, \mathcal{T}^{(i)})^{(1)} = a_i (1 - \alpha a_i)^r ( \theta_{k - 1}^{(1)} - \frac{b_i}{a_i} ).
\end{equation}
For $i \in \{ 1, 2 \}$ let random number $v_i \in {0, 1}$ denote an indicator that $\mathcal{T}_k = \mathcal{T}^{(i)}$ ($v_1 + v_2 = 1$). Then from (\ref{eq:fomamloutcnt}) we deduce that
\begin{align*}
    \theta_k^{(1)} &= \theta_{k - 1}^{(1)} - \gamma_k \sum_{i = 1}^2 v_i a_i (1 - \alpha a_i)^r ( \theta_{k - 1}^{(1)} - \frac{b_i}{a_i} ) \\
    &= (1 - \gamma_k \sum_{i = 1}^2 v_i a_i (1 - \alpha a_i)^r ) \cdot \theta_{k - 1}^{(1)} + \gamma_k v_1 a_1 (1 - \alpha a_1)^r \cdot \frac{b_1}{a_1} \\
    &+ \gamma_k v_2 a_2 (1 - \alpha a_2)^r \cdot \frac{b_2}{a_2} \\
    &\in [\min(\theta_{k - 1}^{(1)}, \frac{b_1}{a_1}, \frac{b_2}{a_2}), \max(\theta_{k - 1}^{(1)}, \frac{b_1}{a_1}, \frac{b_2}{a_2})] \subseteq I
\end{align*}
since $\theta_k^{(1)}$ is a convex combination of $\theta_{k - 1}^{(1)}, \frac{b_1}{a_1}, \frac{b_2}{a_2}$. Indeed, due to (\ref{eq:k0def2})
\begin{equation*}
    0 \leq (1 - \gamma_k \sum_{i = 1}^2 v_i a_i (1 - \alpha a_i)^r ), \gamma_k v_1 a_1 (1 - \alpha a_1)^r, \gamma_k v_2 a_2 (1 - \alpha a_2)^r \leq 1
\end{equation*}
and
\begin{equation*}
    (1 - \gamma_k \sum_{i = 1}^2 v_i a_i (1 - \alpha a_i)^r ) + \gamma_k v_1 a_1 (1 - \alpha a_1)^r + \gamma_k v_2 a_2 (1 - \alpha a_2)^r = 1.
\end{equation*}
As a result of this Case we conclude that if $k \geq k_0$ and $\theta_{k - 1}^{(1)} \in I$, then for all $k' \geq k$ it also holds that $\theta_{k'}^{(1)} \in I$.
\item Case 2: $\theta_{k - 1}^{(1)} > \frac{b_1}{a_1} + A$. From (\ref{eq:fdif2}) observe that for $i \in \{ 1, 2 \}$ and any $x \in \mathbb{R}$ $f''_i (x) \leq a_i$. Hence, $f'_i$'s Lipschitz constant is $a_i$. Let $\phi_0, \dots, \phi_r$ and $\overline{\phi}_0, \dots, \overline{\phi}_r$ be two inner-GD (\ref{eq:maml}) rollouts for task $\mathcal{T}^{(i)}$ and $\phi_0^{(1)} > \overline{\phi}_0^{(1)}$. For $j \in \{ 1, \dots, r \}$ suppose that $\phi_{j - 1}^{(1)} > \overline{\phi}_{j - 1}^{(1)}$. Then
\begin{align*}
    \phi_j^{(1)} - \overline{\phi}_j^{(1)} &= \phi_{j - 1}^{(1)} - \overline{\phi}_{j - 1}^{(1)} - \alpha (f_i' (\phi_{j - 1}^{(1)}) - f_i'(\overline{\phi}_{j - 1}^{(1)})) \\
    &\geq \phi_{j - 1}^{(1)} - \overline{\phi}_{j - 1}^{(1)} - \alpha | f_i' (\phi_{j - 1}^{(1)}) - f_i'(\overline{\phi}_{j - 1}^{(1)}) | \\
    &\geq \phi_{j - 1}^{(1)} - \overline{\phi}_{j - 1}^{(1)} - \alpha a_i | \phi_{j - 1}^{(1)} - \overline{\phi}_{j - 1}^{(1)} | \\
    &> \phi_{j - 1}^{(1)} - \overline{\phi}_{j - 1}^{(1)} - | \phi_{j - 1}^{(1)} - \overline{\phi}_{j - 1}^{(1)} | \\
    &= 0
\end{align*}
or $\phi_j^{(1)} > \overline{\phi}_j^{(1)}$ where we use Lipschitz continuity of $f'_i$ and that $\alpha a_i < 1$ by the choice of $a_1, a_2$. Therefore, since $\phi_0^{(1)} > \overline{\phi}_0^{(1)}$, $\phi_1^{(1)} > \overline{\phi}_1^{(1)}$ and so on, eventually $\phi_r^{(1)} > \overline{\phi}_r^{(1)}$. Observe that $f_i' (x)$ is a strictly monotonously increasing function, therefore $f_i' (\phi_r^{(1)}) > f_i' (\overline{\phi}_r^{(1)})$. To sum up:
\begin{equation} \label{eq:mon}
    f_i' (\phi_r^{(1)}) > f_i' (\overline{\phi}_r^{(1)}) \quad \text{when } \phi_0^{(1)} > \overline{\phi}_0^{(1)}.
\end{equation}

Set $\overline{\phi}_0^{(1)} = \frac{b_i}{a_i}$, then $f'(\overline{\phi}_{j - 1}^{(1)}) = 0$ and $\overline{\phi}_1^{(1)} = \overline{\phi}_0^{(1)} - \alpha \cdot 0 = \overline{\phi}_0^{(1)}$ and so on, eventually $\overline{\phi}_r^{(1)} = \frac{b_i}{a_i}$ and $f'(\overline{\phi}_r^{(1)}) = 0$. Therefore, if $\phi_0^{(1)} = \frac{b_1}{a_1} + A > \max ( \frac{b_1}{a_1}, \frac{b_2}{a_2} )$ then $f_i' (\phi_r^{(1)}) > f_i' (\overline{\phi}_r^{(1)}) = 0$. For $i \in \{ 1, 2 \}$ denote a deterministic value of $f_i' (\phi_r^{(1)})$ by $B_i  > 0$. By setting $\phi_0^{(1)} = \theta_{k - 1}^{(1)}, \overline{\phi}_0^{(1)} = \frac{b_1}{a_1} + A$ and using (\ref{eq:mon}) we obtain:
\begin{equation} \label{eq:glb}
    \mathcal{G}_{FO} (\theta_{k - 1}, \mathcal{T}_i)^{(1)} = f'_i (\phi_r^{(1)}) > f'_i (\overline{\phi}_r^{(1)}) = B_i \geq B > 0 .
\end{equation}
where we denote $B = \min (B_1, B_2)$. 

In addition, set $\phi_0^{(1)} = \theta_{k - 1}^{(1)}, \overline{\phi}_0^{(1)} = \frac{b_i}{a_i}$. Then
\begin{align*}
    \mathcal{G}_{FO} (\theta_{k - 1}, \mathcal{T}^{(i)})^{(1)} &= | \mathcal{G}_{FO} (\theta_{k - 1}, \mathcal{T}^{(i)})^{(1)} | = | f'_i (\phi_r^{(1)}) - 0 | = | f'_i (\phi_r^{(1)}) - f'_i (\overline{\phi}_r^{(1)}) | \\
    &\leq a_i | \phi_r^{(1)} - \overline{\phi}_r^{(1)} | = a_i | \phi_{r - 1}^{(1)} - \overline{\phi}_{r - 1}^{(1)} - \alpha ( f'_i (\phi_{r - 1}^{(1)}) - f'_i (\overline{\phi}_{r - 1}^{(1)}) ) | \\
    &\leq a_i | \phi_{r - 1}^{(1)} - \overline{\phi}_{r - 1}^{(1)} | + \alpha a_i | f'_i (\phi_{r - 1}^{(1)}) - f'_i (\overline{\phi}_{r - 1}^{(1)}) | \\
    &\leq a_i (1 + \alpha a_i) | \phi_{r - 1}^{(1)} - \overline{\phi}_{r - 1}^{(1)} | \\
    &\dots \\
    &\leq a_i (1 + \alpha a_i)^r | \phi_0^{(1)} - \overline{\phi}_0^{(1)} | \\
    &= a_i (1 + \alpha a_i)^r | \theta_{k - 1}^{(1)} - \frac{b_i}{a_i} |.
\end{align*}
Since $\theta_{k - 1}^{(1)} > \frac{b_1}{a_1} + A > \max (\frac{b_1}{a_1}, \frac{b_2}{a_2})$, we derive that
\begin{align*}
    \mathcal{G}_{FO} (\theta_{k - 1}, \mathcal{T}^{(i)})^{(1)} &\leq a_i (1 + \alpha a_i)^r (\theta_{k - 1}^{(1)} - \frac{b_i}{a_i}) \leq \frac{1}{\gamma_k} (\theta_{k - 1}^{(1)} - \frac{b_i}{a_i}) \\
    &\leq \max_{i' \in \{ 1, 2 \}} \frac{1}{\gamma_k} (\theta_{k - 1}^{(1)} - \frac{b_{i'}}{a_{i'}}) = \frac{1}{\gamma_k} (\theta_{k - 1}^{(1)} - \min_{i' \in \{ 1, 2 \}} \frac{b_{i'}}{a_{i'}}) \\
    &= \frac{1}{\gamma_k} (\theta_{k - 1}^{(1)} - \frac{b_1}{a_1})
\end{align*}
where we use (\ref{eq:k0def2}) and the fact that $\frac{b_1}{a_1} = 0, \frac{b_2}{a_2} > 0$. Next, we deduce that
\begin{equation} \label{eq:b1}
    \theta_k^{(1)} = \theta_{k - 1}^{(1)} - \gamma_k \mathcal{G}_{FO} (\theta_{k - 1}, \mathcal{T}_k)^{(1)} \geq \theta_{k - 1}^{(1)} - \frac{\gamma_k}{\gamma_k} (\theta_{k - 1}^{(1)} - \frac{b_1}{a_1}) = \frac{b_1}{a_1}.
\end{equation}
On the other hand, from (\ref{eq:glb})
\begin{equation*}
    \mathcal{G}_{FO} (\theta_{k - 1}, \mathcal{T}_k)^{(1)} > B
\end{equation*}
and
\begin{equation} \label{eq:b2}
    \theta_k^{(1)} = \theta_{k - 1}^{(1)} - \gamma_k \mathcal{G}_{FO} (\theta_{k - 1}, \mathcal{T}_k)^{(1)} < \theta_{k - 1}^{(1)} - \gamma_k B .
\end{equation}
According to (\ref{eq:step}) there exists a number $k_1 > k$ such that
\begin{equation} \label{eq:k1def}
    \sum_{k' = k}^{k_1 - 1} \gamma_{k'} > \frac{1}{B} (\theta^{(1)}_{k - 1} - \frac{b_1}{a_1}).
\end{equation}
In addition, let $k_1$ be a minimal such number. Suppose that for all $k \leq k' \leq k_1$ $\theta_{k' - 1}^{(1)} > \frac{b_1}{a_1} + A$. Then by applying bound (\ref{eq:b2}) for all $k = k'$ we obtain that
\begin{equation*}
    \theta_{k_1}^{(1)} < \theta_{k_1 - 1}^{(1)} - \gamma_{k_1} B < \dots < \theta_{k - 1}^{(1)} - B \sum_{k' = k}^{k_1} \gamma_{k'} < \frac{b_1}{a_1}
\end{equation*}
which is a contradiction with the bound (\ref{eq:b1}) applied to $k = k_1$. Therefore, there exists $k \leq k' < k_1$ such that $\theta_{k'}^{(1)} \leq \frac{b_1}{a_1} + A$. Then there exists a number
\begin{equation} \label{eq:k2def}
    k_2 = \min_{k \leq k' < k_1, \theta_{k'}^{(1)} \leq \frac{b_1}{a_1} + A} k' .
\end{equation}
Hence, $\theta_{k_2 - 1}^{(1)} > \frac{b_1}{a_1} + A$ and by applying bound (\ref{eq:b1}) to $k = k_2$ we conclude that $\theta_{k_2}^{(1)} \geq \frac{b_1}{a_1}$. Averall:
\begin{equation*}
    \theta_{k_2}^{(1)} \in [\frac{b_1}{a_1}, \frac{b_1}{a_1} + A] \subseteq I.
\end{equation*}
As shown in Case 1, for all $k' > k_2$ (including $k_1$) it also holds that $\theta_{k'}^{(1)} \in I$. To summarize, we have proven that there exists a deterministic number $B > 0$ such that for $k_1$ defined by (\ref{eq:k1def}) $\theta_{k'}^{(1)} \in I$ for all $k' \geq k_1$.

\item Case 3: $\theta_{k - 1}^{(1)} < \frac{b_2}{a_2} - A$. Using a symmetric argument as in Case 2 it can be shown that there exists a deterministic number $C > 0$ so that the following holds. According to (\ref{eq:step}) there exists $k_3 \geq k$ such that
\begin{equation} \label{eq:k3def}
    \sum_{k' = k}^{k_3 - 1} \gamma_{k'} > \frac{1}{C} (\frac{b_2}{a_2} - \theta^{(1)}_{k - 1}).
\end{equation}
In addition, let $k_3$ be a minimal such number. Then $\theta_{k'}^{(1)} \in I$ for all $k' \geq k_3$.
\end{enumerate}

Since $p(\mathcal{T})$ is a discrete distribution, there only exists a finite number of outcomes for a set of random variables $\{ \mathcal{T}_k \}_{k < k_0}$. Therefore, there is only a finite set of possible outcomes of the $\theta_{k_0 - 1}^{(1)}$ random variable. Consequently, there exists a deterministic number $E > 0$ such that $| \theta_{k_0 - 1}^{(1)} | < E$. According to (\ref{eq:step}) there exist deterministic numbers $k_4, k_5 \geq k_0$ such that
\begin{equation} \label{eq:k45def}
    \sum_{k' = k_0}^{k_4 - 1} \gamma_{k'} > \frac{1}{B} (E - \frac{b_1}{a_1}), \quad \sum_{k' = k_0}^{k_5 - 1} \gamma_{k'} > \frac{1}{C} (\frac{b_2}{a_2} + E).
\end{equation}
and let $k_6 = \max (k_4, k_5)$, which is also a deterministic number. Let $k_1, k_3$ be random numbers from Cases 2, 3 applied to $k = k_0$. Then from (\ref{eq:k45def}) and $E$'s definition, it follows that $k_1, k_3 \leq k_6$. In addition, $k_0 \leq k_6$ from $k_6$'s definition. As a result of all cases, we conclude that for any $k \geq k_6$ $\phi_k^{(1)} \in I$.

Denote
\begin{gather*}
    a^* = \frac{1}{2} ( a_1 (1 - \alpha a_1)^r + a_2 (1 - \alpha a_2)^r ), \quad b^* = \frac{1}{2} (b_1 (1 - \alpha a_1)^r + b_2 (1 - \alpha a_2)^r ),  \quad x^* = \frac{b^*}{a^*}
\end{gather*}
and consider arbitrary $k > k_6$. Denote $\overline{\mathcal{G}} = \mathcal{G}_{FO} (\theta_{k - 1}, \mathcal{T}_k)$ and let $\mathcal{F}_k$ be a $\sigma$-algebra populated by $\{ \mathcal{T}_\kappa \}, \kappa < k$. From Equation (\ref{eq:fomamloutcnt}) we conclude that
\begin{equation*}
    \E{ \overline{\mathcal{G}}^{(1)} | \mathcal{F}_k } = a^* \theta_{k - 1}^{(1)} - b^*
\end{equation*}
Outer-loop update leads to an expression:
\begin{equation*}
    \theta_k^{(1)} = \theta_{k - 1}^{(1)} - \gamma_k \overline{\mathcal{G}}^{(1)}.
\end{equation*}
Subtract $x^*$:
\begin{equation*}
    \theta_k^{(1)} - x^* = \theta_{k - 1}^{(1)} - x^* - \gamma_k \overline{\mathcal{G}}^{(1)}.
\end{equation*}
Take a square:
\begin{align*}
    ( \theta_k^{(1)} - x^* )^2 &= ( \theta_{k - 1}^{(1)} - x^* - \gamma_k \overline{\mathcal{G}}^{(1)} )^2
    &= (\theta_{k - 1}^{(1)} - x^*)^2 - 2 \gamma_k (\theta_{k - 1}^{(1)} - x^*) \overline{\mathcal{G}}^{(1)} + \gamma_k^2 \overline{\mathcal{G}}^{(1) 2} .
\end{align*}
Take expectation conditioned on $\mathcal{F}_k$:
\begin{align*}
    \E{ ( \theta_k^{(1)} - x^* )^2 | \mathcal{F}_k} &= ( \theta_{k - 1}^{(1)} - x^* )^2 - 2 \gamma_k (\theta_{k - 1}^{(1)} - x^*) \E{ \overline{\mathcal{G}}^{(1)} | \mathcal{F}_k} + \gamma_k^2 \E {  \left(\overline{\mathcal{G}}^{(1)}\right)^{2} | \mathcal{F}_k} \\
    &= ( \theta_{k - 1}^{(1)} - x^* )^2 - 2 \gamma_k (\theta_{k - 1}^{(1)} - x^*) (a^* \theta_{k - 1}^{(1)} - b^*) + \gamma_k^2 \E { \left(\overline{\mathcal{G}}^{(1)} \right)^{2} | \mathcal{F}_k} \\
    &= ( \theta_{k - 1}^{(1)} - x^* )^2 - 2 \gamma_k a^* (\theta_{k - 1}^{(1)} - x^*) ( \theta_{k - 1}^{(1)} - \frac{b^*}{a^*}) + \gamma_k^2 \E { \left(\overline{\mathcal{G}}^{(1)}\right)^{2} | \mathcal{F}_k} \\
    &= (1 - 2 \gamma_k a^*) ( \theta_{k - 1}^{(1)} - x^* )^2 + \gamma_k^2 \E{ \left(\overline{\mathcal{G}}^{(1)}\right)^{2} | \mathcal{F}_k}.
\end{align*}
For $i \in \{ 1, 2 \}$, $\mathcal{G}_{FO} (\theta_{k - 1}, \mathcal{T}^{(i)})^{(1)}$ depends linearly on $\theta_{k - 1}^{(1)}$ (\ref{eq:fomamloutcnt}) and, therefore, is bounded on $\theta_{k - 1}^{(1)} \in I$. Hence, $\overline{\mathcal{G}}^{(1) 2}$ is also bounded by a deterministic number $F > 0$: $\overline{\mathcal{G}}^{(1) 2} < F$. Then:
\begin{equation*}
    \E{ ( \theta_k^{(1)} - x^* )^2 | \mathcal{F}_k} \leq (1 - 2 \gamma_k a^*) ( \theta_{k - 1}^{(1)} - x^* )^2 + \gamma_k^2 F.
\end{equation*}
Take a full expectation:
\begin{equation*}
    \E{( \theta_k^{(1)} - x^* )^2} \leq (1 - 2 \gamma_k a^*) \mathbb{E} ( \theta_{k - 1}^{(1)} - x^* )^2 + \gamma_k^2 F,
\end{equation*}
and denote $y_k = \E{ ( \theta_k^{(1)} - x^* )^2}$:
\begin{equation} \label{eq:yineq}
    y_k \leq (1 - 2 \gamma_k a^*) y_{k - 1} + \gamma_k^2 F.
\end{equation}
Now, we prove that $\lim_{k \to \infty} y_k = 0$. Indeed, consider arbitrary $\epsilon > 0$. According to (\ref{eq:step}) there exists $k_\epsilon > k_6$ such that $\forall k' \geq k_\epsilon : \gamma_{k'} \leq \frac{a^* \epsilon}{F}$. As a result of (\ref{eq:yineq}) for every $k \geq k_\epsilon$ it holds
\begin{equation*}
    y_k \leq (1 - 2 \gamma_k a^*) y_{k - 1} + \gamma_k^2 F \leq (1 - 2 \gamma_k a^*) y_{k - 1} + \gamma_k a^* \epsilon.
\end{equation*}
Subtract $\frac{\epsilon}{2}$:
\begin{equation} \label{eq:yineq2}
    y_k - \frac{\epsilon}{2} \leq (1 - 2 \gamma_k a^*) ( y_{k - 1} - \frac{\epsilon}{2} ) .
\end{equation}
Observe that by (\ref{eq:k0def1}), $a^*$'s definition and since $k \geq k_0$ it holds that $1 - 2 \gamma_k a^* > 0$. Therefore and since (\ref{eq:yineq2}) holds for all $k \geq k_\epsilon$, it can be written that for all $k \geq k_\epsilon$
\begin{equation*}
    y_k - \frac{\epsilon}{2} \leq \biggl( \prod_{k' = k_\epsilon}^k (1 - 2 \gamma_{k'} a^*) \biggr) ( y_{k_\epsilon - 1} - \frac{\epsilon}{2} ) \leq \biggl( \prod_{k' = k_\epsilon}^k (1 - 2 \gamma_{k'} a^*) \biggr) | y_{k_\epsilon - 1} - \frac{\epsilon}{2} | .
\end{equation*}
We use inequality $1 - x \leq \exp(- x)$ to deduce that
\begin{align}
    y_k - \frac{\epsilon}{2} &\leq \biggl( \prod_{k' = k_\epsilon}^k (1 - 2 \gamma_{k'} a^*) \biggr) | y_{k_\epsilon - 1} - \frac{\epsilon}{2} | \nonumber \\
    &\leq \exp \biggl( - 2 a^* \sum_{k' = k_\epsilon}^k \gamma_{k'} \biggr) | y_{k_\epsilon - 1} - \frac{\epsilon}{2} | . \label{eq:yineq3}
\end{align}
If $ | y_{k_\epsilon - 1} - \frac{\epsilon}{2} | = 0$, then from (\ref{eq:yineq3}) it follows that $y_k \leq 0 + \frac{\epsilon}{2} < \epsilon$ for all $k \geq k_\epsilon$. Otherwise, from (\ref{eq:step}) there exists $k_\epsilon'$ such that $\sum_{k' = k_\epsilon}^{k_\epsilon'} \gamma_{k'} > \frac{\log | y_{k_\epsilon - 1} - \frac{\epsilon}{2} | - \log \frac{\epsilon}{2} }{2 a^*}$. Then from (\ref{eq:yineq3}) it follows that for all $k \geq k_\epsilon'$ $y_k - \frac{\epsilon}{2} < \frac{\epsilon}{2}$ or $y_k < \epsilon$. Since $y_k \geq 0$ by definition, we have proven that $\lim_{k \to \infty} y_k = 0$, or
\begin{equation} \label{eq:lim}
    \lim_{k \to \infty} \E {(\theta_k^{(1)} - x^* )^2} = 0 .
\end{equation}
Again, let $k > k_6$. Let $\phi_0 = \theta_k, \dots, \phi_r$ be a rollout (\ref{eq:maml}) of inner GD for task $\mathcal{T}^{(i)}$. Then according to (\ref{eq:gradfirst}-\ref{eq:gradlast})
\begin{equation*}
    \nabla_{\theta_k} \mathcal{L}^{out} (\theta_k, U^{(r)} (\theta_k, \mathcal{T}^{(i)}), \mathcal{T}^{(i)})^{(1)} = \mathcal{G}_{FO} (\theta_k, \mathcal{T}^{(i)})^{(1)} \prod_{j = 0}^{r - 1} (1 - \alpha f_i'' (\phi_j^{(1)})) .
\end{equation*}
From (\ref{eq:iprop}) it follows that $f_i'' (\phi_j^{(1)}) = a_i$ for $j \in \{ 0, \dots, r - 1 \}$. Moreover, we use (\ref{eq:fomamloutcnt}) to obtain that
\begin{equation*}
    \nabla_{\theta_k} \mathcal{L}^{out} (\theta_k, U^{(r)} (\theta_k, \mathcal{T}^{(i)}), \mathcal{T}^{(i)})^{(1)} = a_i (1 - \alpha a_i)^{2 r} ( \theta_k^{(1)} - \frac{b_i}{a_i} )
\end{equation*}
and
\begin{align*}
    \frac{\partial}{\partial \theta} \mathcal{M}^{(r)} (\theta_k)^{(1)} &= \mathbb{E}_{p(\mathcal{T})} \left[\nabla_{\theta_k} \mathcal{L}^{out} (\theta_k, U^{(r)} (\theta_k, \mathcal{T}^{(i)}), \mathcal{T}^{(i)})^{(1)} \right] \\
    &= \frac{1}{2} \biggl( a_1 (1 - \alpha a_1)^{2 r} ( \theta_k^{(1)} - \frac{b_1}{a_1} ) + a_2 (1 - \alpha a_2)^{2 r} ( \theta_k^{(1)} - \frac{b_2}{a_2} ) \biggr) \\
    &= \widehat{a} \theta_k^{(1)} - \widehat{b}
\end{align*}
where
\begin{equation*}
    \widehat{a} = \frac{1}{2} (a_1 (1 - \alpha a_1)^{2 r} + a_2 (1 - \alpha a_2)^{2 r}), \quad \widehat{b} = \frac{1}{2}(b_1 (1 - \alpha a_1)^{2 r} + b_2 (1 - \alpha a_2)^{2 r}).
\end{equation*}
Notice that since $b_1 = 0$ and $b_2$ is defined by (\ref{eq:b2def}), it appears that $| \widehat{a} x^* - \widehat{b} | = \sqrt{2 D}$, or $(\widehat{a} x^* - \widehat{b})^2 = 2 D$. Multiply (\ref{eq:lim}) by $\widehat{a}$ to obtain that
\begin{align}
    \lim_{k \to \infty} \mathbb{E} \Big[ (\widehat{a} \theta_k^{(1)} &- \widehat{a} x^* )^2 \Big] = 0, \\
    \lim_{k \to \infty} \mathbb{E} \Big[( \widehat{a} \theta_k^{(1)} - \widehat{b} &- (\widehat{a} x^* - \widehat{b}))^2 \Big] = 0, \nonumber \\
    \lim_{k \to \infty} \mathbb{E} \Big[ ( \frac{\partial}{\partial \theta} \mathcal{M}^{(r)} (\theta_k)^{(1)} &- (\widehat{a} x^* - \widehat{b}))^2 \Big]  = 0 . \label{eq:lim2}
\end{align}
For each $k \geq 1$ by Jensen's inequality :
\begin{equation*}
     \left( \E {\frac{\partial}{\partial \theta} \mathcal{M}^{(r)} (\theta_k)^{(1)}} - (\widehat{a} x^* - \widehat{b}) \right)^2 \leq  \E{ (\frac{\partial}{\partial \theta} \mathcal{M}^{(r)} (\theta_k)^{(1)} - (\widehat{a} x^* - \widehat{b}))^2}.
\end{equation*}
Hence,
\begin{equation*}
    \lim_{k \to \infty} \left( \E {\frac{\partial}{\partial \theta} \mathcal{M}^{(r)} (\theta_k)^{(1)} } - (\widehat{a} x^* - \widehat{b})) \right)^2 = 0, \quad \lim_{k \to \infty} \E{\frac{\partial}{\partial \theta} \mathcal{M}^{(r)} (\theta_k)^{(1)}} = (\widehat{a} x^* - \widehat{b})
\end{equation*}
and by expanding (\ref{eq:lim2}) we derive that
\begin{equation*}
    \lim_{k \to \infty} \E{ \left(\frac{\partial}{\partial \theta} \mathcal{M}^{(r)} (\theta_k)^{(1)}\right)^{2} } = 2 (\widehat{a} x^* - \widehat{b}) \lim_{k \to \infty} \E{\frac{\partial}{\partial \theta} \mathcal{M}^{(r)} (\theta_k)^{(1)}} - (\widehat{a} x^* - \widehat{b})^2 = (\widehat{a} x^* - \widehat{b})^2 = 2 D.
\end{equation*}
We conclude the proof by observing that
\begin{equation*}
    \liminf_{k \to \infty}\E{ \| \frac{\partial}{\partial \theta} \mathcal{M}^{(r)} (\theta_k)\|^2_2 } =  \lim_{k \to \infty} \E{\| \frac{\partial}{\partial \theta}\mathcal{M}^{(r)} (\theta_k) \|^2_2} = \lim_{k \to \infty} \E{ \left(\frac{\partial}{\partial \theta} \mathcal{M}^{(r)} (\theta_k)^{(1)} \right)^{2} } = 2 D > D.
\end{equation*}
\end{proof}

\section{Optimal Choice of $q$} \label{sec:qchoice}

The derivative of (\ref{eq:time}) has the form
\begin{align}
    g (q) &= \frac{2}{2 \epsilon - 1} ((1 / q - 1) \mathbb{D}^2_{true} + \mathbb{V}^2_{true})^{\frac{1 + 2 \epsilon}{1 - 2 \epsilon}} q^{-2} \mathbb{D}^2_{true} (C_{det} + C_{rnd} q) +  \left((1/q - 1) \mathbb{D}^2_{true} + \mathbb{V}^2_{true} \right)^{\frac{2}{1 - 2 \epsilon}} C_{rnd} \nonumber \\
    &= q^{-2} ((1 / q - 1) \mathbb{D}^2_{true} + \mathbb{V}^2_{true})^{\frac{1 + 2 \epsilon}{1 - 2 \epsilon}} \biggl( \frac{2}{2 \epsilon - 1} \mathbb{D}^2_{true} (C_{det} + C_{rnd} q) + C_{rnd} (\mathbb{D}^2_{true} + q (\mathbb{V}^2_{true} - \mathbb{D}^2_{true})) q \biggr) \nonumber \\
    &= q^{-2} ((1 / q - 1) \mathbb{D}^2_{true} + \mathbb{V}^2_{true})^{\frac{1 + 2 \epsilon}{1 - 2 \epsilon}} \biggl( C_{rnd} (\mathbb{V}^2_{true} - \mathbb{D}^2_{true}) q^2 + \frac{2 \epsilon + 1}{2 \epsilon - 1} \mathbb{D}^2_{true} C_{rnd} q + \frac{2}{2 \epsilon - 1} \mathbb{D}^2_{true} C_{det} \biggr) . \label{eq:quadreq}
\end{align}

We further deduce that
\begin{equation*}
    g (1) = ( \mathbb{V}^2_{true} )^{\frac{1 + 2 \epsilon}{1 - 2 \epsilon}} \biggl( C_{rnd} \mathbb{V}^2_{true} + \frac{2}{2 \epsilon - 1} (C_{det} + C_{rnd}) \mathbb{D}^2_{true} \biggr) .
\end{equation*}
Assuming that $\mathbb{V}^2_{true} > 0$, we conclude that $g(1) > 0$ iff
\begin{equation}
    \mathbb{D}_{true}^2 < \frac{C_{rnd}}{ \frac{2}{1 - 2 \epsilon} (C_{det} + C_{rnd})} \mathbb{V}^2_{true} . \label{eq:cond22}
\end{equation}
%Since $0 < \epsilon < \frac12$ is any number and $\frac{2 C_{rnd}}{C_{det} + C_{rnd}} \mathbb{V}^2_{true} < \frac{C_{rnd}}{ (\frac12 - \epsilon) (C_{det} + C_{rnd})} \mathbb{V}^2_{true}$, the sufficient condition for $g(1) > 0$ is obtained by setting $\epsilon = 0$:
%\begin{equation}
%    \mathbb{D}_{true}^2 < \frac{2 C_{rnd}}{C_{det} + C_{rnd}} \mathbb{V}^2_{true} . \label{eq:cond3}
%\end{equation}
$g(q)$ is differentiable on $(0, + \infty)$. Let $q^*$ be the value corresponding to the minimum of (\ref{eq:time}) on $(0, 1]$. $q^*$ exists, since (\ref{eq:time}) approaches $+ \infty$ when $q^* \to 0$. $g(1) > 0$ indicates that $q^*$ is not equal to $1$, i.e. we obtain a tigher upper bound using UFOM rather than exact gradients. Further, solving $g (q) = 0$ reduces to solving a quadratic equation induced by the polynomial inside big brackets of (\ref{eq:quadreq}):
\begin{equation*}
    \mathrm{poly} (q) = C_{rnd} (\mathbb{V}^2_{true} - \mathbb{D}^2_{true}) q^2 + \frac{2 \epsilon + 1}{2 \epsilon - 1} \mathbb{D}^2_{true} C_{rnd} q + \frac{2}{2 \epsilon - 1} \mathbb{D}^2_{true} C_{det} = 0 \label{eq:quadr}
\end{equation*}

Notice, that if $\epsilon < \frac12$ and (\ref{eq:cond22}) is satisfied, then $\mathrm{poly} (1) > 0$ and $\mathrm{poly} (0) < 0$. Hence, from the continuity of $\mathrm{poly} (q)$, it follows that there is an odd number of roots of $\mathrm{poly} (q)$ on $(0, 1)$. Since the quadratic equation has at most 2 roots, we conclude that there is a single root on $(0, 1)$. Since (\ref{eq:time}) is differentiable on $(0, 1]$, $q = 1$ is not a local minimum of (\ref{eq:time}) on $(0, 1]$ and $\lim_{q \to +0} g(q) = + \infty$, we conclude that this single root is a minimum of (\ref{eq:time}) on $(0, 1]$.

\end{document}